\documentclass[11pt]{article}

\usepackage[margin=1in]{geometry}
\usepackage[T1]{fontenc}
\usepackage[utf8]{inputenc}
\usepackage{amsmath,amsthm,mathtools}
\usepackage{newtxtext,newtxmath}
\usepackage{bm}
\usepackage{microtype}
\usepackage{booktabs,tabularx,array,multirow}
\usepackage{graphicx}
\usepackage{tikz}
\usetikzlibrary{arrows.meta,positioning,fit,backgrounds}
\usepackage{algorithm}
\usepackage[noend]{algpseudocode}
\usepackage{placeins}

\usepackage{needspace}
\usepackage[round,authoryear]{natbib}
\usepackage{xcolor}
\usepackage{url}
\usepackage{hyperref}

\definecolor{wmnavy}{RGB}{24,57,84}
\definecolor{wmteal}{RGB}{24,122,110}
\definecolor{wmorange}{RGB}{184,91,35}
\definecolor{wmlight}{RGB}{215,230,238}
\tikzset{
  latentnode/.style={circle,draw=wmnavy,thick,minimum size=8mm,inner sep=1pt,fill=white},
  obsnode/.style={circle,draw=wmnavy,thick,minimum size=8mm,inner sep=1pt,fill=wmlight},
  detnode/.style={rectangle,draw=wmnavy,thick,minimum size=8mm,inner sep=1pt,fill=white},
  priornode/.style={rectangle,draw=wmteal,thick,rounded corners=1pt,minimum width=12mm,minimum height=6mm,fill=white},
  platenode/.style={draw=wmteal,thick,rounded corners=2pt,inner sep=4mm},
  gmarrow/.style={-{Latex[length=2mm]},thick,draw=wmnavy}
}
\hypersetup{
  colorlinks=true,
  linkcolor=wmnavy,
  citecolor=wmteal,
  urlcolor=wmorange,
  pdftitle={Predictive Likelihood Ratios for Language Model Watermark Detection},
  pdfauthor={Li Ma},
  pdfsubject={Technical Report}
}
\graphicspath{{figures/}}

\newtheorem{theorem}{Theorem}
\newtheorem{proposition}{Proposition}

\newtheorem{remark}{Remark}

\newcommand{\E}{\mathbb E}
\newcommand{\Prb}{\mathbb P}

\newcommand{\F}{\mathcal F}
\newcommand{\V}{\mathcal V}
\newcommand{\ind}{\mathbf 1}
\newcommand{\bP}{\bm p}

\newcommand{\LSE}{\operatorname{LSE}}
\newcommand{\logit}{\operatorname{logit}}
\newcommand{\Unif}{\operatorname{Unif}}
\newcommand{\mcse}[2]{\mbox{#1\,{\fontsize{7}{8}\selectfont(#2)}}}
\newcommand{\bestmcse}[2]{\mbox{\textbf{#1}\,{\fontsize{7}{8}\selectfont(#2)}}}

\title{\textbf{Predictive Likelihood Ratios\\for Language Model Watermark Detection}\\[0.5em]
  {\large Technical Report}}
\newcommand{\AuthorAffiliation}{Department of Statistics and Data Science Institute, University of Chicago}
\newcommand{\AuthorEmail}{li.ma@uchicago.edu}
\author{Li Ma\\
  {\normalsize\AuthorAffiliation}\\
  {\small\href{mailto:\AuthorEmail}{\texttt{\AuthorEmail}}}}
\date{September 12, 2026}

\begin{document}
\maketitle

\begin{abstract}
Keyed watermark detection tests dependence between observed tokens and
pseudorandom variables reconstructed from a secret key.  Building on the pivotal
framework of \citet{li2025framework}, we construct predictive likelihood ratios
that average over uncertain probability deficits and residual-tail distributions.
The aim is robust detection power across alternative specifications without
requiring a single signal-strength tuning.  A mixture prior combines tail shape
and effective width; hierarchical extensions allow within-document variation in
deficit or width.  The test maximizes prior-averaged power at a fixed size, but
is not generally uniformly most powerful or minimax.  Under the exact conditional
pivot null, normalized predictive alternatives selected before each observation
yield a Bayes factor that is also a test martingale: Type~I error control is unaffected by alternative
misspecification and remains valid under optional stopping.  This guarantee
does not cover violations of the conditional null, and the interpolated
implementation has no certified anytime guarantee.  Gumbel marginal likelihoods
are evaluated by fixed quadrature.  Across the evaluated tail-shape and
tail-width alternatives and three horizons, the union-tail mixture has maximum
observed Type~II error regret $.0080$, compared with $.0962$ for the equal-tail
mixture, relative to the best tested rule.  On temperature-matched outputs from
two open models, it improves AUC over the equal-tail baseline in all eight
non-saturated model--temperature cells, although the leading reference score
generally has higher AUC.  Supplementary experiments show retained power under
independent null-like replacement and smaller changes from hierarchical dependence
modeling.  The evidence supports robustness across the evaluated alternatives,
not uniform power guarantees or resistance to arbitrary text edits.
\end{abstract}


\section{Introduction}

Language-model watermarking induces dependence between generated tokens and pseudorandom variables determined by a secret key.  Detection tests this dependence against a null under which token generation is conditionally independent of the keyed variables.  The statistical problem differs from classification of machine-generated text: the null distribution must remain valid without knowledge of the generating model's next-token probabilities (NTPs).  Early watermarking schemes include green-list logit modulation \citep{kirchenbauer2023watermark} and distortion-free sampling based on randomized races \citep{aaronson2023watermarking} or inverse transforms \citep{kuditipudi2024robust}.  For the latter class, \citet{li2025framework} construct pivots with nuisance-free null distributions, derive asymptotic Type~II error exponents, and obtain a minimax saddle-point score for the Gumbel-max pivot studied here.

The pivotal formulation eliminates the NTP vector from the null distribution but leaves a composite alternative.  NTP distributions vary across prompts and token contexts and depend on sampling temperature and decoding rules.  We represent this variation by a prespecified prior over a low-dimensional alternative family.  For Gumbel-max and a fixed $\Delta$, \citet{li2025framework} construct a score that maximizes a detection-efficiency rate against the least favorable member of the $\Delta$-regular class.  Its size is unchanged under the exact pivot null, but its power depends on the alternative distribution.  We distinguish misspecification of $\Delta_0$, examined in Supplementary Section~\ref{sec:regimes}, from misspecification of the residual-tail distribution, examined at a fixed deficit law in Supplementary Section~\ref{sec:tails}.  \citet{li2025framework} discuss adapting the distribution class to empirical NTP distributions.  Their least favorable vector concentrates residual mass on as few coordinates as the constraint permits, whereas the equal-tail and Dirichlet working alternatives distribute it more broadly.

Robustness to alternative specification is a central motivation: a detector should retain power when watermark strength or residual-tail structure differs from a single working specification.  We test the null against a marginal density that averages over a prespecified prior on these quantities.  By the Neyman--Pearson lemma, the test maximizes prior-averaged power at a fixed size.  A prior on $\Delta$ replaces the single regularity bound whose selection \citet{li2025framework} identify as an open problem.  Robustness of power is assessed across generating regimes without retuning the priors; it is not guaranteed by prior averaging alone.

Write $Y_t$ for the pivot at token $t$ and $\F_{t-1}$ for the information available before it is observed.  We call the one-step ratio $E_t=\widetilde f_{1,t}(Y_t\mid\F_{t-1})/f_{0,t}(Y_t\mid\F_{t-1})$ the \emph{predictive likelihood ratio}; its cumulative product $B_n=\prod_{t=1}^n E_t$ is the Bayes factor for the specified alternative.  Because the exact pivot null is free of the unknown NTP vector, prior averaging occurs only in the alternative predictive density.  If that density is normalized and selected before each observation, $(B_n)$ is also a test martingale, so Ville's inequality controls Type~I error under optional stopping as well as at fixed horizons.  Theorem~\ref{thm:martingale} formalizes this construction.  This validity is robust to misspecification of the alternative, but not of the conditional null.  The guarantee is exact for the analytic densities; Supplementary Section~\ref{sec:numerical-checks} reports numerical validation, not a certified anytime guarantee, for the interpolated implementation.

Likelihood-based watermark detectors differ in their treatment of the unknown alternative.  The Bayesian scoring rule for SynthID-Text estimates likelihoods for keyed features and returns a posterior probability of watermark presence; the same study's supplementary experiments include a fitted Bayesian scorer for a Gumbel watermark \citep{dathathri2024synthid}.  The Bayes' Rule derived Watermark Detector (BRWD) computes model-based tokenwise likelihood-ratio contributions from unperturbed and watermark-perturbed logits and aggregates them into a score monotone in the model-implied posterior watermark probability for fixed prior odds; the logits may be supplied by a smaller surrogate language model \citep[Section~4.2]{huang2025brwd}.  A complementary likelihood-based detector uses the key together with known, estimated, or shrunken NTPs \citep{li2025likelihood}.  These methods either estimate the alternative from labeled data or require NTP information.  The present construction uses the analytic pivots of \citet{li2025framework}, with prespecified priors on $\Delta$ and the Gumbel residual tail.

For online decisions, test martingales and e-processes provide optional-stopping-safe evidence \citep{ville1939etude,shafer2011test,ramdas2023game,grunwald2024safe}.  This sequential property is not implied by fixed-horizon e-validity \citep{vovk2021evalues,ramdas2023game}.  Watermark-specific e-processes support continuous monitoring under conditional pivotal-null assumptions \citep{su2026eprocess,luo2026rao}.  Anchored E-Watermarking instead uses a shared anchor as the seed distribution and constructs token--seed e-values that are uniformly valid when the target distribution lies in a specified $\ell_1$ neighborhood of that anchor \citep{huang2026anytime}, while robust e-processes seek Type~I control over contamination neighborhoods of the null \citep{saha2026huber}.

\Needspace{6\baselineskip}
The analysis has three components:
\begin{enumerate}
  \item {\em Predictive likelihood ratio construction:} The exact Gumbel pivot likelihood is integrated over prespecified deficit and residual-tail priors.  The spike and tail-width densities are closed form; the finite-concentration tail-shape model requires one univariate transform.  The deficit-prior integral uses fixed Gauss--Legendre quadrature, with sensitivity analyses for the number of nodes.  We distinguish posterior expected-loss thresholds, calibrated fixed-horizon tests, and anytime-valid tests under explicit conditional-null assumptions.

  \item {\em Alternative specification and hierarchical models:} A mixture prior on the union of two families separates residual-tail shape from effective width.  The exchangeable Dirichlet component reduces a $K$-dimensional integral to one univariate transform per concentration parameter; the width component includes concentrated residual tails.  Both families are normalized and evaluated under specified tail-shape and width alternatives.  Shared and tokenwise specifications determine which latent parameters are common within a document.  Further hierarchical models in Supplementary Section~\ref{sec:hierarchical-extensions} place priors on document-level hyperparameters governing the conditional distributions of token-specific deficits or widths, thereby inducing partial pooling.  A document-level tail state is inferred from tokens within each document.  Supplementary Section~\ref{sec:contamination-supplement} adds independent null-like contamination.

  \item {\em Empirical robustness assessment:} Synthetic comparisons with the minimax and baseline scores of \citet{li2025framework}, at a common nominal 5\% size, assess sensitivity to deficit, residual-tail, and within-document dependence specifications.  Maximum regret is evaluated over finite collections of alternatives; supplementary contamination experiments use independent null-like replacement.  A supplementary analysis reexamines the \emph{archived benchmark data of \citet{li2025framework}}: watermarked and unwatermarked continuations from OPT-1.3B and Sheared-LLaMA-2.7B on C4 prompts, stored in their \emph{WatermarkFramework} GitHub repository \citep{li2024watermarkframework}.  These pre-existing benchmark samples reveal temperature confounding and lack of fit of the full-width tail model.  Using temperature-matched data generated for this study, we assess the change in performance associated with narrower tail support through paired prompt-cluster inference and an extension-prompt holdout.  The reanalysis of the archived benchmark data uses independent exact-null calibration and reports every prefix through 200; neither the archived benchmark sequences nor our newly generated sequences support an anytime-valid claim.
\end{enumerate}

\section{Method}

\subsection{Keyed pivots and the conditional null}
\label{sec:pivots}

Let $W_t\in\V$, $|\V|=M$, denote token $t$, let $\xi_t$ denote the keyed pseudorandom object reconstructed at that position, and let $\F_t=\sigma(W_{1:t},\xi_{1:t})$.  The idealized null condition of \citet{li2025framework} is
\begin{equation*}
  W_t\mathrel{\perp\!\!\!\perp}\xi_t\mid\F_{t-1},
  \qquad
  \xi_t\mid\F_{t-1}\sim P_\xi,
  \tag{A1}
\end{equation*}
with fresh key randomness across positions.  A pivot $Y_t=Y(W_t,\xi_t)$ is selected so that
\begin{equation}
  Y_t\mid\F_{t-1},H_0\sim f_{0,t},
  \label{eq:pivot-null}
\end{equation}
where $f_{0,t}$ may depend on the vocabulary size $M$.  The right-hand side does not depend on the unobserved human NTP vector.  The sequential results below require the conditional statement in \eqref{eq:pivot-null}; literal independence and identical distribution are sufficient but not necessary.  Repeated contexts and hash reuse can violate the fresh-randomness idealization, as discussed by \citet[Section~III-C]{fernandez2023three} and \citet[Appendix~D.1, Remark~D.1]{li2025framework}.  Masking rules, key collisions, or a nonideal pseudorandom function can also invalidate the conditional null in \eqref{eq:pivot-null}.

\subsection{Gumbel-max likelihood}

For a watermarked decoder with NTP vector $\bP=(p_w)_{w\in\V}$, generate $U_{t,w}\stackrel{\mathrm{iid}}{\sim}\Unif(0,1)$ and set
\begin{equation}
  W_t=\arg\max_{w:p_w>0}\frac{\log U_{t,w}}{p_w},
  \qquad R_t=U_{t,W_t}.
\end{equation}

\begin{proposition}[Exact Gumbel pivot law]\label{prop:gumbel}
Assume \textup{(A1)} with the Gumbel-max key law, that is $P_\xi$ the law of $\xi_t=(\xi_{t,1},\ldots,\xi_{t,M})$ with the $\xi_{t,w}$ independent $\Unif(0,1)$; \textup{(A1)} alone leaves $P_\xi$ unspecified and does not give the pivot law.  Then under $H_0$, $R_t\mid\F_{t-1}\sim\Unif(0,1)$, so $f_0^{\mathrm{gum}}(r)=1$.  Under the watermarked decoder with fixed $\bP$,
\begin{align}
  F_{\bP}^{\mathrm{gum}}(r)
    &=\sum_{w:p_w>0}p_w r^{1/p_w}, \\
  f_{\bP}^{\mathrm{gum}}(r)
    &=\sum_{w:p_w>0}r^{1/p_w-1},
  \qquad 0<r<1.
  \label{eq:gumbel-density}
\end{align}
\end{proposition}

\begin{proof}
Under $H_0$, $W_t$ is conditionally independent of the current uniform vector, so selecting coordinate $W_t$ preserves uniformity.  Under $H_1$, conditional on $U_w=r$, token $w$ is selected exactly when $U_v\le r^{p_v/p_w}$ for every $v\ne w$.  Therefore
\[
  \Prb(W=w,R\in dr)
  =\prod_{v\ne w}r^{p_v/p_w}\,dr
  =r^{1/p_w-1}\,dr.
\]
Summing over $w$ gives the density, and integrating gives the CDF.
\end{proof}

The conditional beta law given the selected token appears in \citet{fernandez2023three}.  The marginal beta-mixture CDF in Proposition~\ref{prop:gumbel} is given explicitly by \citet{piet2025mark} and restated as Lemma~3.1 of \citet{li2025framework}.  Differentiating that CDF gives the density used for mixture integration.

For the $\Delta$-regular class $\{\bP:\max_w p_w\le 1-\Delta\}$, the least-favorable representative of \citet{li2025framework} is
\begin{align}
  m_\Delta&=\left\lfloor\frac{1}{1-\Delta}\right\rfloor,
  &q_\Delta&=1-m_\Delta(1-\Delta),\\
  \bP_\Delta^\star
  &=\big(\underbrace{1-\Delta,\ldots,1-\Delta}_{m_\Delta},q_\Delta,0,\ldots,0\big),
  \label{eq:least-fav}
\end{align}
where the residual coordinate is omitted when $q_\Delta=0$.  Its density is
\begin{equation}
  f_\Delta^{\mathrm{gum},\star}(r)
  =m_\Delta r^{\Delta/(1-\Delta)}
  +\ind\{q_\Delta>0\}r^{1/q_\Delta-1}.
  \label{eq:gumbel-lf}
\end{equation}
The log of \eqref{eq:gumbel-lf} is the optimal least-favorable Gumbel score derived by \citet{li2025framework} for a point value of $\Delta$.

The clean Gumbel benchmark uses the equal-tail spike family described by \citet{li2025framework},
\begin{equation}
  \bP_\Delta^{\mathrm{sp}}
  =\left(1-\Delta,\frac{\Delta}{M-1},\ldots,\frac{\Delta}{M-1}\right),
  \label{eq:spike-ntp}
\end{equation}
whose exact density, for $\Delta>0$, is
\begin{equation}
  f_\Delta^{\mathrm{gum},\mathrm{sp}}(r)
  =r^{\Delta/(1-\Delta)}
   +(M-1)r^{(M-1)/\Delta-1}.
  \label{eq:gumbel-spike}
\end{equation}
Supplementary Section~\ref{sec:family} evaluates the effect of the Gumbel component family.  Both constructions require $0\le\Delta\le1-M^{-1}$.  A top-one probability alone does not determine \eqref{eq:gumbel-density}: a general empirical alternative must specify a full normalized NTP vector or a prior over such vectors.

\subsection{Averaging over the NTP family}
\label{sec:tail-prior}

Both \eqref{eq:gumbel-lf} and \eqref{eq:gumbel-spike} determine the entire NTP vector from $\Delta$, although the constraint fixes only its largest coordinate.  A conditionally symmetric Dirichlet prior allocates the residual mass across the remaining $K=M-1$ tokens.  This permits variation within an explicit tail family.

For $K=M-1\ge2$, $0<\Delta<1$, and $0<\alpha<\infty$, write the tail at one token as $\Delta\bm q_t$ with $\bm q_t\mid\alpha\sim\operatorname{Dirichlet}(\alpha,\ldots,\alpha)$ on the $(K-1)$-simplex, where $\alpha$ governs how evenly the residual mass is spread.  When $K=1$, the sole tail coordinate equals one and this concentration layer is degenerate.  Conditional on a shared $\alpha$, the working model redraws $\bm q_t$ independently at every token; the tokenwise-$\alpha$ hierarchy instead redraws both $\alpha_t$ and then $\bm q_t\mid\alpha_t$.  The tokenwise placement permits multiplication of one-token marginal likelihoods; a document-shared $\bm q$ would require integrating their product over $\bm q$.  Because \eqref{eq:gumbel-density} is a \emph{sum} over coordinates, linearity and tail exchangeability leave only the univariate marginal $q\sim\operatorname{Beta}(\alpha,(K-1)\alpha)$:
\begin{equation}
  f^{\mathrm{gum}}_{\Delta,\alpha}(r)
  =r^{\Delta/(1-\Delta)}
   +K\,\E_q\big[r^{1/(\Delta q)-1}\big],
  \qquad 0<r<1.
  \label{eq:gumbel-dirichlet}
\end{equation}
At the deficit endpoint $\Delta=0$, define $f^{\mathrm{gum}}_{0,\alpha}(r)=1$ by continuity; at the concentration endpoint $\alpha=\infty$, define $f^{\mathrm{gum}}_{\Delta,\infty}$ as the equal-tail spike density \eqref{eq:gumbel-spike}.  Values of a density at $r\in\{0,1\}$ may be assigned arbitrarily.  Supplementary Section~\ref{sec:stable-computation} gives the transform representation and its numerical evaluation.

For the designated coordinate $1-\Delta$ to remain the largest coordinate for every Dirichlet draw, this parameterization requires $\Delta\le1/2$; otherwise a residual coordinate $\Delta q_{t,k}$ can exceed $1-\Delta$, and $\Delta$ is no longer the top-probability deficit even though the resulting vector is still normalized.  Thus the working tail family has $0\le\Delta\le1/2$ and $0<\alpha\le\infty$, with the endpoints interpreted above.  Every Dirichlet-layer experiment below lies in this domain except the widened-prior sensitivity analysis of Supplementary Section~\ref{sec:deficit-support}, which extends the union's finite-$\alpha$ branch to $\Delta=.999$.  In that analysis, $\Delta$ indexes one minus the mass of a designated coordinate rather than the top-probability deficit, as stated in Supplementary Section~\ref{sec:deficit-support}.  An extension beyond one half that preserved the top-probability interpretation would require truncating or reordering the tail prior and accounting for the resulting normalization.

The density is normalized:
\begin{equation}
  \int_0^1 f^{\mathrm{gum}}_{\Delta,\alpha}(r)\,dr
  =(1-\Delta)+\Delta K\,\E[q]=1,
\end{equation}
because $\E[q]=1/K$.  The analytic component therefore satisfies the numerator-normalization condition of Theorem~\ref{thm:martingale}.  The theorem applies when this component, or a normalized mixture of such components, is selected $\F_{t-1}$-measurably and divided by the true conditional null density.

The parameter $\alpha$ controls tail concentration.  As $\alpha\to\infty$ the tail becomes exactly equal and \eqref{eq:gumbel-dirichlet} recovers the spike density \eqref{eq:gumbel-spike}, the equal-tail specification stated by \citet{li2025framework}; small $\alpha$ concentrates the residual mass on a few tokens.  Supplementary Section~\ref{sec:dirichlet-details} gives convergence calculations, the comparison with normalized-uniform tails, and a size-biased simulation construction.

The layer is specific to the Gumbel pivot, the only scheme treated here.  Figure~\ref{fig:hierarchical-model} displays the uncollapsed Gumbel hierarchies, including the tail-width extension introduced next.  The token-indexed $\bm q_t$ remains inside the plate even when $(S,\Delta,\alpha,J)$ are document-shared.  It is redrawn for finite $\alpha$ in the tail-shape branch; in the tail-width branch it is deterministic given $J_{[t]}$, and thus identical across tokens when $J$ is shared.  It combines with the applicable deficit to form $\bP_t=(1-\Delta_{[t]},\Delta_{[t]}\bm q_t)$, which indexes \eqref{eq:gumbel-density}; integrating $\bm q_t$ in the tail-shape branch gives \eqref{eq:gumbel-dirichlet}.  A shared random $\bm q$ would instead induce a joint integral over all tokens.

\subsection{Averaging over the tail width}
\label{sec:tail-width}

The Dirichlet layer varies how evenly the residual mass is distributed but assigns
positive mass to all $K=M-1$ non-leading coordinates.  This is a second point
assumption, and unlike tail shape it depends on the vocabulary size.  Its
consequence is visible in the equal-tail density \eqref{eq:gumbel-spike}, whose
second term has exponent $(M-1)/\Delta-1$ and concentrates its probability mass near
$r=1$.  At $M=1000$ that term is nonnegligible over a visible range of $r$; at the
OPT-1.3B and Sheared-LLaMA-2.7B vocabulary sizes $M=50{,}272$ and $M=32{,}000$,
respectively, it is numerically concentrated
at $r=1$.  Its integrated mass remains $\Delta$ for every $M$, since
$\int_0^1(M-1)r^{(M-1)/\Delta-1}\,dr=\Delta$.  For $r$ bounded away from one, the
component is approximately $r^{\Delta/(1-\Delta)}$, and its log likelihood ratio is
approximately proportional to $\log r$ at each fixed $\Delta$.  The implementation
evaluates the exact normalized density \eqref{eq:gumbel-spike}, not this
away-from-boundary approximation.  Integration over $\Delta$ does not alter the
boundary concentration because it occurs for every component.

The number of nonzero tail coordinates is therefore modeled as an additional latent
quantity.  Writing $J$ for that number and assigning equal mass within it, the
working NTP vector is $(1-\Delta,\Delta/J,\ldots,\Delta/J)$ and the exact Gumbel
pivot density is
\begin{equation}
  f^{\mathrm{gum}}_{\Delta,J}(r)
  =r^{\Delta/(1-\Delta)}+J\,r^{J/\Delta-1},
  \qquad 1\le J\le K.
  \label{eq:gumbel-tailwidth}
\end{equation}
This is closed form, so it introduces no transform table and no interpolation:
each component integrates to one identically and Theorem~\ref{thm:martingale}
applies to the density that is evaluated.  At $J=K$ it is exactly
\eqref{eq:gumbel-spike}, and at $J=1$ it is the two-coordinate profile.

The two extensions parameterize tail concentration through different coordinates,
so the prior is specified as a two-component mixture rather than a product measure.  With
$\pi_J$ uniform on the scale-free ladder $\{1,4,16,\ldots\}\cap[1,K)$ and a
mixing weight $w$, the tail prior is
\begin{equation}
  \pi_{\mathrm{tail}}
  =w\,\pi_\alpha\otimes\delta_{J=K}
  +(1-w)\,\delta_{\alpha=\infty}\otimes\pi_J .
  \label{eq:union-tail}
\end{equation}
Both component priors are exact submodels, obtained at $w=1$ and $w=0$, and the
equal-tail spike is the atom $(\alpha=\infty,J=K)$ in the first block.  The
full-width atom is excluded from $\pi_J$ to avoid duplicate prior mass.  Unless varied in a sensitivity analysis, the
experiments fix $w=.5$, assigning equal probability to the two tail specifications.
A product grid over $(\Delta,\alpha,J)$ is not
evaluated.  Because the two tail coordinates represent overlapping forms of
concentration, a six-point product over $\alpha$ would impose $\log 6$ nats of
component-weight dilution on configurations already represented by the union,
whereas the two-block mixture imposes $\log 2$.  This is a prior-design rationale;
no direct empirical comparison with the product grid is reported.  The archived
samples from \citet{li2025framework}, described in Section~\ref{sec:released-output},
have profile estimates of one or two effective tail coordinates and
substantial full-width misfit, with nominal $p\approx10^{-25}$ for OPT-1.3B and
$p\approx3\times10^{-20}$ for Sheared-LLaMA-2.7B
(Supplementary Section~\ref{sec:released-width}; dependence makes these descriptive
misfit summaries rather than calibrated significance levels).  The equal-tail
synthetic designs and these exploratory width estimates motivate retaining both
components.  The equal mixture is a fixed working choice, not an empirically
optimized weight; Supplementary Section~\ref{sec:union-weight-sweep} evaluates
weight sensitivity on the temperature-matched data.

Because the tail branch is document-shared, integrating over a prior on $w$ is equivalent to fixing $w$ at its prior mean (Supplementary Section~\ref{sec:union-weight-sweep}).

\subsection{Hierarchical generative model}
\label{sec:hierarchy}

Let $f_0$ be the exact Gumbel null density at the relevant vocabulary size and write $f_{\Delta,\alpha}$ for the clean component after any tail vector has been marginalized.  For Gumbel there are three component families: the rules labelled \emph{spike} use \eqref{eq:gumbel-spike}, represented by $\alpha=\infty$; the rules labelled \emph{Dirichlet} use \eqref{eq:gumbel-dirichlet}; and the tail-width rules use \eqref{eq:gumbel-tailwidth}, which the spike family is the $J=K$ member of.  The union tail of \eqref{eq:union-tail} mixes the last two.  The Gumbel component likelihood is exact for a spike NTP.  Introduce a watermark-presence indicator
\begin{equation}
  H\sim\operatorname{Bernoulli}(q_H),
  \qquad
  Y_t\mid H=0,\F_{t-1}\sim f_0.
  \label{eq:hier-null}
\end{equation}
Under $H=1$, the document-shared hierarchy draws one latent pair for the whole sequence and, for the Gumbel Dirichlet component, a fresh residual vector inside the token plate,
\begin{equation}
  \begin{aligned}
    \Delta&\sim\pi_\Delta, &
    \alpha&\sim\pi_\alpha, &
    \bm q_t\mid\alpha&\stackrel{\mathrm{iid}}{\sim}
      \operatorname{Dirichlet}_K(\alpha),
    &t=1,\ldots,n.
  \end{aligned}
  \label{eq:hier-shared}
\end{equation}
whereas the tokenwise hierarchy redraws the pair, followed by its residual vector, at every position,
\begin{equation}
  \begin{aligned}
    (\Delta_t,\alpha_t)&\stackrel{\mathrm{iid}}{\sim}
      \pi_\Delta\otimes\pi_\alpha,
    & \bm q_t\mid\alpha_t&\sim\operatorname{Dirichlet}_K(\alpha_t),
    \qquad\text{independently over }t.
  \end{aligned}
  \label{eq:hier-tokenwise}
\end{equation}
The $\bm q_t$ draw is present only for a Gumbel Dirichlet component; at $\alpha=\infty$ it is the equal vector.  Conditional on a finite $\alpha_{[t]}$, the full-vector Gumbel law is \eqref{eq:gumbel-density} with $\bP_t=(1-\Delta_{[t]},\Delta_{[t]}\bm q_t)$; integrating the independent $\bm q_t$ gives \eqref{eq:gumbel-dirichlet}.

The union tail prior of Section~\ref{sec:tail-width} adds one further latent coordinate.  Let $S\in\{\mathrm{shape},\mathrm{width}\}$ be a document-level tail state with $\Prb(S=\mathrm{shape})=w$, drawn once and independently of $\Delta$, and let
\begin{equation}
  \bm q_t\mid S=\mathrm{shape},\alpha\sim\operatorname{Dirichlet}_K(\alpha),
  \qquad
  \bm q_t=\tfrac1J\ind_J\quad(S=\mathrm{width}),
  \qquad J\sim\pi_J,
  \label{eq:hier-tail-state}
\end{equation}
where $\ind_J$ is the $0/1$ indicator vector of the first $J$ of the $K$ residual coordinates, so $J^{-1}\ind_J$ spreads the tail equally over those $J$ and puts zero on the rest, and $\pi_J$ is the base-four ladder prior of Section~\ref{sec:tail-width}.  This fixed choice of coordinates entails no restriction on the pivot law, which is invariant under permutation of the residual coordinates.  The shape branch fixes $J=K$ and the width branch fixes $\alpha=\infty$, so the two branches are disjoint sub-models of one family and \eqref{eq:union-tail} is the marginal of \eqref{eq:hier-tail-state}.  Write $\vartheta_{[t]}=(S,\eta_{[t]})$ with $\eta_{[t]}=(\Delta_{[t]},\alpha_{[t]},J_{[t]})$, the state $S$ determining which of the last two coordinates is active; the tuple $\eta_{[t]}$ equals the document-level tuple $(\Delta,\alpha,J)$ in the shared hierarchy and the position-specific tuple in the tokenwise hierarchy.  The tail state is document level in every experiment reported here: a rule that redrew $S$ per token is not evaluated.  After marginalizing $\bm q_t$, the clean component density is whichever branch the state selects,
\begin{equation}
  f_{\vartheta}
  =\begin{cases}
    f_{\Delta,\alpha}, & S=\mathrm{shape}\text{, where }J=K,\\
    f_{\Delta,J},      & S=\mathrm{width}\text{, where }\alpha=\infty,
  \end{cases}
  \label{eq:branch-density}
\end{equation}
which for Gumbel are \eqref{eq:gumbel-dirichlet} and \eqref{eq:gumbel-tailwidth}.  The two cases agree at $(\alpha=\infty,J=K)$, where both reduce to the equal-tail spike \eqref{eq:gumbel-spike}, so $f_\vartheta$ is well defined on the whole of $\{\mathrm{shape},\mathrm{width}\}\times\{\eta\}$ as a function; whether the prior charges that point in both branches is a separate question, and under \eqref{eq:union-tail} it does not.  The working alternative treats the observations as conditionally independent with layer
\begin{equation}
  Y_t\mid H=1,\vartheta_{[t]}\sim f_{\vartheta_{[t]}},
  \qquad
  \lambda_{\vartheta_{[t]}}(y)=\frac{f_{\vartheta_{[t]}}(y)}{f_0(y)},
  \label{eq:hier-observation}
\end{equation}
the component density and its likelihood ratio to the exact null.  A deployed
sequence need not have every position coupled to its key;
Supplementary Section~\ref{sec:contamination-supplement} adds a contamination layer that
replaces a position by the null with probability $\rho$, and recovers
\eqref{eq:hier-observation} at $\rho=0$.

\begin{figure}[t]
\centering
\begin{minipage}[t]{0.48\textwidth}
\centering
\textbf{(a) Clean shared parameters}\par\smallskip
\begin{tikzpicture}[font=\small]
  \node[priornode] (pis) at (-1.75,4.6) {$\pi_S$};
  \node[priornode] (pia) at (0,4.6) {$\pi_\alpha$};
  \node[priornode] (pij) at (1.4,4.6) {$\pi_J$};
  \node[priornode] (pid) at (2.8,4.6) {$\pi_\Delta$};
  \node[latentnode] (b) at (-1.75,3.1) {$S$};
  \node[latentnode] (alpha) at (0,3.1) {$\alpha$};
  \node[latentnode] (j) at (1.4,3.1) {$J$};
  \node[latentnode] (delta) at (2.8,3.1) {$\Delta$};
  \node[latentnode] (q) at (0,1.5) {$\bm q_t$};
  \node[detnode] (p) at (1.4,0) {$\bP_t$};
  \node[obsnode] (y) at (2.8,0) {$Y_t$};
  \draw[gmarrow] (pis) -- (b);
  \draw[gmarrow] (pia) -- (alpha);
  \draw[gmarrow] (pij) -- (j);
  \draw[gmarrow] (pid) -- (delta);
  \draw[gmarrow] (b) -- (q);
  \draw[gmarrow] (alpha) -- (q);
  \draw[gmarrow] (j) -- (q);
  \node[font=\scriptsize,anchor=north,fill=white,inner sep=1pt]
    at (alpha.south) {shape};
  \node[font=\scriptsize,anchor=north,fill=white,inner sep=1pt]
    at (j.south) {width};
  \draw[gmarrow] (delta) -- (p);
  \draw[gmarrow] (q) -- (p);
  \draw[gmarrow] (p) -- (y);
  \begin{scope}[on background layer]
    \node[platenode,inner xsep=5mm,inner ysep=4mm,fit=(q)(p)(y)] (plateU) {};
  \end{scope}
  \node[font=\scriptsize,anchor=north east,yshift=-1pt]
    at (plateU.south east) {$t=1{:}n$};
\end{tikzpicture}
\end{minipage}\hfill
\begin{minipage}[t]{0.48\textwidth}
\centering
\textbf{(b) Clean tokenwise parameters}\par\smallskip
\begin{tikzpicture}[font=\small]
  \node[priornode] (pis) at (-1.75,4.6) {$\pi_S$};
  \node[priornode] (pia) at (0,4.6) {$\pi_\alpha$};
  \node[priornode] (pij) at (1.4,4.6) {$\pi_J$};
  \node[priornode] (pid) at (2.8,4.6) {$\pi_\Delta$};
  \node[latentnode] (b) at (-1.75,3.1) {$S$};
  \node[latentnode] (alpha) at (0,3.1) {$\alpha_t$};
  \node[latentnode] (j) at (1.4,3.1) {$J_t$};
  \node[latentnode] (delta) at (2.8,3.1) {$\Delta_t$};
  \node[latentnode] (q) at (0,1.5) {$\bm q_t$};
  \node[detnode] (p) at (1.4,0) {$\bP_t$};
  \node[obsnode] (y) at (2.8,0) {$Y_t$};
  \draw[gmarrow] (pis) -- (b);
  \draw[gmarrow] (pia) -- (alpha);
  \draw[gmarrow] (pij) -- (j);
  \draw[gmarrow] (pid) -- (delta);
  \draw[gmarrow] (b) -- (q);
  \draw[gmarrow] (alpha) -- (q);
  \draw[gmarrow] (j) -- (q);
  \node[font=\scriptsize,anchor=north,fill=white,inner sep=1pt]
    at (alpha.south) {shape};
  \node[font=\scriptsize,anchor=north,fill=white,inner sep=1pt]
    at (j.south) {width};
  \draw[gmarrow] (delta) -- (p);
  \draw[gmarrow] (q) -- (p);
  \draw[gmarrow] (p) -- (y);
  \begin{scope}[on background layer]
    \node[platenode,inner xsep=5mm,inner ysep=5mm,
      fit=(j)(alpha)(delta)(q)(p)(y)] (plateUT) {};
  \end{scope}
  \node[font=\scriptsize,anchor=north east,yshift=-1pt]
    at (plateUT.south east) {$t=1{:}n$};
\end{tikzpicture}
\end{minipage}

\caption{Uncollapsed Gumbel union-tail hierarchies under $H=1$, with every watermarked position coupled to its key.  Rounded rectangles denote fixed prior laws, unshaded circles latent variables, squares deterministic quantities, and shading the observed pivot.  The tail state $S$ of \eqref{eq:hier-tail-state}, with $\Prb(S=\mathrm{shape})=w$, selects the tail-shape or tail-width branch.  The priors $\pi_\alpha$ and $\pi_J$ apply to the active tail coordinate; the shape branch fixes $J=K$ and the width branch fixes $\alpha=\infty$.  These branch restrictions are implicit in the diagrams.  In both panels, $\bP_t=(1-\Delta_{[t]},\Delta_{[t]}\bm q_t)$.  (a) The shared model draws one $(\Delta,\alpha,J)$ per document, so the plate contains only $(\bm q_t,\bP_t,Y_t)$.  (b) Conditional on $S$, the tokenwise model redraws $(\Delta_t,\alpha_t,J_t)$ by position, as specified in \eqref{eq:token-bf-state}.  In both panels, $S$ is document level; no evaluated rule redraws the tail state by token.  These model families are displayed in Figure~\ref{fig:shared}, with selected endpoints in Supplementary Table~\ref{tab:shared}.  Supplementary Section~\ref{sec:contamination-supplement} adds the contamination layer.}
\label{fig:hierarchical-model}
\end{figure}
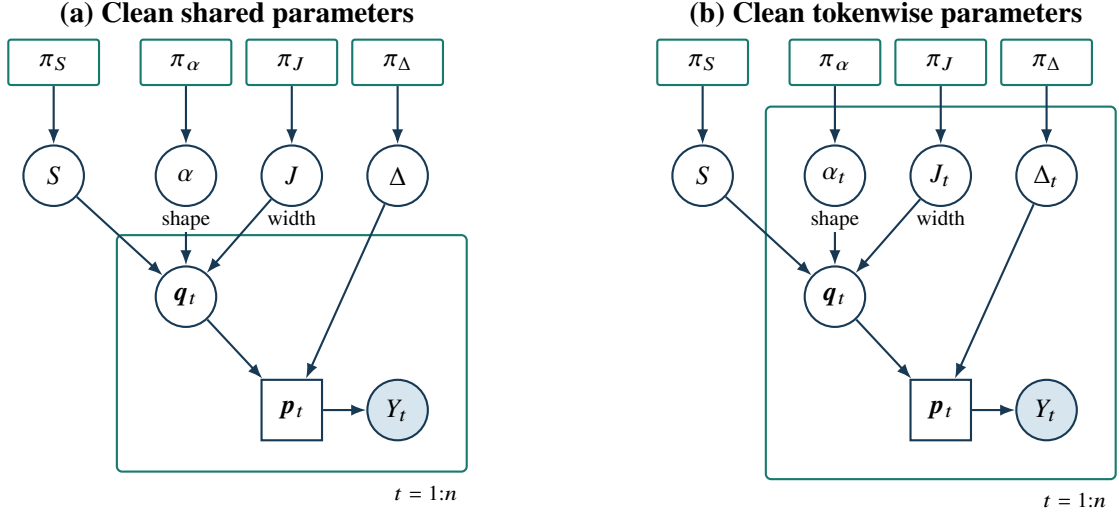

The plate placement in Figure~\ref{fig:hierarchical-model} determines the integration order: the shared model of panel (a) multiplies token likelihoods before integrating its document-level variables, whereas the tokenwise model of panel (b) integrates the position-specific variables before multiplication.  The two hierarchies differ structurally only in the placement of $(\Delta,\alpha,J)$ relative to the token plate.  The spike family has $\pi_\alpha=\delta_\infty$.  A state-space or hidden Markov hierarchy could represent intermediate persistence in $(\Delta_t,\alpha_t,J_t)$.

\FloatBarrier
\subsection{Bayes factors}
\label{sec:bayes-factors}

Let the parameter be $\vartheta=(\Delta,S,\alpha,J)$, where $S\in\{\mathrm{shape},\mathrm{width}\}$ is the tail state of \eqref{eq:hier-tail-state} and the pair $(\alpha,J)$ is the tail coordinate that $S$ makes active.  The prospective prior is
\begin{equation}
  \Pi(d\vartheta)=\pi_\Delta(d\Delta)\,\pi_{\mathrm{tail}}(dS,d\alpha,dJ),
  \qquad
  \pi_{\mathrm{tail}}=w\,\delta_{S=\mathrm{shape}}\otimes\pi_\alpha\otimes\delta_{J=K}
  +(1-w)\,\delta_{S=\mathrm{width}}\otimes\delta_{\alpha=\infty}\otimes\pi_J,
  \label{eq:prospective-prior}
\end{equation}
so $\pi_S(\mathrm{shape})=w$ and the $(\alpha,J)$ marginal of $\pi_{\mathrm{tail}}$ is exactly the union tail of \eqref{eq:union-tail}.  The two branches have disjoint prior support because $\pi_J$ excludes the full width: the equal-tail atom $(\alpha=\infty,J=K)$ belongs only to the shape branch.  The point masses on $S$ can therefore be suppressed in the $(\alpha,J)$ marginal.  The branch densities coincide at $(\alpha=\infty,J=K)$, so $f_\vartheta$ in \eqref{eq:branch-density} is well defined even though the prior assigns that atom only to the shape branch.  The other detector priors can be represented in the same parameter space but are not all endpoints of this fixed union prior.  In particular, the equal-tail spike is obtained by setting $w=1$ and $\pi_\alpha=\delta_\infty$; the inactive width prior $\pi_J$ is then irrelevant.  Setting $w=0$ instead gives the width ladder with $J<K$, while $w=1$ with the full concentration prior gives the Dirichlet layer.  Thus varying $w$ in the fixed union prior interpolates between the width-ladder and shape-layer priors; obtaining the spike alone additionally requires concentrating the shape prior at its equal-tail atom.  For the synthetic experiments, the deficit factor is parameterized as
\begin{equation}
  \Delta=\Delta_{\min}+(\Delta_{\max}-\Delta_{\min})X_\Delta,
  \qquad X_\Delta\sim\operatorname{Beta}(a_\Delta,b_\Delta).
\end{equation}
The synthetic experiments and reanalysis of the archived benchmark data use $\Delta\sim\Unif(0.001,0.5)$.  The contamination study of Supplementary Section~\ref{sec:contamination-supplement} retains that prior and adds one on the contamination rate.  The Dirichlet rules take $\pi_\alpha$ uniform on $\{0.1,1,10,100,1000,\infty\}$, fixed before simulation, so their component index is $\vartheta_g=(S_g,\eta_g)$ with $\eta_g=(\Delta_g,\alpha_g,J_g)$ and the clean grid has $G=96\times6=576$ components.  The support contains five finite concentration levels and the equal-tail limit.  Because the $\alpha$ prior is uniform on its six atoms, the choice of support also determines the weights: four atoms lie at $\alpha\ge10$, so the grid resolves near-equal tails more finely than sparse ones.  A midpoint or Gauss--Legendre grid $\{(\vartheta_g,w_g)\}_{g=1}^G$ makes the prior finite and computation deterministic; $G$ counts grid components throughout and $J$ is reserved for the tail width.  The grid is an exact discrete prior for the computed model and an approximation to a continuous prior.

The hierarchy determines whether the latent component is shared across the document or redrawn by token.  With one latent component shared by the document,
\begin{equation}
  B_n^{\mathrm{sh}}
  =\int\prod_{t=1}^n\lambda_\vartheta(Y_t)\,\Pi(d\vartheta).
  \label{eq:shared-bf}
\end{equation}
With a fresh component $\vartheta_t\stackrel{\mathrm{iid}}{\sim}\Pi$ at every token,
\begin{equation}
  B_n^{\mathrm{tw}}
  =\prod_{t=1}^n\int\lambda_\vartheta(Y_t)\,\Pi(d\vartheta).
  \label{eq:token-bf}
\end{equation}
Equation~\eqref{eq:token-bf} redraws every coordinate of $\vartheta$, including
the tail state $S$.  The evaluated tokenwise hierarchy instead holds $S$ at the
document level.  Writing $\vartheta=(S,\eta)$, where
$\eta=(\Delta,\alpha,J)$ contains the coordinates redrawn by token, its union-tail
Bayes factor is
\begin{equation}
  B_n^{\mathrm{tw},S}
  =\sum_{s}\pi_S(s)\prod_{t=1}^n\int\lambda_{(s,\eta)}(Y_t)\,\Pi(d\eta\mid s),
  \label{eq:token-bf-state}
\end{equation}
a mixture of the two branch products rather than a product of per-token mixtures.
This distinction is numerically material: on 200-token paths the two formulations
differ by up to 26 nats because \eqref{eq:token-bf-state} updates the posterior
probability of the document-level branch.  Both are normalized
nonnegative mixtures of products of $\lambda$, so Theorem~\ref{thm:martingale}
covers either; the implementation computes \eqref{eq:token-bf-state}.

\begin{proposition}[Shared and tokenwise alternatives]\label{prop:hierarchy}
Equation~\eqref{eq:shared-bf} is the Bayes factor for one document-shared $\vartheta\sim\Pi$.  When $B_n^{\mathrm{sh}}>0$, its component posterior is
\begin{equation}
  \Pi_n(d\vartheta)
  =\frac{\prod_{t=1}^n\lambda_\vartheta(Y_t)\Pi(d\vartheta)}{B_n^{\mathrm{sh}}},
\end{equation}
and, when $B_{n-1}^{\mathrm{sh}}>0$, its next multiplier is
\begin{equation}
  \frac{B_n^{\mathrm{sh}}}{B_{n-1}^{\mathrm{sh}}}
  =\int\lambda_\vartheta(Y_n)\Pi_{n-1}(d\vartheta).
\end{equation}
If $B_{n-1}^{\mathrm{sh}}=0$, the Bayes factor remains zero and the within-alternative component posterior is undefined, as handled by Algorithm~\ref{alg:bf} in the Supplementary Materials.  Equation~\eqref{eq:token-bf} is the Bayes factor when $\vartheta_t$ is redrawn independently at every token, and \eqref{eq:token-bf-state} the corresponding factor when the tail state is held at the document level, as it is in every rule reported here.  It does not yield a posterior for one common document-level $\Delta$.  The two marginal likelihoods are generally unequal.
\end{proposition}

\begin{proof}
Integrate the joint likelihood under each hierarchy, then apply Bayes' rule and Fubini's theorem.  The shared update uses the component posterior conditional on the preceding observations.  The fully tokenwise update in \eqref{eq:token-bf} uses the initial prior independently at every position.  For the document-level-state factor \eqref{eq:token-bf-state}, write $m_s(y)=\int\lambda_{(s,\eta)}(y)\,\Pi(d\eta\mid s)$.  Then $B_n^{\mathrm{tw},S}/B_{n-1}^{\mathrm{tw},S}=\sum_s\pi_S(s\mid Y_{1:n-1})m_s(Y_n)$ weights the branches by $\pi_S(s\mid Y_{1:n-1})\propto\pi_S(s)\prod_{t<n}m_s(Y_t)$.  Thus the within-branch coordinates are independently integrated under their initial conditional priors at each position, whereas the posterior distribution of $S$ is updated across positions.
\end{proof}

Bayes factors compare the marginal likelihoods of specified alternative and null models \citep{kass1995bayes}.  The class prior $q_H$ does not enter either Bayes factor; it enters only the posterior odds and the loss-sensitive deployment action.  For watermark-presence prior $q_H=\Prb(H_1)\in(0,1)$, either Bayes factor gives
\begin{equation}
  \Prb(H_1\mid Y_{1:n})
  =\frac{q_HB_n}{1-q_H+q_HB_n},
  \qquad
  \logit\Prb(H_1\mid Y_{1:n})=\logit(q_H)+\log B_n.
  \label{eq:posterior}
\end{equation}
If correct decisions have zero loss, a false positive costs $C_{\mathrm{FP}}$, and a false negative costs $C_{\mathrm{FN}}$, posterior-risk minimization declares a watermark exactly when
\begin{equation}
  B_n>\frac{C_{\mathrm{FP}}}{C_{\mathrm{FN}}}\frac{1-q_H}{q_H}.
  \label{eq:loss-threshold}
\end{equation}
This is a deployment decision rather than an $\alpha$-level test.  The threshold differs from an empirically calibrated fixed-horizon threshold and from the anytime threshold $1/\alpha_{\mathrm{test}}$ in \eqref{eq:ville}.

\subsection{Martingale Bayes factors and e-values}

\begin{theorem}[Sequential pivot Bayes factors are test martingales]\label{thm:martingale}
Assume that under $H_0$, $Y_t\mid\F_{t-1}$ has density $f_{0,t}(\cdot\mid\F_{t-1})$.  Before observing $Y_t$, let $\widetilde f_{1,t}(\cdot\mid\F_{t-1})$ be an $\F_{t-1}$-measurable probability density satisfying $\widetilde f_{1,t}\ll f_{0,t}$.  Define
\begin{equation}
  E_t=\frac{\widetilde f_{1,t}(Y_t\mid\F_{t-1})}
  {f_{0,t}(Y_t\mid\F_{t-1})},
  \qquad B_t=\prod_{s=1}^tE_s,\quad B_0=1.
  \label{eq:predictive-bf}
\end{equation}
Then $(B_t,\F_t)$ is a nonnegative $H_0$-martingale with $\E_0B_t=1$.  Consequently, for every test level $\alpha_{\mathrm{test}}\in(0,1)$,
\begin{equation}
  \Prb_0\left(\sup_{t\ge1}B_t\ge\frac1{\alpha_{\mathrm{test}}}\right)\le\alpha_{\mathrm{test}}.
  \label{eq:ville}
\end{equation}
For every bounded stopping time $\tau$, $\E_0B_\tau=1$; for an almost surely finite unbounded stopping time, $\E_0B_\tau\le1$.
\end{theorem}

\begin{proof}
Conditional normalization gives
\[
  \E_0[E_t\mid\F_{t-1}]
  =\int\frac{\widetilde f_{1,t}(y\mid\F_{t-1})}
  {f_{0,t}(y\mid\F_{t-1})}f_{0,t}(y\mid\F_{t-1})\,dy=1.
\]
Thus $\E_0[B_t\mid\F_{t-1}]=B_{t-1}$.  Ville's inequality yields \eqref{eq:ville}; optional stopping gives the bounded-time identity, while nonnegativity and Fatou's lemma give the unbounded-time inequality.
\end{proof}

For the shared model, $\widetilde f_{1,t}$ is the posterior predictive induced by $\Pi_{t-1}$.  For the fully tokenwise model \eqref{eq:token-bf} it is the fixed prior predictive; for the document-level-state rule \eqref{eq:token-bf-state} it is the two-branch mixture of prior predictives taken at the accumulated state posterior, which is $\F_{t-1}$-measurable and so still predictably selected.  Each $E_t$ is a conditional e-value, and $(B_t)$ is an e-process.  Thus $(B_t)$ is both a Bayes factor under the specified alternative and an e-process under the null \citep{shafer2011test,grunwald2024safe}.  Alternative misspecification can affect posterior calibration and power but not null e-validity, provided the numerator remains a normalized, predictably selected density and the denominator is the true conditional null.  Supplementary Sections~\ref{sec:stable-computation} and~\ref{sec:numerical-checks} describe the interpolated implementation and its numerical validation.

Analytic Gumbel monitoring is exact under Assumption~\textup{(A1)} together with the Gumbel-max key law in Proposition~\ref{prop:gumbel}, when the numerator is normalized and chosen $\F_{t-1}$-measurably.  Tuning the prior or likelihood using the current or future monitored pivots invalidates the proof unless the resulting construction is separately shown to be e-valid; adaptation based only on past pivots remains permissible when it preserves predictable normalization.

\citet[Sections~3.3--3.4]{luo2026rao} transform a Gumbel pivot to $Z_t\sim\operatorname{Exp}(1)$ under $H_0$ and use one-step factors $(1-\lambda)e^{\lambda Z_t}$.  Their surrogate family has density $q_\kappa(z)=(1-\kappa)e^{-(1-\kappa)z}$; setting $\lambda=\kappa$ identifies the one-step factor with $q_\kappa/q_0$.  Their mixture averages fixed-$\lambda$ components with $\lambda$ shared across tokens and can therefore be interpreted as a shared-parameter Bayes factor for that surrogate family.  Here $\Delta$ indexes the $\Delta$-regular NTP family of \citet{li2025framework}; the analysis compares shared and tokenwise hierarchies for the Gumbel pivot.

\subsection{Relationship to the optimal scores of Li et al. (2025)}

Within a fixed one-parameter component family $f_\Delta$, take $\pi_\Delta=\delta_{\Delta_0}$.  Then
\begin{equation}
  \log B_n=\sum_{t=1}^n
  \log\frac{f_{\Delta_0}(Y_t)}{f_0(Y_t)}.
\end{equation}
For Gumbel with the least-favorable component $f_{\Delta_0}=f_{\Delta_0}^{\mathrm{gum},\star}$ of \eqref{eq:gumbel-lf}, this is exactly the likelihood-ratio statistic underlying the optimal score of \citet{li2025framework}.  The equivalence is specific to that family.  The equal-tail baseline evaluated in Section~\ref{sec:experiments} instead uses the spike component \eqref{eq:gumbel-spike}, so its point-prior limit is a different statistic; the empirical consequences are evaluated in Supplementary Section~\ref{sec:family}.  A nondegenerate prior maximizes prior-averaged power for its predictive alternative but is not generally minimax over the $\Delta$-regular class.

Stable log-space evaluation, Algorithm~\ref{alg:bf}, and implementation details are given in Supplementary Section~\ref{sec:stable-computation}.

\section{Numerical Experiments}
\label{sec:experiments}

\subsection{Design and competitors}

Two clean synthetic designs are adapted from \citet{li2025framework}.  The vocabulary has $M=1000$ tokens and, conditional on a deficit $\Delta$, the alternative uses the equal-tail spike NTP $\bP_\Delta^{\mathrm{sp}}$ defined in \eqref{eq:spike-ntp}.  Supplementary Table~\ref{tab:design} gives the common design.  Each fixed-horizon score is calibrated separately using 10,000 exact-null sequences and evaluated on independent samples of 5,000 null and 5,000 alternative sequences.

Score-distribution atoms are randomized to attain 5\% size in the calibration sample.  Reported error rates average over this boundary randomization, and paired comparisons use a common auxiliary uniform for each document.

Parentheses report evaluation-sample Monte Carlo standard errors conditional on the calibrated cutoffs.  Synthetic error rates use plug-in binomial standard errors, except the clean-benchmark indicator rows (Supplementary Tables~\ref{tab:shared} and~\ref{tab:family}), which use the standard deviation of conditional miss contributions divided by the square root of the sample size.  Maximum-regret standard errors use 2,000 paired document-bootstrap replicates.  Supplementary Section~\ref{sec:numerical-checks} specifies the bootstrap design, the fractional-rate variance approximation, and excluded sources of uncertainty.

The Gumbel reference scores evaluated by \citet{li2025framework} are
\begin{equation}
  h_{\mathrm{ars}}(r)=-\log(1-r),\qquad
  h_{\log}(r)=\log r,\qquad
  h_{\mathrm{ind}}(r)=\ind\{r\ge e^{-1}\},
\end{equation}
and the least-favorable log likelihood ratios $h_{\mathrm{gum},\Delta}^\star$ at $\Delta\in\{0.1,0.01,0.005\}$, the tunings displayed across Figures~6--7 of \citet{li2025framework}.  Every fixed-horizon method sums its token scores and uses its own null cutoff.  The Bayesian score is $\log B_n$, also fixed-horizon calibrated for the power comparison.  The uncalibrated anytime rule $\max_{t\le n}B_t\ge20$ is evaluated separately in Supplementary Section~\ref{sec:anytime-monitoring}.  Hereafter, ``reference score'' denotes one of these baseline or minimax scores evaluated by \citet{li2025framework}.

The Gumbel comparison differs in component family.  The reference scores $h_{\mathrm{gum},\Delta}^\star$ are built from the least-favorable family \eqref{eq:gumbel-lf}, whereas the Bayes detector uses the spike family \eqref{eq:gumbel-spike}, which generates the data.  In the shared-$\Delta$ Gumbel scenario, the Bayes model is matched in family, deficit range, and hierarchy, although it still averages over an unobserved deficit.  Supplementary Section~\ref{sec:family} separates the component-family difference from prior averaging.  The numerical specification in \citet{li2025framework} draws $\Delta_t$ independently at each token and uses an equal tail.  By contrast, the accompanying simulation snapshot samples one deficit per simulated document and, at every token, redraws tail probabilities by normalizing independent uniform variates \citep{li2024watermarkframework}.  To isolate deficit persistence from tail-shape variation, both clean designs use the equal tail:
\begin{enumerate}
  \setlength{\itemsep}{0pt}
  \setlength{\parsep}{0pt}
  \item \emph{Shared-deficit equal-tail design:} draw one $\Delta\sim\Unif(0.001,0.5)$ per document and use it at every token.
  \item \emph{Tokenwise-deficit equal-tail design:} draw $\Delta_t\stackrel{\mathrm{iid}}{\sim}\Unif(0.001,0.5)$ at every token.
\end{enumerate}
The shared-deficit design isolates the document-level persistence in the simulation snapshot; the tokenwise-deficit design implements the published latent structure.  Conditional on the deficits, pivots are generated from the exact finite-$M$ watermark and null laws of Section~\ref{sec:pivots}.

\FloatBarrier
\subsection{Shared-deficit equal-tail experiment}
\label{sec:shared-sensitivity}

Figure~\ref{fig:shared} reports fixed-horizon Type~II error for the shared-deficit experiment.  The panel includes equal-tail, Dirichlet tail-shape, and union-tail versions of both hierarchies.  The shared hierarchy integrates the persistent component after multiplying token likelihoods, whereas the tokenwise hierarchy integrates the position-specific parameters before multiplication.  For the union-tail rules, the tail state $S$ remains document-shared in both hierarchies.

\begin{figure}[tbp]
\centering
\includegraphics[width=\textwidth]{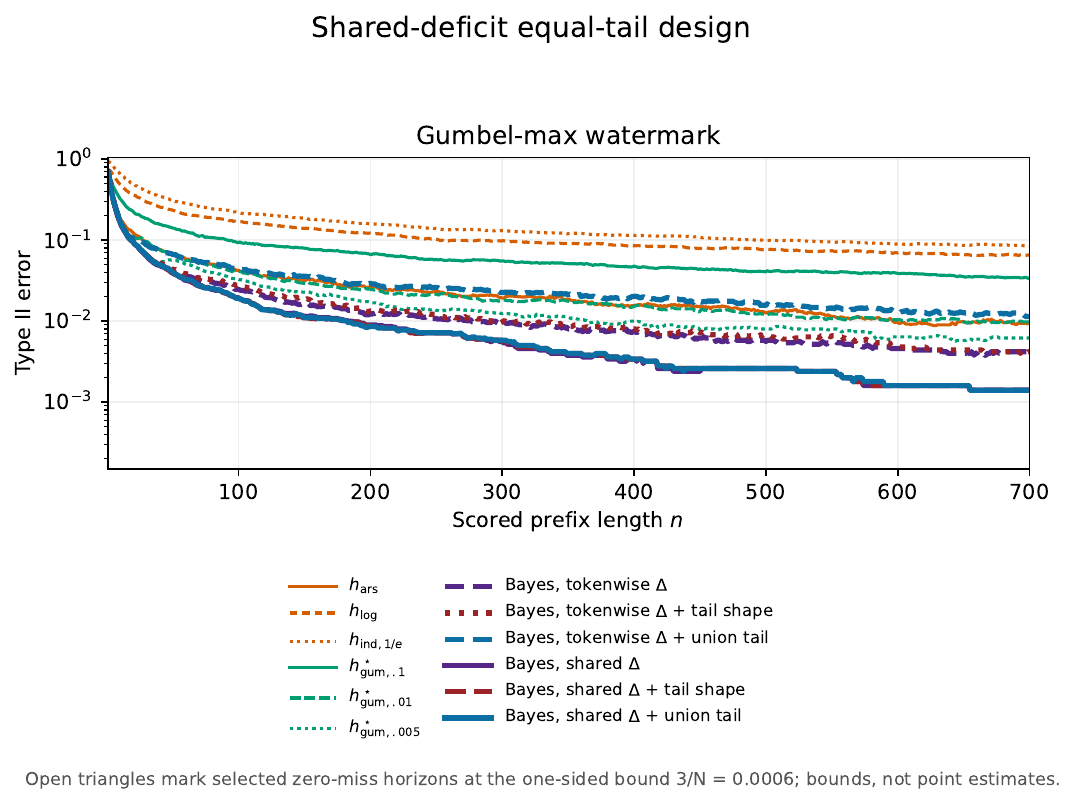}
\caption{Fixed-horizon Type~II error in the shared-deficit equal-tail experiment.  Each rule uses its independently calibrated 5\% null cutoff.  The shared-$\Delta$ and tokenwise-$\Delta$ Bayes labels denote the two hierarchies in Figure~\ref{fig:hierarchical-model}; unqualified labels use equal tails, ``tail shape'' denotes the Dirichlet tail-shape prior, and ``union tail'' uses $w=.5$.  Colour identifies the score or tail family, and line style distinguishes members and hierarchies.  Curves show positive error estimates; open triangles, when present, mark selected zero-miss horizons at the one-sided $3/N$ upper bound and are not connected to the estimate curves.}
\label{fig:shared}
\end{figure}

Supplementary Table~\ref{tab:shared} reports selected horizons.  For Gumbel at $n=700$, the shared Bayes rule misses 7 of 5,000 documents, $0.0014$ (Wilson 95\% interval $[0.0007,0.0029]$), compared with 31 misses, $0.0062$ $[0.0044,0.0088]$, for the best tested Gumbel reference score $h^\star_{\mathrm{gum},.005}$.  The relative reduction is 77.4\%, and all 24 discordant errors favor the shared rule (exact two-sided McNemar $p=1.19\times10^{-7}$).  The respective Type~I estimates are $0.0452$ and $0.0548$.  The tokenwise Bayes error is $0.0042$.

Because the best reference score is selected from the displayed score grid, the paired p-values are descriptive.  Holm correction over all 18 Bayes-versus-reference-score comparisons (six scores at three horizons) leaves the three Gumbel comparisons significant, with adjusted $p=1.4\times10^{-15}$, $2.2\times10^{-9}$, and $1.2\times10^{-7}$ at $n=100,300,700$.  The comparison family, discordance counts, intervals, and adjusted p-values are provided in the companion numerical summaries indexed in Supplementary Section~\ref{sec:computational-materials}.  Supplementary Section~\ref{sec:family} attributes the Gumbel difference to the component family rather than to prior averaging.

The Dirichlet-layer comparison estimates the effect of averaging over tail shape under an equal-tail generator.  At $n=700$, the shared layer and spike rules reject the same 4{,}993 documents.  At $n=100$ and $n=300$, their discordances are 3 versus 3 and 1 versus 0.  Their errors are $.0194/.0058/.0014$ and $.0194/.0056/.0014$, respectively; these counts do not resolve a difference for the shared rules.  At $n=100$, the tokenwise layer rule misses 16 documents caught by the tokenwise spike rule, while the reverse discordance is 2; their errors are $.0274$ and $.0246$ (exact two-sided McNemar $p=.00131$, Holm-adjusted $p=.00787$ across the six layer-versus-spike comparisons).

When $\Delta_t$ is instead redrawn independently at every token, all three tokenwise Bayesian mixtures have no observed misses among 5,000 alternatives by $n=100$; the one-sided 95\% upper bound is $0.000599$.  The complete tokenwise experiment is reported in Supplementary Section~\ref{sec:tokenwise-supplement}.  The comparison depends on the generating hierarchy, and the equal-tail component is matched to both clean generators.

\FloatBarrier
\subsection{Robustness to deficit and tail-shape specifications}
\label{sec:synthetic-sensitivity}

Two further sweeps separate sensitivity to the deficit distribution from
sensitivity to the residual-tail shape.  The first applies the same calibrated
rules to six equal-tail deficit regimes; the second holds the deficit law fixed
and varies six full-width tail laws, including a normalized-uniform law outside
the Dirichlet family.  Priors and scores are not retuned within either sweep, so these comparisons assess robustness of a fixed detector across generating distributions.
Supplementary Sections~\ref{sec:regimes} and~\ref{sec:tails} give the designs,
complete error tables, and paired-bootstrap maximum-regret estimates.

In the deficit sweep at $n=100$, the three published least-favorable tunings
have maximum regrets $.2322$, $.3162$, and $.6756$, compared with $.0036$ for
the shared equal-tail mixture.  The union-tail mixture does not attain the smallest observed maximum regret
at any of the three horizons under these equal-tail generators.  In the tail-shape sweep at
$n=100$, the concentration mixture has maximum regret $.0014$, compared with
$.0046$ for the equal-tail rule, but many individual differences involve only
one to three documents and are not statistically resolved.  These comparisons
are conditional on the generating families and the finite menu of scored rules;
they do not establish minimax guarantees or an ordering of population regrets
among close entries.  Both sweeps hold tail width fixed, motivating the
separate width experiment below.

\subsection{Robustness to tail-width misspecification}
\label{sec:tail-width-sweep}

The deficit and tail-shape sweeps summarized in
Section~\ref{sec:synthetic-sensitivity} hold the tail width at $M-1$.
In particular, Supplementary Table~\ref{tab:tails} varies tail shape without
varying its support, so it does not evaluate the width component of
\eqref{eq:union-tail} under width misspecification.  This distinction matters
because Section~\ref{sec:temperature-matched} attributes essentially the entire
matched-data improvement within the Bayesian family to that component.

Table~\ref{tab:tail-widths} varies this additional axis.  The deficit is placed
equally on $J$ coordinates and zero on the rest, the width branch of
\eqref{eq:hier-tail-state}; $J=M-1$ recovers the equal tail exactly.
The width experiment therefore complements the full-width tail-shape sweep.

For narrow generating widths, the union tail has lower Type~II error than the
equal-tail layer that it generalizes.  At $J=1$ and $n=700$, the respective
errors are $.0266$ and $.0726$; at $J=4$, they are $.0090$ and $.0652$.  Under
the same width misspecification, the hierarchical tail-shape model has error
$.4934$ at $J=1$ and $n=700$, nearly seven times the equal-tail layer's error,
and the single-shape layer at $\alpha_0=0.1$ has error $.7472$, the largest entry
in the table.  Thus, averaging over dispersed tail shapes does not accommodate a
tail whose support is narrow rather than merely uneven.

Across the four widths, the six shape laws of Supplementary Table~\ref{tab:tails}, and all
three horizons, the union tail has maximum regret $.0080$, compared with $.0962$
for the equal-tail layer, $.4668$ for the shape mixture, and $.5486$, $.3716$,
and $.0792$ for the three published tunings of $h^\star_{\mathrm{gum}}$.  The
variation across these three tunings exceeds the difference between the
best-performing tuning and $h_{\mathrm{ars}}$, whose maximum regret is $.0288$.
The corresponding values are $.1518$ for $h_{\log}$ and $.2082$ for
$h_{\mathrm{ind},1/e}$.  The maximum regret of $h_{\mathrm{ars}}$ is about 3.5
times that of the union tail.  This finite-grid regret reduction is evidence
of robustness to residual-tail misspecification.  The archived benchmark data
yield estimates of one or two effective tail coordinates, bracketed by the W1
and W2 columns.

\begin{table}[p]
\centering
\scriptsize
\renewcommand{\arraystretch}{0.9}
\setlength{\tabcolsep}{3pt}
\setlength{\aboverulesep}{0pt}
\setlength{\belowrulesep}{0pt}
\setlength{\extrarowheight}{0pt}

\caption{Tail-\emph{width} sweep for the shared-$\Delta$ Gumbel generator.  The deficit is spread equally over $J$ tail coordinates and is zero on the rest: W1--W4 have $J=1,4,16,64$.  W1 and W2 bracket the effective width of one or two coordinates that Supplementary Section~\ref{sec:released-width} fits to the archived benchmark data of \citet{li2025framework}.  At $J=M-1$ the law is exactly the equal tail of Supplementary Table~\ref{tab:tails}, so these regimes extend that sweep along the one axis it holds fixed.  Type~II entries carry binomial Monte Carlo standard errors; maximum regret is the largest excess over the best rule scored in the sweep, across the four widths, and its standard error is a paired document bootstrap over 2000 replicates, since a maximum of differences is not a binomial proportion.  Bold marks observed column minima, ties included.  Tail shape is the Dirichlet layer \eqref{eq:gumbel-dirichlet} at concentration $\alpha$; the union tail \eqref{eq:union-tail} mixes it with the width ladder \eqref{eq:gumbel-tailwidth}.  Fixed-deficit equal-tail scores are treated as diagnostics and excluded from the regret benchmark; they need not match the generating deficit or tail family.  Table~\ref{tab:family} gives a clean-benchmark illustration; the computational-materials index in Supplementary Section~\ref{sec:computational-materials} identifies the full diagnostics for this sweep.  The $h^\star_{\mathrm{gum},\Delta_0}$ rows are the least-favorable score of \citet{li2025framework} at its three published tunings.}
\label{tab:tail-widths}
\begin{tabular}{@{}lrrrrcr@{}}
\toprule
& \multicolumn{4}{c}{Type~II error by generating tail width} & & Max \\
\cmidrule(lr){2-5}\cmidrule(lr){7-7}
Rule & W1 & W2 & W3 & W4 & & regret \\
\midrule
\multicolumn{7}{@{}l}{\textit{Horizon $n=100$}}\\
Tail shape, $\alpha=\infty$ (Table~\ref{tab:shared} rule) & \mcse{.1954}{.0056} & \mcse{.1638}{.0052} & \mcse{.1136}{.0045} & \mcse{.0392}{.0027} & & \mcse{.0962}{.0043} \\
Tail shape, $\alpha_0=1000$ & \mcse{.1954}{.0056} & \mcse{.1640}{.0052} & \mcse{.1136}{.0045} & \mcse{.0392}{.0027} & & \mcse{.0964}{.0043} \\
Tail shape, $\alpha_0=100$ & \mcse{.1956}{.0056} & \mcse{.1640}{.0052} & \mcse{.1132}{.0045} & \mcse{.0378}{.0027} & & \mcse{.0964}{.0043} \\
Tail shape, $\alpha_0=10$ & \mcse{.1980}{.0056} & \mcse{.1630}{.0052} & \mcse{.1068}{.0044} & \mcse{.0326}{.0025} & & \mcse{.0954}{.0042} \\
Tail shape, $\alpha_0=1$ & \mcse{.2316}{.0060} & \mcse{.1762}{.0054} & \mcse{.0776}{.0038} & \mcse{.0228}{.0021} & & \mcse{.1086}{.0043} \\
Tail shape, $\alpha_0=0.1$ & \mcse{.7444}{.0062} & \mcse{.3188}{.0066} & \mcse{.0454}{.0029} & \bestmcse{.0202}{.0020} & & \mcse{.6126}{.0074} \\
Bayes, shared $\Delta$ $+$ tail shape & \mcse{.3786}{.0069} & \mcse{.2072}{.0057} & \mcse{.0440}{.0029} & \mcse{.0204}{.0020} & & \mcse{.2468}{.0066} \\
Bayes, shared $\Delta$ $+$ union tail & \mcse{.1398}{.0049} & \bestmcse{.0676}{.0036} & \mcse{.0368}{.0027} & \mcse{.0262}{.0023} & & \bestmcse{.0080}{.0022} \\
\addlinespace[1pt]
$h_{\mathrm{ars}}$ & \mcse{.1372}{.0049} & \mcse{.0774}{.0038} & \mcse{.0582}{.0033} & \mcse{.0490}{.0031} & & \mcse{.0288}{.0024} \\
$h_{\log}$ & \mcse{.1774}{.0054} & \mcse{.1676}{.0053} & \mcse{.1626}{.0052} & \mcse{.1548}{.0051} & & \mcse{.1346}{.0045} \\
$h_{\mathrm{ind},1/e}$ & \mcse{.2182}{.0058} & \mcse{.2242}{.0059} & \mcse{.2137}{.0058} & \mcse{.2112}{.0058} & & \mcse{.1910}{.0052} \\
$h^\star_{\mathrm{gum},.1}$ & \bestmcse{.1318}{.0048} & \mcse{.0916}{.0041} & \mcse{.0886}{.0040} & \mcse{.0830}{.0039} & & \mcse{.0628}{.0034} \\
$h^\star_{\mathrm{gum},.01}$ & \mcse{.5034}{.0071} & \mcse{.0762}{.0038} & \mcse{.0384}{.0027} & \mcse{.0380}{.0027} & & \mcse{.3716}{.0075} \\
$h^\star_{\mathrm{gum},.005}$ & \mcse{.6804}{.0066} & \mcse{.1656}{.0053} & \bestmcse{.0346}{.0026} & \mcse{.0308}{.0024} & & \mcse{.5486}{.0077} \\
\addlinespace[1pt]
\multicolumn{7}{@{}l}{\textit{Horizon $n=300$}}\\
Tail shape, $\alpha=\infty$ (Table~\ref{tab:shared} rule) & \mcse{.1136}{.0045} & \mcse{.0994}{.0042} & \mcse{.0646}{.0035} & \mcse{.0130}{.0016} & & \mcse{.0766}{.0038} \\
Tail shape, $\alpha_0=1000$ & \mcse{.1136}{.0045} & \mcse{.0994}{.0042} & \mcse{.0644}{.0035} & \mcse{.0130}{.0016} & & \mcse{.0766}{.0038} \\
Tail shape, $\alpha_0=100$ & \mcse{.1134}{.0045} & \mcse{.0994}{.0042} & \mcse{.0644}{.0035} & \mcse{.0124}{.0016} & & \mcse{.0766}{.0038} \\
Tail shape, $\alpha_0=10$ & \mcse{.1170}{.0045} & \mcse{.0994}{.0042} & \mcse{.0632}{.0034} & \mcse{.0110}{.0015} & & \mcse{.0766}{.0038} \\
Tail shape, $\alpha_0=1$ & \mcse{.1642}{.0052} & \mcse{.1238}{.0047} & \mcse{.0538}{.0032} & \mcse{.0060}{.0011} & & \mcse{.1070}{.0042} \\
Tail shape, $\alpha_0=0.1$ & \mcse{.7280}{.0063} & \mcse{.2362}{.0060} & \mcse{.0186}{.0019} & \bestmcse{.0044}{.0009} & & \mcse{.6708}{.0068} \\
Bayes, shared $\Delta$ $+$ tail shape & \mcse{.3596}{.0068} & \mcse{.1670}{.0053} & \mcse{.0190}{.0019} & \bestmcse{.0044}{.0009} & & \mcse{.3024}{.0066} \\
Bayes, shared $\Delta$ $+$ union tail & \mcse{.0588}{.0033} & \mcse{.0258}{.0022} & \bestmcse{.0118}{.0015} & \mcse{.0076}{.0012} & & \bestmcse{.0032}{.0008} \\
\addlinespace[1pt]
$h_{\mathrm{ars}}$ & \bestmcse{.0572}{.0033} & \mcse{.0348}{.0026} & \mcse{.0272}{.0023} & \mcse{.0276}{.0023} & & \mcse{.0232}{.0021} \\
$h_{\log}$ & \mcse{.0974}{.0042} & \mcse{.0990}{.0042} & \mcse{.0954}{.0042} & \mcse{.0950}{.0041} & & \mcse{.0906}{.0037} \\
$h_{\mathrm{ind},1/e}$ & \mcse{.1238}{.0047} & \mcse{.1274}{.0047} & \mcse{.1290}{.0047} & \mcse{.1230}{.0046} & & \mcse{.1186}{.0038} \\
$h^\star_{\mathrm{gum},.1}$ & \mcse{.0590}{.0033} & \mcse{.0516}{.0031} & \mcse{.0522}{.0031} & \mcse{.0508}{.0031} & & \mcse{.0464}{.0028} \\
$h^\star_{\mathrm{gum},.01}$ & \mcse{.1438}{.0050} & \bestmcse{.0228}{.0021} & \mcse{.0176}{.0019} & \mcse{.0178}{.0019} & & \mcse{.0866}{.0045} \\
$h^\star_{\mathrm{gum},.005}$ & \mcse{.3292}{.0066} & \mcse{.0242}{.0022} & \mcse{.0124}{.0016} & \mcse{.0146}{.0017} & & \mcse{.2720}{.0065} \\
\addlinespace[1pt]
\multicolumn{7}{@{}l}{\textit{Horizon $n=700$}}\\
Tail shape, $\alpha=\infty$ (Table~\ref{tab:shared} rule) & \mcse{.0726}{.0037} & \mcse{.0652}{.0035} & \mcse{.0368}{.0027} & \mcse{.0082}{.0013} & & \mcse{.0580}{.0033} \\
Tail shape, $\alpha_0=1000$ & \mcse{.0726}{.0037} & \mcse{.0652}{.0035} & \mcse{.0368}{.0027} & \mcse{.0082}{.0013} & & \mcse{.0580}{.0033} \\
Tail shape, $\alpha_0=100$ & \mcse{.0730}{.0037} & \mcse{.0648}{.0035} & \mcse{.0362}{.0026} & \mcse{.0082}{.0013} & & \mcse{.0576}{.0033} \\
Tail shape, $\alpha_0=10$ & \mcse{.0764}{.0038} & \mcse{.0666}{.0035} & \mcse{.0350}{.0026} & \mcse{.0068}{.0012} & & \mcse{.0594}{.0034} \\
Tail shape, $\alpha_0=1$ & \mcse{.2478}{.0061} & \mcse{.1850}{.0055} & \mcse{.0566}{.0033} & \mcse{.0032}{.0008} & & \mcse{.2212}{.0059} \\
Tail shape, $\alpha_0=0.1$ & \mcse{.7472}{.0061} & \mcse{.2114}{.0058} & \mcse{.0082}{.0013} & \bestmcse{.0024}{.0007} & & \mcse{.7206}{.0064} \\
Bayes, shared $\Delta$ $+$ tail shape & \mcse{.4934}{.0071} & \mcse{.1850}{.0055} & \mcse{.0082}{.0013} & \bestmcse{.0024}{.0007} & & \mcse{.4668}{.0072} \\
Bayes, shared $\Delta$ $+$ union tail & \bestmcse{.0266}{.0023} & \mcse{.0090}{.0013} & \bestmcse{.0042}{.0009} & \mcse{.0038}{.0009} & & \bestmcse{.0018}{.0006} \\
\addlinespace[1pt]
$h_{\mathrm{ars}}$ & \mcse{.0272}{.0023} & \mcse{.0208}{.0020} & \mcse{.0142}{.0017} & \mcse{.0148}{.0017} & & \mcse{.0136}{.0014} \\
$h_{\log}$ & \mcse{.0608}{.0034} & \mcse{.0582}{.0033} & \mcse{.0618}{.0034} & \mcse{.0594}{.0033} & & \mcse{.0576}{.0027} \\
$h_{\mathrm{ind},1/e}$ & \mcse{.0817}{.0039} & \mcse{.0779}{.0038} & \mcse{.0804}{.0038} & \mcse{.0794}{.0038} & & \mcse{.0770}{.0031} \\
$h^\star_{\mathrm{gum},.1}$ & \mcse{.0314}{.0025} & \mcse{.0304}{.0024} & \mcse{.0290}{.0024} & \mcse{.0320}{.0025} & & \mcse{.0296}{.0022} \\
$h^\star_{\mathrm{gum},.01}$ & \mcse{.0312}{.0025} & \mcse{.0094}{.0014} & \mcse{.0092}{.0014} & \mcse{.0102}{.0014} & & \mcse{.0078}{.0012} \\
$h^\star_{\mathrm{gum},.005}$ & \mcse{.0906}{.0041} & \bestmcse{.0072}{.0012} & \mcse{.0070}{.0012} & \mcse{.0076}{.0012} & & \mcse{.0640}{.0036} \\
\bottomrule
\end{tabular}
\end{table}

\FloatBarrier
\subsection{Hierarchical extensions}
\label{sec:hierarchical-summary}

Supplementary Section~\ref{sec:hierarchical-extensions} develops hierarchical
models in which token-specific deficits or tail widths are conditionally
independent given document-level hyperparameters.  Integrating over those
hyperparameters induces marginal dependence and partial pooling.  These models
include the corresponding shared and independent specifications as special
cases; the width extension retains a document-shared deficit.

In the fixed-marginal persistence experiments, hierarchical and shared models
differ in estimated Type~II error by at most $.0012$ for deficits and $.0038$
for widths.  After the reported Holm adjustments, no hierarchical-versus-shared
deficit contrast is significant; two width contrasts are significant, both
under independent token-specific widths.  The nonsignificant contrasts do not
establish equivalence.  On the temperature-matched data of
Section~\ref{sec:temperature-matched}, the hierarchical deficit model increases
estimated AUC over the shared-deficit model by $.0004$--$.0031$, with mean
$.0015$.  These AUC comparisons are descriptive and outside the Holm family.
The observed changes from additional hierarchical dependence are smaller than
those from extending residual-tail support in the evaluated designs.

\FloatBarrier
\subsection{Archived benchmark data of Li et al.}
\label{sec:released-output}

We first reanalyze the benchmark data accompanying \citet{li2025framework},
\emph{A Statistical Framework of Watermarks for Large Language Models:
Pivot, Detection Efficiency and Optimal Rules}.  Their \emph{WatermarkFramework}
GitHub repository \citep{li2024watermarkframework}, at the commit specified in
that reference, archives the data and reference implementation used here.
For each of OPT-1.3B \citep{zhang2022opt} and Sheared-LLaMA-2.7B
\citep{xia2023sheared}, the data contain 500 watermarked and 500 unwatermarked
continuations sharing the same C4 prompts \citep{raffel2020exploring}.
These are pre-existing benchmark data, not the temperature-matched data
generated for this study in Section~\ref{sec:temperature-matched}.
The watermarked arm uses temperature $.1$, whereas
the unwatermarked generator samples from unscaled logits, corresponding to
temperature one.  A statistic based only on the number of distinct pivot values
has AUC $.997$--$.998$ and Type~II error $.004$--$.012$ on these samples.
Thus, the comparison on these archived data conflates watermark sensitivity with sensitivity
to low-temperature repetition and does not identify a matched-temperature
ranking of detection power.

These benchmark data also provide exploratory evidence against the full-width
tail specification.  Profiling the tail-width likelihood on first-use seed
positions gives effective widths $J=2$ for OPT-1.3B and $J=1$ for
Sheared-LLaMA-2.7B.  These estimates motivate allowing narrow tails, but they
neither establish the adequacy of the equal-tail width family nor imply that
only one or two residual coordinates have nonzero probability.  The reference
keyed sampler of Li et al. can reuse seeds, so the conditional-null assumption required for
anytime validity is not verified.  Supplementary
Section~\ref{sec:released-analysis} reports our comparisons on these archived data,
repetition diagnostics, width estimates, and replay checks; the analysis is
restricted to fixed-horizon empirical errors.  Section~\ref{sec:temperature-matched}
regenerates both arms at a common temperature to remove the temperature confound.

\FloatBarrier

\subsection{Temperature-matched regeneration}
\label{sec:temperature-matched}

Both arms are regenerated at the same temperature for each model, using the
key, sampler, and pivot functions from the reference implementation of Li et al. \citep{li2024watermarkframework}.
The unwatermarked generator is modified to divide its logits by the same
temperature as the watermarked generator.  Temperatures are
$\{.1,.2,.3,.4,.5,.6,.7,1\}$, with 200 scored tokens at $.3$ and below and 100
above.  Each arm contains 2500 documents per model--temperature cell in
$[.2,.5]$ and 500 elsewhere.  The larger sample combines the 500 original benchmark
prompts with 2000 additional prompts from the same C4
\texttt{realnewslike} configuration under the same selection rule.
Supplementary Section~\ref{sec:matched-design-details} gives the prompt
construction and null-pivot replay details.

At each prefix, the cutoff is estimated from the full unwatermarked arm without using watermarked scores.  The empirical quantile rule equalizes realized size when the relevant null order statistics are untied; continuity of a score formula alone does not exclude ties in finite-precision pivots.  Type~I error is estimated by five-fold cross-fitting, and area under the curve is also reported as a cutoff-free estimand.  Cross-fitted Type~I error averages $.0505$ across the 152 continuous-rule-by-cell estimates (19 rules in eight cells), ranging $.0476$ to $.0528$, and averages $.0410$ for $h_{\mathrm{ind},1/e}$, whose count statistic is conservative under the strict inequality used to reject.

Table~\ref{tab:temperature-matched} gives the endpoint at 100 tokens; the hierarchical deficit comparator is specified in Supplementary Section~\ref{sec:deficit-pooling}.  Temperatures $.2$ to $.5$ define the non-saturated comparison window: at $.1$ no rule reaches power $.20$ at 100 tokens, and by $.7$ every rule detects nearly all alternatives.  Averaging the Type~II errors over the four temperatures within each model yields the lowest mean for $h_{\mathrm{ars}}$, followed by $h^\star_{\mathrm{gum},.1}$ and the union-tail mixture.  This mean ordering does not hold in every individual cell.

For cutoff-free inference, we test the corresponding AUC contrasts rather than differences in Type~II error at a single cutoff.  Within each model, the four temperature cells share a model-specific prompt set, and within each cell the null and watermarked arms are paired by prompt.  The two models need not share identical prompts, but the within-model overlap already makes the eight signs dependent, so no independent-cell sign test is used.  The analytic variance of \citet{delong1988comparing} accounts for the correlation between rules scored on the same documents, but not for that pairing between the two class samples, so inference is taken from a joint prompt-cluster bootstrap: each of $2{,}000$ replicates resamples prompt indices once and applies the same draw to both arms and every rule, preserving the pairing and the between-rule correlation together.  Reported $p$-values are normal-tail values computed from the bootstrap standard error, not empirical bootstrap-tail probabilities; the percentile probabilities are recorded alongside them.  Significance is Holm-adjusted across the forty-eight reported contrasts; the full numerical summaries are included in the companion repository indexed in Supplementary Section~\ref{sec:computational-materials}.  Supplementary Section~\ref{sec:matched-inference-details} gives further bootstrap reporting details and the secondary 500-document comparison.

The bootstrap conditions on the specified rules and captures prompt-sampling variability, but not the uncertainty associated with selecting a rule.  The tail-width family and the equal mixing weight were motivated in Section~\ref{sec:tail-width} partly by the archived benchmark data of Li et al., whose 500 prompts are the first 500 of the 2500 used here.  Thus, one fifth of the analysis sample also contributed to model development.  The analysis is repeated on the 2000 extension prompts alone, excluding all prompts used in that development; the bootstrap procedure and Holm family are unchanged apart from the reduction from 2500 to 2000 documents per cell.   On the holdout, the union tail has higher AUC than the shared-$\Delta$ rule in all eight cells, with differences from $.0096$ to $.0568$, and each contrast is significant.  The width-only prior has higher AUC in all eight cells, with differences from $.0096$ to $.0560$, and $w=0$ has higher AUC than $w=1$ in all eight cells, with differences from $.0098$ to $.0477$; all sixteen contrasts are significant.  Every one of these twenty-four contrasts has a Holm-adjusted $p$ below $1.2\times10^{-7}$.  The effect estimates remain similar, and all twenty-four contrasts remain significant despite wider intervals; twenty point estimates increase and four decrease slightly when the reused prompts are removed.  The two comparisons that are marginal in the full sample remain marginal: the shape-only prior has a positive contrast in six rather than seven cells and a significant contrast in two rather than three.  The $h_{\mathrm{ars}}$ versus union-tail contrast is positive in all eight holdout cells rather than seven, remains significant in five, and ranges from $.0003$ to $.0113$.  Because no rule is re-estimated on the holdout, this analysis addresses overlap in prompts but not reuse of the model-development rationale.  The estimated tail-width improvements nevertheless persist without overlapping prompts.

Sweeping the mixing weight separates the two model components of \eqref{eq:union-tail}.  The width-only prior $w=0$ has higher AUC than the shared-$\Delta$ rule in all eight cells, with significant differences ranging from $.0091$ at the most saturated cell to $.0550$ at the least.  The shape-only prior $w=1$ has higher AUC in seven of eight cells, but only three differences are significant and the largest is $.0076$.  Comparing the endpoints, $w=0$ has higher AUC than $w=1$ in all eight cells, with significant differences from $.0091$ to $.0475$.  Thus, the tail-width component accounts for essentially the entire estimated improvement within the Bayesian family on these data, consistent with the fitted width of one on Sheared-LLaMA-2.7B and two on OPT-1.3B reported in Supplementary Section~\ref{sec:released-width}.

Table~\ref{tab:union-weight} evaluates an eleven-point grid in $w$, and Supplementary Section~\ref{sec:union-weight-sweep} reports the broad region over which the estimated AUC is nearly constant.  The evidence supports assigning positive prior mass to the tail-width component but does not identify a precise interior value of $w$.

The AUC of $h_{\mathrm{ars}}$ exceeds that of $h^\star_{\mathrm{gum},.1}$ in all eight cells, but the contrast is significant in only three, all on Sheared-LLaMA, and its magnitude never exceeds $.006$.  For $h_{\mathrm{ars}}$ versus the union-tail mixture, the contrast favors $h_{\mathrm{ars}}$ in seven of eight cells and is significant in five.  In the remaining cell, OPT-1.3B at temperature $.5$, the union-tail AUC is higher, although both AUCs round to $.9963$.

The variation within the $h^\star_{\mathrm{gum}}$ family demonstrates sensitivity to the tuning constant.  On OPT-1.3B at temperature $.3$, the three tuning constants of \citet{li2025framework} give Type~II errors of $.2268$, $.3932$, and $.4904$.  Selection among these constants would require an external criterion.  Averaging over $\Delta$ alone does not reproduce the lowest error: the plain shared-$\Delta$ rule has error $.4396$ in this cell, between the middle and largest errors among the three tunings.  The union-tail error is $.2512$, exceeding the lowest tuned error by $.0244$ and below the middle and largest tuned errors by $.1420$ and $.2392$.  This rule averages over both $\Delta$ and the enlarged tail family.

\begin{table}[htbp]
\centering
\scriptsize
\setlength{\tabcolsep}{2.2pt}
\caption{Temperature-matched regeneration: Type~II error at 100 tokens, with both arms generated at the temperature shown.  Each cell uses 2500 documents per arm.  Cutoffs are taken at each prefix from the whole unwatermarked arm, which equalizes realized size when the relevant null order statistics are untied; cross-fitted Type~I error averages $.0505$, ranging $.0476$--$.0528$ across the 152 continuous-rule-by-cell estimates (19 rules in eight cells), against a mean of $.0410$ and a range of $.0360$--$.0464$ across the eight cells for $h_{\mathrm{ind},1/e}$, whose statistic is a count and is left conservative by the strict inequality used to reject.  Columns outside $[.2,.5]$ are omitted: at temperature $.1$ every rule has Type~II error above $.80$ at that horizon, so the thresholded errors are weakly separated, and by $.7$ all have saturated.  The union-tail model sweeps the mixing weight $w$ of \eqref{eq:union-tail}: $w=1$ is the Dirichlet shape prior at full width, $w=0$ the tail-width prior at $\alpha=\infty$, and both endpoints are exact sub-models of the union.  Boldface marks the lowest Type~II error in each column.}
\label{tab:temperature-matched}
\begin{tabular}{@{}lrrrrrrrr@{}}
\toprule
& \multicolumn{4}{c}{OPT-1.3B} & \multicolumn{4}{c}{Sheared-LLaMA-2.7B} \\
\cmidrule(lr){2-5}\cmidrule(lr){6-9}
Rule & $.2$ & $.3$ & $.4$ & $.5$ & $.2$ & $.3$ & $.4$ & $.5$ \\
\midrule
$h_{\mathrm{ars}}$ & \bestmcse{.5420}{.0100} & \bestmcse{.2128}{.0082} & \bestmcse{.0460}{.0042} & \mcse{.0108}{.0021} & \bestmcse{.5340}{.0100} & \bestmcse{.2508}{.0087} & \mcse{.0676}{.0050} & \mcse{.0196}{.0028} \\
$h^\star_{\mathrm{gum},.1}$ & \mcse{.5616}{.0099} & \mcse{.2268}{.0084} & \mcse{.0564}{.0046} & \mcse{.0124}{.0022} & \mcse{.5568}{.0099} & \mcse{.2684}{.0089} & \mcse{.0672}{.0050} & \mcse{.0232}{.0030} \\
$h^\star_{\mathrm{gum},.01}$ & \mcse{.6948}{.0092} & \mcse{.3932}{.0098} & \mcse{.1332}{.0068} & \mcse{.0440}{.0041} & \mcse{.7244}{.0089} & \mcse{.4428}{.0099} & \mcse{.1952}{.0079} & \mcse{.0644}{.0049} \\
$h^\star_{\mathrm{gum},.005}$ & \mcse{.7260}{.0089} & \mcse{.4904}{.0100} & \mcse{.2228}{.0083} & \mcse{.0812}{.0055} & \mcse{.7628}{.0085} & \mcse{.5160}{.0100} & \mcse{.2740}{.0089} & \mcse{.1048}{.0061} \\
$h_{\log}$ & \mcse{.7340}{.0088} & \mcse{.4420}{.0099} & \mcse{.1568}{.0073} & \mcse{.0440}{.0041} & \mcse{.6880}{.0093} & \mcse{.4204}{.0099} & \mcse{.1700}{.0075} & \mcse{.0628}{.0049} \\
$h_{\mathrm{ind},1/e}$ & \mcse{.8112}{.0078} & \mcse{.5704}{.0099} & \mcse{.3452}{.0095} & \mcse{.0996}{.0060} & \mcse{.7692}{.0084} & \mcse{.5332}{.0100} & \mcse{.3264}{.0094} & \mcse{.1420}{.0070} \\
\addlinespace
Bayes, shared $\Delta$ (equal tail) & \mcse{.7320}{.0089} & \mcse{.4396}{.0099} & \mcse{.1528}{.0072} & \mcse{.0428}{.0040} & \mcse{.6836}{.0093} & \mcse{.4176}{.0099} & \mcse{.1664}{.0074} & \mcse{.0616}{.0048} \\
\quad hierarchical $\Delta$ (equal tail) & \mcse{.7336}{.0088} & \mcse{.4336}{.0099} & \mcse{.1456}{.0071} & \mcse{.0408}{.0040} & \mcse{.6816}{.0093} & \mcse{.4132}{.0098} & \mcse{.1564}{.0073} & \mcse{.0584}{.0047} \\
\quad $+$ tail shape only ($w=1$) & \mcse{.7252}{.0089} & \mcse{.4352}{.0099} & \mcse{.1424}{.0070} & \mcse{.0384}{.0038} & \mcse{.6808}{.0093} & \mcse{.4024}{.0098} & \mcse{.1656}{.0074} & \mcse{.0524}{.0045} \\
\quad $+$ union tail, $w=.9$ & \mcse{.6256}{.0097} & \mcse{.2508}{.0087} & \mcse{.0536}{.0045} & \mcse{.0092}{.0019} & \mcse{.6336}{.0096} & \mcse{.2760}{.0089} & \mcse{.0668}{.0050} & \bestmcse{.0160}{.0025} \\
\quad $+$ union tail, $w=.5$ & \mcse{.6176}{.0097} & \mcse{.2512}{.0087} & \mcse{.0544}{.0045} & \bestmcse{.0088}{.0019} & \mcse{.6268}{.0097} & \mcse{.2736}{.0089} & \mcse{.0664}{.0050} & \mcse{.0168}{.0026} \\
\quad $+$ union tail, $w=.1$ & \mcse{.6176}{.0097} & \mcse{.2512}{.0087} & \mcse{.0544}{.0045} & \bestmcse{.0088}{.0019} & \mcse{.6268}{.0097} & \mcse{.2740}{.0089} & \bestmcse{.0656}{.0050} & \mcse{.0168}{.0026} \\
\quad $+$ tail width only ($w=0$) & \mcse{.6176}{.0097} & \mcse{.2512}{.0087} & \mcse{.0544}{.0045} & \bestmcse{.0088}{.0019} & \mcse{.6272}{.0097} & \mcse{.2740}{.0089} & \bestmcse{.0656}{.0050} & \mcse{.0168}{.0026} \\
\bottomrule
\end{tabular}
\end{table}

The complete temperature-by-prefix curves for both models and the accompanying
repetition diagnostics are given in Supplementary
Section~\ref{sec:temperature-prefix-details}.

Sampling temperature changes the empirical deficit distribution and the fraction
of tokens outside the prior support $[.001,.5]$; Supplementary
Section~\ref{sec:deficit-regime} reports these summaries and
Supplementary Section~\ref{sec:deficit-support} examines a wider prior.
These observations concern the two evaluated models under the specified decoder,
not typical deployment conditions or other models and decoding rules.

\FloatBarrier
\section{Conclusion}

This manuscript develops the pivotal watermark framework of \citet{li2025framework}
toward robust detection under uncertain signal strength and residual-tail
structure, by integrating the alternative likelihood over these quantities.
A further extension averages over a contamination rate in
Supplementary Section~\ref{sec:contamination-supplement}.  The hierarchical extensions
in Supplementary Section~\ref{sec:hierarchical-extensions} specify token-level variation conditional on document-level hyperparameters, inducing
partial pooling through their marginal likelihoods.  The resulting Bayes
factor quantifies evidence under the specified models; posterior probabilities
additionally require prior odds, and optimal decisions depend on the loss
function.  Under the exact conditional null, normalized predictive densities
selected before each observation yield a test martingale.  Ville's inequality
then gives an anytime-valid threshold of $1/\alpha_{\mathrm{test}}$.  An
interpolated numerator certified to have mass at most one instead yields a
test supermartingale with the same bound; the implemented interpolation has
not been certified in this manner.

In the evaluated designs, residual-tail specification has a larger effect on
power than the additional hierarchical dependence structure.  On the
matched-temperature data, the union-tail model has higher AUC than the
shared-deficit equal-tail model in all eight cells of the non-saturated window,
with each difference significant after multiplicity adjustment.  In the
synthetic persistence experiments of Supplementary
Section~\ref{sec:hierarchical-extensions}, hierarchical extensions of the deficit or
width model change estimated Type~II error by at most $.0038$.  None of the
eighteen paired hierarchical-versus-shared deficit comparisons is significant;
two of the fifteen width comparisons are significant at the $5\%$ level, both
under independent token-specific widths.  Averaging over the deficit alone
does not account for the improvement from the union-tail specification.  In
the cell where the three published tunings of $h^\star_{\mathrm{gum}}$ give
Type~II errors from $.2268$ to $.4904$, the equal-tail mixture gives $.4396$,
whereas the union-tail model gives $.2512$.

\Needspace{3\baselineskip}
The principal robustness result concerns performance across alternative
specifications, rather than superiority at every alternative.  Across the six
full-width tail laws, four narrow-width regimes, and three horizons, the
union-tail model has maximum regret $.0080$, compared with $.0962$ for the
equal-tail Bayesian model.  The three published tunings of
$h^\star_{\mathrm{gum}}$ also give substantially different errors in the example
above.  These findings support using a fixed mixture across uncertain
alternatives, but constitute neither a minimax guarantee over a larger class nor
uniform dominance.  On the matched-temperature data, $h_{\mathrm{ars}}$ has
higher AUC than the union-tail model in seven of eight cells, with five
differences significant after adjustment.

Interpretation depends on the sampling design.  The archived watermarked and
unwatermarked benchmark samples of Li et al. have different temperatures and can be distinguished by
repetition statistics alone.  They provide evidence of residual-tail
misspecification but do not identify detector rankings under matched
generation conditions.  Regenerating both samples at the same temperature
removes this confounding; temperatures $.2$--$.5$ provide the informative
comparison range, with limited power below this range and near-perfect
discrimination above it.  Similarly, Supplementary Section~\ref{sec:deficit-pooling}
shows that the hierarchical deficit model's lower
estimated error under a Beta generative design is not reproduced under the
fixed-marginal copula design, which separates changes in persistence from
changes in the marginal deficit distribution.

The synthetic contamination experiment provides a separate robustness check: at 700
tokens with $60\%$ independent null-like replacement, the clean and
contamination-aware shared equal-tail rules have Type~II errors $.0050$ and
$.0056$, respectively.  This shows retained power under the evaluated dilution
mechanism, not an advantage from modeling contamination or protection against
arbitrary edits.  All empirical comparisons are conditional on the specified
grids and Monte Carlo samples.  Several leading rules differ by only one to
three misclassified documents; nonsignificant differences do not establish
equivalence, and no equivalence margin was prespecified.

\Needspace{6\baselineskip}
\section*{Computational materials}

Machine-readable numerical summaries, prefix-curve tables, document-level
rejection indicators, analysis code, and the stored temperature-matched
generation arrays and prompt tables accompany this technical report in the
companion repository:
\url{https://github.com/MaStatLab/PLRWatermark/}.
Supplementary Section~\ref{sec:computational-materials} and the accompanying
README identify their contents, sources, and reproducibility limitations.

\Needspace{6\baselineskip}
\section*{Acknowledgment}

The AI systems GPT-5.6 Sol, GPT-6 Astra, and Claude Opus 5 were used under the author's direction to implement the numerical experiments and draft substantial portions of the manuscript.  The author reviewed and edited the generated material.

\bibliographystyle{plainnat}
\bibliography{references}

\clearpage
\appendix
\section{Supplementary Materials}
\label{sec:supplementary}

\setcounter{equation}{0}
\renewcommand{\theequation}{S\arabic{equation}}
\renewcommand{\theHequation}{S\arabic{equation}}
\setcounter{algorithm}{0}
\renewcommand{\thealgorithm}{S\arabic{algorithm}}

\FloatBarrier
\subsection{Contamination}
\label{sec:contamination-supplement}

The working alternative of Section~\ref{sec:hierarchy} assumes that every
watermarked position is coupled to its key.  A deployed sequence may instead
include positions generated under an independently chosen key or replaced by
unwatermarked text.  This section specifies a contamination model for these
departures, gives its graphical representation in
Figure~\ref{fig:contamination-model}, and evaluates its effect on detection error.
The main-text analyses use the clean model of Section~\ref{sec:hierarchy}; it is
recovered here when the contamination probability is zero.

\subsubsection{The contamination layer}
\label{sec:contamination-model}

Let $C_t\mid\rho_{[t]}\sim\operatorname{Bernoulli}(\rho_{[t]})$ independently
over $t$, with $\rho$ document level in the shared hierarchy and $\rho_t$
position specific in the tokenwise one.  Then
\begin{equation}
  Y_t\mid H=1,\vartheta_{[t]},C_t
  \sim
  \begin{cases}
    f_0, & C_t=1,\\
        f_{\vartheta_{[t]}}, & C_t=0.
  \end{cases}
  \label{eq:contam-observation}
\end{equation}
Marginalizing $C_t$ gives the component density and likelihood ratio
\begin{align}
  g_{\vartheta_{[t]}}(y)
  &=\rho_{[t]}f_0(y)
    +(1-\rho_{[t]})f_{\vartheta_{[t]}}(y),\\
  \lambda_{\vartheta_{[t]}}(y)
  &=\frac{g_{\vartheta_{[t]}}(y)}{f_0(y)}
  =\rho_{[t]}+(1-\rho_{[t]})
    \frac{f_{\vartheta_{[t]}}(y)}{f_0(y)}.
  \label{eq:contam-lr}
\end{align}
The contamination model applies to positions generated with an independently chosen nonmatching key or replaced by null-like pivots.  It excludes burst edits, context-dependent paraphrases, and insertion/deletion alignment, and $\rho$ need not equal the realized edit fraction.

Equation~\eqref{eq:contam-lr} specializes Huber's gross-error model by fixing the contaminating law to the exact pivot null \citep{huber1964robust}; the indicators $C_t$ use the Bayesian latent good/bad allocation idea of \citet{box1968bayesian}.  With $\rho=1-\varepsilon$, the marginal law equals the watermark-edit mixture of \citet{li2024robust}.  That work uses a frequentist goodness-of-fit test; the present model specifies a prior on $\rho$ and integrates it jointly with $\Delta$.

For fixed $\vartheta=(S,\eta)$, write $L_t=f_{\vartheta}(Y_t)/f_0(Y_t)$.  The contaminated one-step factor is
\begin{equation}
  e_t(\vartheta)=\rho+(1-\rho)L_t,
  \qquad
  \E_0\!\left[e_t(\vartheta)\mid\F_{t-1}\right]=1.
  \label{eq:contamination-evalue}
\end{equation}
Products of these factors, and nonnegative prior mixtures of the products, are therefore test martingales under the exact null.  This is exact-null e-validity, not Huber robustness of the null.  Huber's clipped likelihood-ratio test \citep{huber1965robust} and the contamination-robust e-processes of \citet{saha2026huber} seek Type~I control over neighborhoods of $H_0$.  The present mixture instead specifies an $H_1$ model for independent dilution toward $f_0$; any power gain depends on agreement between this mixture and the generating mechanism.

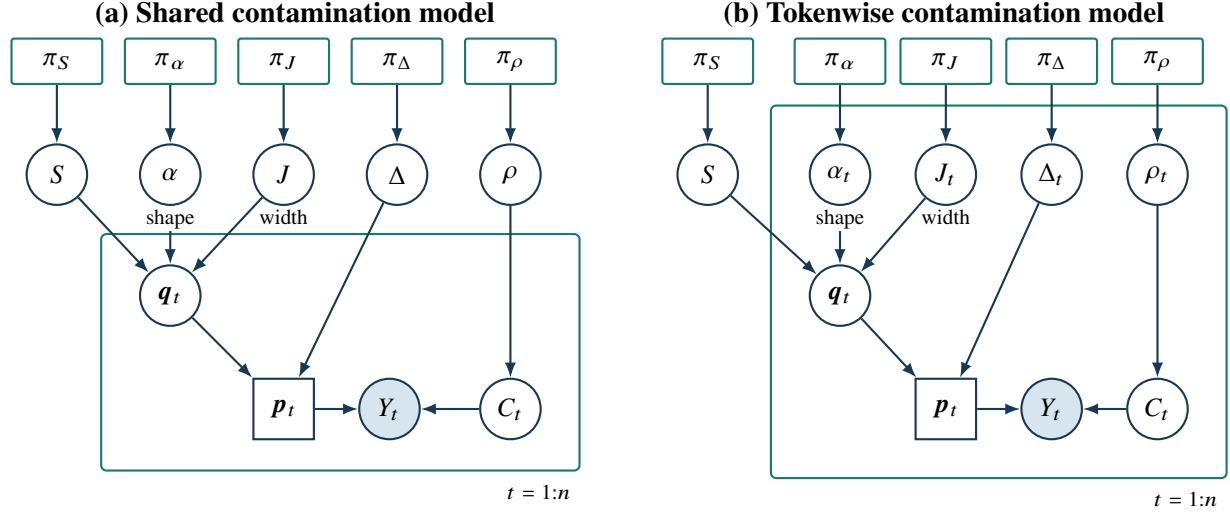
\begin{figure}[t]
\centering
\begin{minipage}[t]{0.48\textwidth}
\centering
\textbf{(a) Shared contamination model}\par\smallskip
\begin{tikzpicture}[font=\small]
  \node[priornode] (pis) at (0,4.6) {$\pi_S$};
  \node[priornode] (pia) at (1.5,4.6) {$\pi_\alpha$};
  \node[priornode] (pij) at (3.0,4.6) {$\pi_J$};
  \node[priornode] (pid) at (4.5,4.6) {$\pi_\Delta$};
  \node[priornode] (pir) at (6.0,4.6) {$\pi_\rho$};
  \node[latentnode] (b) at (0,3.1) {$S$};
  \node[latentnode] (alpha) at (1.5,3.1) {$\alpha$};
  \node[latentnode] (j) at (3.0,3.1) {$J$};
  \node[latentnode] (delta) at (4.5,3.1) {$\Delta$};
  \node[latentnode] (rho) at (6.0,3.1) {$\rho$};
  \node[latentnode] (q) at (1.5,1.5) {$\bm q_t$};
  \node[detnode] (p) at (3.0,0) {$\bP_t$};
  \node[obsnode] (y) at (4.4,0) {$Y_t$};
  \node[latentnode] (c) at (6.0,0) {$C_t$};
  \draw[gmarrow] (pis) -- (b);
  \draw[gmarrow] (pia) -- (alpha);
  \draw[gmarrow] (pij) -- (j);
  \draw[gmarrow] (pid) -- (delta);
  \draw[gmarrow] (pir) -- (rho);
  \draw[gmarrow] (b) -- (q);
  \draw[gmarrow] (alpha) -- (q);
  \draw[gmarrow] (j) -- (q);
  \node[font=\scriptsize,anchor=north,fill=white,inner sep=1pt]
    at (alpha.south) {shape};
  \node[font=\scriptsize,anchor=north,fill=white,inner sep=1pt]
    at (j.south) {width};
  \draw[gmarrow] (delta) -- (p);
  \draw[gmarrow] (q) -- (p);
  \draw[gmarrow] (p) -- (y);
  \draw[gmarrow] (rho) -- (c);
  \draw[gmarrow] (c) -- (y);
  \begin{scope}[on background layer]
    \node[platenode,inner xsep=5mm,inner ysep=4mm,fit=(q)(p)(y)(c)] (plateS) {};
  \end{scope}
  \node[font=\scriptsize,anchor=north east,yshift=-1pt]
    at (plateS.south east) {$t=1{:}n$};
\end{tikzpicture}
\end{minipage}\hfill
\begin{minipage}[t]{0.48\textwidth}
\centering
\textbf{(b) Tokenwise contamination model}\par\smallskip
\begin{tikzpicture}[font=\small]
  \node[priornode] (pis) at (-1.75,4.6) {$\pi_S$};
  \node[priornode] (pia) at (0,4.6) {$\pi_\alpha$};
  \node[priornode] (pij) at (1.4,4.6) {$\pi_J$};
  \node[priornode] (pid) at (2.8,4.6) {$\pi_\Delta$};
  \node[priornode] (pir) at (4.2,4.6) {$\pi_\rho$};
  \node[latentnode] (b) at (-1.75,3.1) {$S$};
  \node[latentnode] (alpha) at (0,3.1) {$\alpha_t$};
  \node[latentnode] (j) at (1.4,3.1) {$J_t$};
  \node[latentnode] (delta) at (2.8,3.1) {$\Delta_t$};
  \node[latentnode] (rho) at (4.2,3.1) {$\rho_t$};
  \node[latentnode] (q) at (0,1.5) {$\bm q_t$};
  \node[detnode] (p) at (1.4,0) {$\bP_t$};
  \node[obsnode] (y) at (2.8,0) {$Y_t$};
  \node[latentnode] (c) at (4.2,0) {$C_t$};
  \draw[gmarrow] (pis) -- (b);
  \draw[gmarrow] (pia) -- (alpha);
  \draw[gmarrow] (pij) -- (j);
  \draw[gmarrow] (pid) -- (delta);
  \draw[gmarrow] (pir) -- (rho);
  \draw[gmarrow] (b) -- (q);
  \draw[gmarrow] (alpha) -- (q);
  \draw[gmarrow] (j) -- (q);
  \node[font=\scriptsize,anchor=north,fill=white,inner sep=1pt]
    at (alpha.south) {shape};
  \node[font=\scriptsize,anchor=north,fill=white,inner sep=1pt]
    at (j.south) {width};
  \draw[gmarrow] (delta) -- (p);
  \draw[gmarrow] (q) -- (p);
  \draw[gmarrow] (p) -- (y);
  \draw[gmarrow] (rho) -- (c);
  \draw[gmarrow] (c) -- (y);
  \begin{scope}[on background layer]
    \node[platenode,inner xsep=5mm,inner ysep=5mm,
      fit=(j)(delta)(alpha)(rho)(q)(p)(y)(c)] (plateT) {};
  \end{scope}
  \node[font=\scriptsize,anchor=north east,yshift=-1pt]
    at (plateT.south east) {$t=1{:}n$};
\end{tikzpicture}
\end{minipage}
\caption{Contamination hierarchies extending Figure~\ref{fig:hierarchical-model} by the probabilities $\rho_{[t]}$ and allocation indicators $C_t$.  The tail priors and implicit branch restrictions have the same meaning as in that figure, and $S$ remains document-shared in both panels.  (a) The shared model draws one $(\Delta,\alpha,J,\rho)$ per document.  (b) Conditional on $S$, the tokenwise model redraws $(\Delta_t,\alpha_t,J_t,\rho_t)$ by position.  In both, $C_t=1$ selects the exact-null component, with probability $\rho$ in (a) and $\rho_t$ in (b).  Setting every $\rho_{[t]}=0$ makes $C_t=0$ almost surely and recovers the corresponding clean hierarchy.}
\label{fig:contamination-model}
\end{figure}

\begin{remark}[Tokenwise contamination prior]
\label{rem:tokenwise-rho}
Within the shape branch, consider the tokenwise product prior $\Pi=\pi_\Delta\otimes\pi_\rho\otimes\pi_\alpha$ and let
$\bar\rho=\int\rho\,\pi_\rho(d\rho)$.  Linearity of the contamination mixture gives
\begin{equation}
  \iiint g_{\Delta,\rho,\alpha}(y)\,
  \pi_\Delta(d\Delta)\pi_\rho(d\rho)\pi_\alpha(d\alpha)
  =\bar\rho f_0(y)
   +(1-\bar\rho)\iint f_{\Delta,\alpha}(y)\,
   \pi_\Delta(d\Delta)\pi_\alpha(d\alpha).
  \label{eq:token-rho-collapse}
\end{equation}
The same calculation holds within the width branch, replacing the concentration integral by an integral over $\pi_J$.  With the same independent contamination prior in both branches, their tokenwise products and the subsequent average over the document-level state $S$ depend on $\pi_\rho$ only through its mean.  Under this hierarchy there is no document-level contamination rate to estimate.  A model with shared $\rho$ and tokenwise $\Delta_t$ defines a separate hybrid hierarchy.

This reduction differs from that in Remark~\ref{rem:union-weight}, where $w$
mixes whole-document branch priors.  Here the position-specific contamination
probability is integrated within each one-token density before the densities are
multiplied.  Both reductions follow from linearity, but at different levels of
the hierarchy.
\end{remark}

\subsubsection{Independent null-like contamination}
\label{sec:contamination}

The main-text comparisons use the clean model, corresponding here to a data-generating contamination rate of $\rho_\star=0$.  The following experiment evaluates the independent replacement mechanism in \eqref{eq:contam-observation}.  It applies a prespecified prior on $\rho$ to the marginal mixture studied by \citet{li2024robust}, integrating over $\rho$ in the likelihood.  Within each contamination stratum, $\rho_\star$ is fixed across positions.  Each shared-$\Delta$ document has a clean watermarked path $Y^{(1)}_{1:n}$, an independent exact-null replacement path $Y^{(0)}_{1:n}$, and uniforms $U_{1:n}$.  The observed path is
\begin{equation}
  C_t(\rho_\star)=\ind\{U_t<\rho_\star\},
  \qquad
  Y_t=(1-C_t)Y_t^{(1)}+C_tY_t^{(0)}.
  \label{eq:contamination-experiment}
\end{equation}
The generating rate is $\rho_\star\in\{0,.1,.25,.4,.6,1\}$.  Shared clean paths, replacements, and nested masks pair methods and contamination levels.  The $\rho_\star=1$ cell is a diagnostic: its alternative is exactly the null, so power must agree with the evaluation Type~I rate up to Monte Carlo error.

The contamination-aware shared Bayes rule combines $\Delta\sim\Unif(.001,.5)$ with the prespecified prior
\begin{equation}
  \pi_\rho=.5\delta_0+.125(\delta_{.1}+\delta_{.25}+\delta_{.4}+\delta_{.6}).
  \label{eq:rho-prior-experiment}
\end{equation}
Half the prior mass is allocated to the exact clean model, paralleling the equal probability assigned to the two tail specifications in \eqref{eq:union-tail}.  Comparators include shared Bayes fixed at $\rho=0$, the reference scores evaluated by \citet{li2025framework}, tokenwise variants, and the Dirichlet-layer versions of both $\rho$ treatments.  Under a tokenwise prior, the marginal likelihood depends on $\rho$ only through $\E\rho$ and does not estimate a persistent document-level rate.  One independent exact-null sample supplies method-specific cutoffs reused throughout the $\rho_\star$ sweep.  \begin{figure}[H]
\centering
\includegraphics[width=\textwidth]{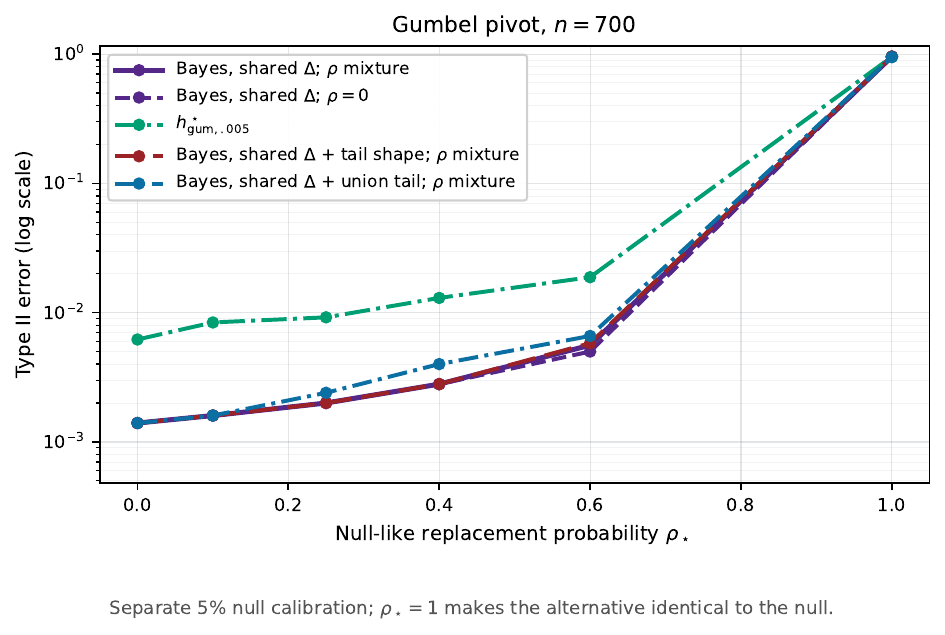}
\caption{Fixed-horizon Type~II error at $n=700$ under independent null-like pivot replacement with probability $\rho_\star$.  The contamination-aware shared-$\Delta$ Bayes rules, labelled ``$\rho$ mixture,'' average over \eqref{eq:rho-prior-experiment}; the clean equal-tail rule fixes $\rho=0$.  The additional tail-shape and union-tail curves use the Dirichlet tail-shape prior and the union-tail prior with $w=.5$, respectively.  Colours identify the same families as in Figure~\ref{fig:shared}; line styles distinguish the overlapping Bayesian curves.  The reference score $h^\star_{\mathrm{gum},.005}$ of \citet{li2025framework} is selected once using the independent clean benchmark and held fixed across the sweep.  One common exact-null sample supplies method-specific nominal 5\% cutoffs.}
\label{fig:contamination}
\end{figure}

Figure~\ref{fig:contamination} displays the full sweep, and Table~\ref{tab:contamination} reports selected cells.  At $\rho_\star=.4$, contamination-aware and clean shared Bayes both have error $.0028$, compared with $.0130$ for $h^\star_{\mathrm{gum},.005}$, the lowest reference-score error in the displayed grid at that cell.  At $\rho_\star=.6$ the corresponding errors are $.0056$, $.0050$, and $.0188$.  Adding the Dirichlet tail prior leaves the Gumbel Bayes rows identical to four decimals for $\rho_\star\le.4$ and produces a spread of at most $.0008$ at $\rho_\star=.6$.  The union-tail error exceeds that of the rule it generalizes by $.0004$ at $\rho_\star=.25$ and $.0012$ at $\rho_\star=.4$ because the generating paths are equal tailed and the tail-width block is misspecified for this experiment.

\begin{table}[htbp]
\centering
\small
\caption{Type~II error at $n=700$ under iid null-like replacement.  Parentheses give Monte Carlo standard errors.  All six prespecified reference scores of \citet{li2025framework} are reported; $h^\star_{\mathrm{gum},.005}$ is the comparator selected before the sweep for Figure~\ref{fig:contamination}.  The contamination-aware tail-shape row uses a $96\times6\times5$ $(\Delta,\alpha,\rho)$ grid; the $\rho=0$ Dirichlet row uses a $96\times6$ $(\Delta,\alpha)$ grid.  }
\label{tab:contamination}
\begin{tabular}{@{}llrrrr@{}}
\toprule
Scheme & Method & $\rho_\star=0$ & $.25$ & $.40$ & $.60$ \\
\midrule
Gumbel & Bayes, shared $\Delta$ ($\rho$ mixture) & \mcse{.0014}{.0005} & \mcse{.0020}{.0006} & \mcse{.0028}{.0007} & \mcse{.0056}{.0011} \\
 & Bayes, shared $\Delta$ $+$ tail shape ($\rho$ mixture) & \mcse{.0014}{.0005} & \mcse{.0020}{.0006} & \mcse{.0028}{.0007} & \mcse{.0058}{.0011} \\
 & Bayes, shared $\Delta$ $+$ union tail ($\rho$ mixture) & \mcse{.0014}{.0005} & \mcse{.0024}{.0007} & \mcse{.0040}{.0009} & \mcse{.0066}{.0011} \\
 & Bayes, shared $\Delta$ ($\rho=0$) & \mcse{.0014}{.0005} & \mcse{.0020}{.0006} & \mcse{.0028}{.0007} & \mcse{.0050}{.0010} \\
 & Bayes, shared $\Delta$ $+$ tail shape ($\rho=0$) & \mcse{.0014}{.0005} & \mcse{.0020}{.0006} & \mcse{.0028}{.0007} & \mcse{.0058}{.0011} \\
 & Bayes, shared $\Delta$ $+$ union tail ($\rho=0$) & \mcse{.0014}{.0005} & \mcse{.0024}{.0007} & \mcse{.0038}{.0009} & \mcse{.0064}{.0011} \\
 & $h^\star_{\mathrm{gum},.1}$ & \mcse{.0304}{.0024} & \mcse{.0418}{.0028} & \mcse{.0518}{.0031} & \mcse{.0808}{.0039} \\
 & $h^\star_{\mathrm{gum},.01}$ & \mcse{.0086}{.0013} & \mcse{.0130}{.0016} & \mcse{.0172}{.0018} & \mcse{.0258}{.0022} \\
 & $h^\star_{\mathrm{gum},.005}$ & \mcse{.0062}{.0011} & \mcse{.0092}{.0014} & \mcse{.0130}{.0016} & \mcse{.0188}{.0019} \\
 & $h_{\mathrm{ars}}$ & \mcse{.0094}{.0014} & \mcse{.0134}{.0016} & \mcse{.0184}{.0019} & \mcse{.0288}{.0024} \\
 & $h_{\log}$ & \mcse{.0622}{.0034} & \mcse{.0840}{.0039} & \mcse{.1112}{.0044} & \mcse{.1596}{.0052} \\
 & $h_{\mathrm{ind},1/e}$ & \mcse{.0814}{.0039} & \mcse{.1135}{.0045} & \mcse{.1403}{.0049} & \mcse{.2116}{.0058} \\
\addlinespace
\bottomrule
\end{tabular}
\end{table}

The clean shared Gumbel rule's Type~II error increases from $.0014$ at $\rho_\star=0$ to $.0050$ at $\rho_\star=.6$, but the clean and contamination-aware rules do not differ detectably.  The Gumbel component \eqref{eq:gumbel-spike} is strictly positive on $(0,1)$, so a null-like observation cannot set the Bayes factor to zero.  The per-token loss is unbounded --- $f^{\mathrm{sp}}_\Delta(r)\sim r^{\Delta/(1-\Delta)}$ as $r\downarrow0$, so $\log f^{\mathrm{sp}}_\Delta(r)\to-\infty$ --- but its expectation under the null, $D_0(\Delta)$, is finite, satisfying the integrability condition for the drift calculation.  Writing $D_1(\Delta)$ for the component-versus-null Kullback--Leibler divergence and $D_0(\Delta)$ for its reverse, a fixed-$\Delta$ clean rule has positive expected drift while $\rho_\star<D_1/(D_1+D_0)$.  At $M=1000$, this threshold decreases from $.83$ at $\Delta=.005$ to $.73$ at $\Delta=.5$, remaining above the largest nondegenerate evaluated contamination rate, $\rho_\star=.6$.  Under contamination, the mixture posterior shifts toward smaller deficits, whose componentwise thresholds are higher.  The shared clean-$\rho$ mixture's mean terminal log Bayes factor at $n=700$ decreases from $898$ at $\rho_\star=0$ to $270$ at $\rho_\star=.6$.  In an auxiliary drift diagnostic at $\rho_\star=.8$, the mean terminal log Bayes factor is $85$; it is negative at the null-equivalent value $\rho_\star=1$.  At $\rho_\star=.6$, its error is $.0050$, compared with $.0060$ for the same rule applied to 280 clean tokens in the independent benchmark.  The clean and contamination-aware Gumbel rules produce identical decisions on all 5,000 documents for $\rho_\star\le.4$ and differ on five at $\rho_\star=.6$.

A finite-grid prior-sensitivity check compares $\pi_\rho$ with point-mass priors.  Replacing $\pi_\rho$ by a point mass at $\rho_0\in\{.1,.15,.25,.4,.5,.6\}$, with the $\Delta$ prior, hierarchy, paths, and calibration unchanged, gives Gumbel mean Type~II errors between $.0026$ and $.0027$ over the contaminated levels $\rho_\star<1$.  Across those five cells, the mixture and $\rho_0=.5$ agree exactly, both with mean $.00268$ and maximum $.0056$.  A $\rho_0$ selected retrospectively at each contamination level has lower observed error than the mixture by $0$ to $.0004$, but $\rho_\star$ is unknown when the rule is selected.  Among the seven tested $\rho$ rules, the mixture's empirical maximum regret over $\rho_\star<1$ is $.0004$, fourth lowest; the six point choices range from $0$ to $.0004$.  The fully contaminated cell is excluded from these comparisons because every rule is then applied to the null; the resulting variation reflects calibration and Monte Carlo error rather than contamination power and is analyzed separately below.  These selected maxima include calibration, Monte Carlo, and best-in-cell uncertainty and do not optimize over continuous $\rho_0$.  The computational-materials index in Section~\ref{sec:computational-materials} identifies the complete contamination-prior diagnostics.

At $\rho_\star=1$, the largest absolute difference between power and the corresponding evaluation Type~I rate among the methods in the mixture-prior contamination experiment is $.0066$, or about 1.5 estimated Monte Carlo standard errors.  Including the point-mass-$\rho$ diagnostic rules, the largest difference is $.0084$.  Both comparisons are consistent with null equivalence at the resolution of these simulations.  The experiment evaluates only independent null-like replacement; it does not cover burst edits, context-dependent paraphrase, key desynchronization, or insertion/deletion alignment.

\FloatBarrier
\subsection{Attribution of the Gumbel difference}
\label{sec:family}

At $\Delta=0.2$ and $M=1000$, the alternative-to-null divergence $D(f_1\Vert f_0)$, which is the mean log likelihood-ratio growth under the alternative, is $1.028$ nats for \eqref{eq:gumbel-spike} and $0.077$ nats for \eqref{eq:gumbel-lf}.  The null-to-alternative divergence relevant to fixed-size Type~II error exponents is $0.245$ and $0.081$ nats, respectively.  At these parameter values, both divergences are larger for the spike member, although neither determines finite-horizon power.

The comparison in Table~\ref{tab:shared} confounds prior averaging with the component family: the Bayes rule averages over $\Pi$ instead of fixing $\Delta$, and it evaluates the spike density \eqref{eq:gumbel-spike} instead of the least-favorable density \eqref{eq:gumbel-lf} that defines $h^\star_{\mathrm{gum},\Delta}$.  To separate these changes, define the point-$\Delta_0$ log likelihood ratio under the generating family by
\begin{equation}
  h^{\mathrm{sp}}_{\Delta_0}(r)
  =\log f_{\Delta_0}^{\mathrm{gum},\mathrm{sp}}(r).
  \label{eq:spike-point}
\end{equation}
This diagnostic uses no prior, posterior updating, or hierarchy.  It is not a score from \citet{li2025framework} and is excluded when identifying the best reference score.

\begin{table}[htbp]
\centering
\small
\caption{Separating the likelihood family from prior averaging, shared-$\Delta$ sensitivity experiment, Type~II error at nominal 5\% size.  Parentheses give Monte Carlo standard errors.  The diagnostic scores fix a single $\Delta_0$ but use the generating spike family.  The final block evaluates Dirichlet tail-shape and union-tail mixtures under the equal-tail generator, separating tail-family averaging from the equal-tail baseline.  Tail shape is the Dirichlet layer \eqref{eq:gumbel-dirichlet} at concentration $\alpha$; the union tail \eqref{eq:union-tail} mixes it with the width ladder \eqref{eq:gumbel-tailwidth}.  The $h^\star_{\mathrm{gum},\Delta_0}$ rows are the least-favorable score of \citet{li2025framework} at its three published tunings.}
\label{tab:family}
\begin{tabular}{@{}lrrr@{}}
\toprule
Gumbel rule & $n=100$ & $n=300$ & $n=700$ \\
\midrule
$h_{\mathrm{ars}}$ (reference score) & \mcse{.0426}{.0029} & \mcse{.0192}{.0019} & \mcse{.0094}{.0014} \\
$h_{\log}$ (reference score) & \mcse{.1698}{.0053} & \mcse{.0984}{.0042} & \mcse{.0654}{.0035} \\
$h_{\mathrm{ind},1/e}$ (reference score) & \mcse{.2192}{.0057} & \mcse{.1309}{.0048} & \mcse{.0847}{.0039} \\
\addlinespace
$h^\star_{\mathrm{gum},.1}$ (least-favorable family) & \mcse{.0934}{.0041} & \mcse{.0556}{.0032} & \mcse{.0342}{.0026} \\
$h^\star_{\mathrm{gum},.01}$ & \mcse{.0420}{.0028} & \mcse{.0178}{.0019} & \mcse{.0098}{.0014} \\
$h^\star_{\mathrm{gum},.005}$ & \mcse{.0328}{.0025} & \mcse{.0124}{.0016} & \mcse{.0062}{.0011} \\
\addlinespace
$h^{\mathrm{sp}}_{.01}$ (diagnostic; equal tail, fixed $\Delta_0$) & \mcse{.0198}{.0020} & \mcse{.0056}{.0011} & \mcse{.0014}{.0005} \\
$h^{\mathrm{sp}}_{.05}$ (diagnostic; equal tail, fixed $\Delta_0$) & \mcse{.0188}{.0019} & \mcse{.0058}{.0011} & \mcse{.0020}{.0006} \\
\addlinespace
Bayes, tokenwise $\Delta$ (equal tail) & \mcse{.0246}{.0022} & \mcse{.0096}{.0014} & \mcse{.0042}{.0009} \\
Bayes, shared $\Delta$ (equal tail) & \mcse{.0194}{.0020} & \mcse{.0056}{.0011} & \mcse{.0014}{.0005} \\
\addlinespace
Bayes, tokenwise $\Delta$ (tail shape) & \mcse{.0274}{.0023} & \mcse{.0098}{.0014} & \mcse{.0042}{.0009} \\
Bayes, shared $\Delta$ $+$ tail shape & \mcse{.0194}{.0020} & \mcse{.0058}{.0011} & \mcse{.0014}{.0005} \\
Bayes, tokenwise $\Delta$ (union tail) & \mcse{.0438}{.0029} & \mcse{.0226}{.0021} & \mcse{.0116}{.0015} \\
Bayes, shared $\Delta$ $+$ union tail & \mcse{.0192}{.0019} & \mcse{.0058}{.0011} & \mcse{.0014}{.0005} \\
\bottomrule
\end{tabular}
\end{table}

Across the displayed horizons in Table~\ref{tab:family}, $h^{\mathrm{sp}}_{.01}$ and the shared spike Bayes rule have similar observed errors.  At $n=300$ and $n=700$, they produce identical decisions, so $b=c=0$ and no nontrivial conditional McNemar test is defined.  The tables assign $p=1$ only to retain these comparisons in the prespecified Holm adjustment family; this value is not a test result.  At $n=100$, the discordance is 2 versus 4 (exact two-sided McNemar $p=0.69$).  No incremental difference attributable to prior averaging is detected in this configuration; evaluating the generating density reproduces the observed difference from $h^\star_{\mathrm{gum},.005}$.  The tested choices $\Delta_0=.01$ and $.05$ also have similar errors.  The diagnostic fixes the deficit and assumes an equal tail.  Sections~\ref{sec:regimes} and~\ref{sec:tails} vary these assumptions over finite grids.
Across the selected horizons, the shared Dirichlet-layer and spike rules differ in observed error by at most $.0002$.  Section~\ref{sec:tails} evaluates the same rules under other generating tail shapes.

Within each component family, the shared and tokenwise rules differ only in hierarchy.  Under the persistent-deficit generator at $n=700$, their observed errors are $0.0014$ and $0.0042$, respectively, for both the spike and Dirichlet specifications.  The tokenwise mixture replaces each token's evidence by its prior average before multiplying.

Component-posterior summaries show concentration around the generating deficit.  At $n=700$, posterior means are $0.006$, $0.051$, $0.201$, and $0.401$ when the generating values are $0.005$, $0.05$, $0.2$, and $0.4$, with posterior standard deviations between $0.002$ and $0.014$ compared with a prior standard deviation of $0.144$.  Realized coverage of the $90\%$ credible intervals ranges from $0.843$ to $0.925$ across twelve deficit-by-horizon cells.  With 400 documents per cell, the binomial standard error is $0.015$; the $\Delta=0.05$, $n=700$ cell is $3.8$ standard errors below nominal coverage.  At $\Delta=0.005$, the posterior mean is $0.0059$.  The component posterior is a summary of the modeled deficit rather than a separate fit.

In this configuration, the point-deficit diagnostic and shared mixture have similar observed errors despite posterior uncertainty in $\Delta$.  Section~\ref{sec:regimes} evaluates point-deficit mismatch, and Section~\ref{sec:contamination} gives the corresponding finite-grid comparison for $\rho$.

\FloatBarrier
\subsection{Sensitivity to the mixing weight}
\label{sec:union-weight-sweep}

\begin{remark}[Integration over the mixing weight]
\label{rem:union-weight}
For a single document, integrating over a hyperprior on $w$ is equivalent to fixing
$w$ at its hyperprior mean.  The marginal likelihood, and hence the Bayes factor,
is linear in the prior measure.  Since $w$ enters \eqref{eq:union-tail} only as a
mixing weight between two fixed priors, $B_n$ is linear in $w$ and
\begin{equation}
  \int B_n(w)\,\pi_w(dw)=B_n\!\left(\int w\,\pi_w(dw)\right)
  \label{eq:weight-collapse}
\end{equation}
for every prior $\pi_w$, provided the tail branch is drawn once per document, as in
every rule reported here.  Thus the marginal likelihood under $\pi_w$ equals the
marginal likelihood obtained by substituting $\E_{\pi_w}[w]$.  If the tail branch
were instead redrawn at every position, the document likelihood would be a product
of one-token mixture densities.  Its Bayes factor would generally be a degree-$n$
polynomial in $w$, so \eqref{eq:weight-collapse} would not generally hold.
The reduction depends on the level at which the mixture is linear.  The
parameters $\Delta$, $\alpha$, and $J$ enter the
component densities nonlinearly, so their prior averages are not obtained by
substituting a single parameter value.

The value of $w$ specifies the prior odds of the two blocks, after which their
posterior probabilities update with the observations.  Under an equal-tail
generator, for which the Dirichlet block is correctly specified, its mean posterior
probability increases from $.5$ to $.92$ within fifty tokens at $\Delta=.1$ and to
$.99$ by two hundred.  At $\Delta=.005$, it is $.44$ at fifty tokens and $.65$ by
two hundred.  These posterior summaries show slower discrimination
between the two branches at the smaller deficit.  They describe branch learning
at $w=.5$, rather than a regret comparison across mixing weights.
Hierarchical estimation of $w$ across documents would define a different model and
is outside the prespecified-prior analysis.
\end{remark}

Table~\ref{tab:union-weight} evaluates $w$ on an eleven-point grid.  Because the
tail state is document level,
$B_n(w)=wB_n^{\mathrm{shape}}+(1-w)B_n^{\mathrm{width}}$ is linear and therefore
continuous in $w$, including at $w=1$.  The reported performance measures need
not be continuous: area under the curve is a rank functional of $B_n$, and
thresholded error is a step functional.  On the temperature-matched data, area under the curve varies
by at most $.0015$ within any cell across the evaluated weights $w=0,.1,\ldots,.9$.  The unrounded maximizing weights are
$w=.9,.7,.9,.9$ in the four OPT-1.3B cells and $w=.8,.9,.9,.9$ in the four
Sheared-LLaMA-2.7B cells.  At $w=1$, which assigns zero mass to the width block,
area under the curve decreases by $.0094$ to $.0489$ across cells and is within
$.0076$ of the shared-deficit equal-tail baseline.  Because no value strictly
between $.9$ and $1$ is evaluated, the location of the change within that interval
is not identified.  The small AUC range indicates robustness to the mixing weight on this grid,
but not to excluding the width branch.  This pattern is
consistent with Remark~\ref{rem:union-weight}: when both blocks have positive
prior probability, their posterior probabilities are updated using the tokens
within each document.

\begin{table}[htbp]
\centering
\scriptsize
\setlength{\tabcolsep}{3pt}%
\caption{Sensitivity of the union tail prior to its mixing weight.  Entries are area under the curve at 100 tokens, with prompt-cluster bootstrap standard errors in parentheses (2000 replicates, prompts resampled jointly across both arms; a binomial standard error does not apply to an AUC and DeLong's does not cover the prompt pairing) on the temperature-matched arms, 2500 documents per arm.  The weight $w$ of \eqref{eq:union-tail} assigns mass $w$ to the Dirichlet shape block at full width and $1-w$ to the tail-width block at $\alpha=\infty$; $w=0$ and $w=1$ define the two endpoint priors.  The equal-tail baseline shares the deficit across tokens without tail-shape or width averaging.  Boldface marks the largest area under the curve in each column of each block, ties in the displayed values included; where two displayed values tie, the larger weight is the unrounded maximizer.  In the reference-score block, boldface denotes within-block column maxima.}
\label{tab:union-weight}
\begin{tabular}{@{}lrrrrrrrr@{}}
\toprule
& \multicolumn{4}{c}{OPT-1.3B} & \multicolumn{4}{c}{Sheared-LLaMA-2.7B} \\
\cmidrule(lr){2-5}\cmidrule(lr){6-9}
$w$ & $T{=}.2$ & $.3$ & $.4$ & $.5$ & $T{=}.2$ & $.3$ & $.4$ & $.5$ \\
\midrule
$0$ (width only) & \mcse{.8378}{.0054} & \mcse{.9382}{.0034} & \mcse{.9797}{.0021} & \mcse{.9962}{.0008} & \mcse{.8325}{.0054} & \mcse{.9321}{.0038} & \mcse{.9797}{.0019} & \mcse{.9927}{.0013} \\
$.1$ & \mcse{.8379}{.0054} & \mcse{.9383}{.0034} & \mcse{.9797}{.0021} & \mcse{.9962}{.0008} & \mcse{.8327}{.0053} & \mcse{.9322}{.0038} & \mcse{.9797}{.0019} & \mcse{.9927}{.0013} \\
$.2$ & \mcse{.8380}{.0054} & \mcse{.9383}{.0034} & \mcse{.9797}{.0021} & \mcse{.9962}{.0008} & \mcse{.8328}{.0053} & \mcse{.9323}{.0038} & \mcse{.9797}{.0019} & \mcse{.9927}{.0013} \\
$.3$ & \mcse{.8381}{.0054} & \mcse{.9384}{.0034} & \mcse{.9798}{.0021} & \mcse{.9963}{.0008} & \mcse{.8330}{.0053} & \mcse{.9324}{.0038} & \mcse{.9797}{.0019} & \mcse{.9927}{.0013} \\
$.4$ & \mcse{.8382}{.0054} & \mcse{.9384}{.0034} & \mcse{.9798}{.0021} & \mcse{.9963}{.0008} & \mcse{.8332}{.0053} & \mcse{.9325}{.0038} & \mcse{.9798}{.0019} & \mcse{.9928}{.0013} \\
$.5$ & \mcse{.8383}{.0054} & \mcse{.9385}{.0034} & \mcse{.9799}{.0021} & \mcse{.9963}{.0008} & \mcse{.8334}{.0053} & \mcse{.9327}{.0038} & \mcse{.9798}{.0019} & \mcse{.9928}{.0013} \\
$.6$ & \mcse{.8385}{.0054} & \mcse{.9386}{.0034} & \mcse{.9800}{.0021} & \mcse{.9963}{.0008} & \mcse{.8337}{.0053} & \mcse{.9329}{.0038} & \mcse{.9798}{.0019} & \mcse{.9928}{.0013} \\
$.7$ & \mcse{.8388}{.0054} & \bestmcse{.9387}{.0033} & \mcse{.9801}{.0021} & \mcse{.9964}{.0008} & \bestmcse{.8339}{.0053} & \mcse{.9332}{.0038} & \mcse{.9799}{.0018} & \mcse{.9929}{.0012} \\
$.8$ & \mcse{.8391}{.0054} & \mcse{.9386}{.0033} & \mcse{.9802}{.0020} & \bestmcse{.9965}{.0007} & \bestmcse{.8339}{.0053} & \mcse{.9334}{.0038} & \mcse{.9800}{.0018} & \mcse{.9930}{.0012} \\
$.9$ & \bestmcse{.8393}{.0054} & \mcse{.9380}{.0034} & \bestmcse{.9805}{.0020} & \bestmcse{.9965}{.0007} & \mcse{.8330}{.0053} & \bestmcse{.9336}{.0038} & \bestmcse{.9801}{.0018} & \bestmcse{.9932}{.0012} \\
$1$ (shape only) & \mcse{.7903}{.0062} & \mcse{.8913}{.0045} & \mcse{.9599}{.0028} & \mcse{.9871}{.0015} & \mcse{.7875}{.0059} & \mcse{.8951}{.0046} & \mcse{.9560}{.0028} & \mcse{.9825}{.0018} \\
\addlinespace
equal-tail baseline & \mcse{.7828}{.0063} & \mcse{.8906}{.0045} & \mcse{.9594}{.0028} & \mcse{.9871}{.0014} & \mcse{.7814}{.0058} & \mcse{.8903}{.0046} & \mcse{.9553}{.0027} & \mcse{.9824}{.0017} \\
\addlinespace
\multicolumn{9}{@{}l}{\textit{Reference scores}}\\
$h_{\mathrm{ars}}$ & \bestmcse{.8499}{.0051} & \bestmcse{.9465}{.0031} & \bestmcse{.9824}{.0020} & \bestmcse{.9963}{.0009} & \bestmcse{.8467}{.0051} & \bestmcse{.9373}{.0036} & \bestmcse{.9829}{.0016} & \bestmcse{.9931}{.0012} \\
$h_{\log}$ & \mcse{.7825}{.0063} & \mcse{.8904}{.0045} & \mcse{.9585}{.0028} & \mcse{.9871}{.0014} & \mcse{.7807}{.0059} & \mcse{.8900}{.0046} & \mcse{.9544}{.0028} & \mcse{.9821}{.0018} \\
$h_{\mathrm{ind},1/e}$ & \mcse{.7531}{.0066} & \mcse{.8614}{.0052} & \mcse{.9320}{.0037} & \mcse{.9783}{.0019} & \mcse{.7573}{.0063} & \mcse{.8616}{.0052} & \mcse{.9364}{.0032} & \mcse{.9691}{.0023} \\
$h^\star_{\mathrm{gum},.1}$ & \mcse{.8459}{.0053} & \mcse{.9440}{.0032} & \mcse{.9819}{.0020} & \mcse{.9961}{.0008} & \mcse{.8409}{.0052} & \mcse{.9331}{.0037} & \mcse{.9810}{.0017} & \mcse{.9922}{.0012} \\
$h^\star_{\mathrm{gum},.01}$ & \mcse{.7836}{.0058} & \mcse{.8859}{.0046} & \mcse{.9563}{.0029} & \mcse{.9873}{.0015} & \mcse{.7649}{.0061} & \mcse{.8744}{.0048} & \mcse{.9486}{.0029} & \mcse{.9825}{.0017} \\
$h^\star_{\mathrm{gum},.005}$ & \mcse{.7577}{.0061} & \mcse{.8576}{.0051} & \mcse{.9378}{.0033} & \mcse{.9785}{.0018} & \mcse{.7376}{.0062} & \mcse{.8470}{.0052} & \mcse{.9278}{.0035} & \mcse{.9734}{.0020} \\
\bottomrule
\end{tabular}
\end{table}

\FloatBarrier
\subsection{Deficit regimes and prior-support sensitivity}
\label{sec:deficit-support-analysis}

\subsubsection{Empirical deficit regimes}
\label{sec:deficit-regime}

The deficit is determined jointly by the model, its context, and the decoding
rule.  Sampling temperature is the variable swept here, and the two ends of
this sweep are qualitatively different regimes.
Table~\ref{tab:deficit-regime} records the empirical distribution of
$\Delta_t=1-\max_w p_t(w)$ over the generated tokens.

\begin{table}[t]
\centering
\small
\caption{Empirical deficit $\Delta=1-\max_w p(w)$ by sampling temperature, over all generated watermarked tokens at every sampled temperature; 2500 documents per cell in $[.2,.5]$ and 500 elsewhere.  The primary deficit prior is $\Unif(.001,.5)$; the last three columns give the fraction of tokens below, inside, and above its support.  Entries summarize the recorded float32 deficits rather than estimates of an unknown parameter, so Monte Carlo standard errors are not reported.  Binomial intervals for the three proportions would understate their sampling uncertainty because tokens within a document share a context and are not independent draws.}
\label{tab:deficit-regime}
\begin{tabular}{llrrrrrr}
\toprule
Model & $T$ & Median & Mean & $q_{.9}$ & $\Delta<.001$ & In support & $\Delta>.5$ \\
\midrule
OPT-1.3B & $.1$ & $.0000$ & $.0226$ & $.0285$ & $.831$ & $.163$ & $.006$ \\
 & $.2$ & $.0000$ & $.0491$ & $.1972$ & $.704$ & $.276$ & $.020$ \\
 & $.3$ & $.0000$ & $.0799$ & $.3513$ & $.601$ & $.354$ & $.045$ \\
 & $.4$ & $.0546$ & $.1785$ & $.5396$ & $.298$ & $.574$ & $.127$ \\
 & $.5$ & $.1433$ & $.2372$ & $.6224$ & $.221$ & $.580$ & $.199$ \\
 & $.6$ & $.2401$ & $.2922$ & $.6910$ & $.172$ & $.553$ & $.275$ \\
 & $.7$ & $.3596$ & $.3591$ & $.7491$ & $.116$ & $.517$ & $.367$ \\
 & $1$ & $.6259$ & $.5422$ & $.8828$ & $.032$ & $.345$ & $.624$ \\
\addlinespace
Sheared-LLaMA-2.7B & $.1$ & $.0000$ & $.0245$ & $.0367$ & $.819$ & $.175$ & $.006$ \\
 & $.2$ & $.0000$ & $.0497$ & $.1997$ & $.690$ & $.290$ & $.020$ \\
 & $.3$ & $.0001$ & $.0810$ & $.3520$ & $.580$ & $.376$ & $.044$ \\
 & $.4$ & $.0452$ & $.1714$ & $.5299$ & $.308$ & $.572$ & $.119$ \\
 & $.5$ & $.1219$ & $.2259$ & $.6107$ & $.235$ & $.579$ & $.186$ \\
 & $.6$ & $.2210$ & $.2814$ & $.6787$ & $.178$ & $.562$ & $.260$ \\
 & $.7$ & $.3163$ & $.3346$ & $.7344$ & $.140$ & $.519$ & $.341$ \\
 & $1$ & $.5868$ & $.5082$ & $.8731$ & $.043$ & $.382$ & $.575$ \\
\bottomrule
\end{tabular}
\end{table}

For OPT-1.3B, support coverage is nonmonotone in temperature.  The share of tokens inside the
prior's support increases from $.163$ at temperature $.1$ to a maximum of $.580$
at $.5$ and decreases to $.345$ at $1$.  At $.1$, more than four fifths of tokens
fall below the lower support limit; at $1$, close to three fifths lie above the
upper limit.  Between temperatures $.3$ and $.4$, the median deficit increases
from $.0000$ to $.0546$ and the share above one half increases from $.045$ to
$.127$; by temperature $.7$, that share is $.367$ and the median exceeds $.35$.
The reanalysis of the archived benchmark data in Supplementary Section~\ref{sec:released-width} is therefore
concentrated in the small-deficit regime, whereas full-softmax sampling near
temperature $1$ produces larger deficits for these two models.  Nucleus sampling
renormalizes probabilities after truncation and can reduce $\Delta$ for a fixed
pre-truncation distribution, by an amount not measured in these experiments.
Changes in model or context can also alter the deficit distribution, and their
direction is not established here.  These observations concern two open models
under the specified decoder; they neither estimate typical deployment conditions
nor bound the deficit distribution for other models, decoders, or temperatures.

\FloatBarrier
\subsubsection{Deficits above one half}
\label{sec:deficit-support}

The primary shared-deficit rules in the temperature-matched analysis use
$\Delta\sim\Unif(.001,.5)$.  The restriction $\Delta\le1/2$ preserves the
top-probability-deficit interpretation for all tail widths considered.
Under the tail-width component, each nonzero residual coordinate has mass
$\Delta/J$, so the designated coordinate with mass $1-\Delta$ is largest only if
$\Delta\le J/(1+J)$.  For $J=1$, the profile estimate reported for
Sheared-LLaMA-2.7B, this threshold is exactly one half; above it, the residual
coordinate has greater mass than the designated coordinate.  Consequently, the vector leaves the $\Delta$-regular class
$\{\bP:\max_w p_w\le1-\Delta\}$ on which the least-favorable score of
\citet{li2025framework} is built, and $\Delta$ ceases to be a top-probability
deficit, becoming one minus the mass of a designated coordinate that need not be
the maximum.  This change in parameter interpretation does not affect null
validity: the component densities integrate to one
for every $\Delta\in(0,1)$, since the leading term contributes $1-\Delta$ and the
tail term $\Delta$ at any width, and the exact null does not involve $\Delta$ at
all, so Theorem~\ref{thm:martingale} continues to apply.  At $J=1$ the relabeling
also leaves the likelihood unchanged: $J/\Delta-1=(1-\Delta)/\Delta$, so
\eqref{eq:gumbel-tailwidth} reads
$f^{\mathrm{gum}}_{\Delta,1}(r)=r^{\Delta/(1-\Delta)}+r^{(1-\Delta)/\Delta}$, which
is symmetric in $\Delta$ and $1-\Delta$.  Therefore any performance difference
from the widened support must arise through components with $J>1$.  The profile
estimates are $J=1$ for Sheared-LLaMA-2.7B and $J=2$ for OPT-1.3B.

Table~\ref{tab:deficit-regime} reports a median deficit above one half at
temperature $1$, outside the support of the primary prior.
Table~\ref{tab:deficit-support} therefore repeats three rules under
$\Unif(.001,.999)$.

\begin{table}[t]
\centering
\scriptsize
\setlength{\tabcolsep}{2pt}%
\caption{Widening the deficit prior past one half.  Entries are area under the curve at 100 tokens on the temperature-matched arms, 2500 documents per arm.  The narrow prior is $\Unif(.001,.5)$ and the wide one $\Unif(.001,.999)$; the wider support admits configurations in which the designated leading coordinate is not the largest.  Prompt-cluster bootstrap standard errors are in parentheses, from the same joint resample as Table~\ref{tab:union-weight}.  The wide-against-narrow differences have paired bootstrap standard errors recorded in the artifact but are excluded from the Holm family of Section~\ref{sec:temperature-matched}: they constitute a prior-sensitivity analysis rather than a comparison included in that testing family, so no adjusted test is reported.  The reference scores have no deficit prior, so their row is em-dashed in that column and appears once.  In the reference-score block, boldface denotes within-block column maxima.}
\label{tab:deficit-support}
\begin{tabular}{@{}llrrrrrrrr@{}}
\toprule
& & \multicolumn{4}{c}{OPT-1.3B} & \multicolumn{4}{c}{Sheared-LLaMA-2.7B} \\
\cmidrule(lr){3-6}\cmidrule(lr){7-10}
Rule & $\Delta$ & $T{=}.2$ & $.3$ & $.4$ & $.5$ & $T{=}.2$ & $.3$ & $.4$ & $.5$ \\
\midrule
Bayes, shared $\Delta$ (equal tail) & narrow & \mcse{.7828}{.0063} & \mcse{.8906}{.0045} & \mcse{.9594}{.0028} & \mcse{.9871}{.0014} & \mcse{.7814}{.0058} & \mcse{.8903}{.0046} & \mcse{.9553}{.0027} & \mcse{.9824}{.0017} \\
 & wide & \mcse{.7828}{.0063} & \mcse{.8906}{.0045} & \mcse{.9594}{.0028} & \mcse{.9871}{.0014} & \mcse{.7814}{.0058} & \mcse{.8903}{.0046} & \mcse{.9553}{.0027} & \mcse{.9824}{.0017} \\
\addlinespace
\quad $+$ tail width only & narrow & \mcse{.8378}{.0054} & \mcse{.9382}{.0034} & \mcse{.9797}{.0021} & \mcse{.9962}{.0008} & \mcse{.8325}{.0054} & \mcse{.9321}{.0038} & \mcse{.9797}{.0019} & \mcse{.9927}{.0013} \\
 & wide & \mcse{.8400}{.0054} & \mcse{.9393}{.0034} & \mcse{.9799}{.0021} & \mcse{.9962}{.0008} & \mcse{.8347}{.0053} & \mcse{.9333}{.0038} & \mcse{.9802}{.0018} & \mcse{.9928}{.0013} \\
\addlinespace
\quad $+$ union tail, $w=.5$ & narrow & \mcse{.8383}{.0054} & \mcse{.9385}{.0034} & \mcse{.9799}{.0021} & \mcse{.9963}{.0008} & \mcse{.8334}{.0053} & \mcse{.9327}{.0038} & \mcse{.9798}{.0019} & \mcse{.9928}{.0013} \\
 & wide & \mcse{.8403}{.0054} & \mcse{.9395}{.0033} & \mcse{.9801}{.0021} & \mcse{.9963}{.0008} & \mcse{.8354}{.0053} & \mcse{.9337}{.0038} & \mcse{.9803}{.0018} & \mcse{.9928}{.0013} \\
\addlinespace
\multicolumn{10}{@{}l}{\textit{Reference scores}}\\
$h_{\mathrm{ars}}$ & --- & \bestmcse{.8499}{.0051} & \bestmcse{.9465}{.0031} & \bestmcse{.9824}{.0020} & \bestmcse{.9963}{.0009} & \bestmcse{.8467}{.0051} & \bestmcse{.9373}{.0036} & \bestmcse{.9829}{.0016} & \bestmcse{.9931}{.0012} \\
$h_{\log}$ & --- & \mcse{.7825}{.0063} & \mcse{.8904}{.0045} & \mcse{.9585}{.0028} & \mcse{.9871}{.0014} & \mcse{.7807}{.0059} & \mcse{.8900}{.0046} & \mcse{.9544}{.0028} & \mcse{.9821}{.0018} \\
$h_{\mathrm{ind},1/e}$ & --- & \mcse{.7531}{.0066} & \mcse{.8614}{.0052} & \mcse{.9320}{.0037} & \mcse{.9783}{.0019} & \mcse{.7573}{.0063} & \mcse{.8616}{.0052} & \mcse{.9364}{.0032} & \mcse{.9691}{.0023} \\
$h^\star_{\mathrm{gum},.1}$ & --- & \mcse{.8459}{.0053} & \mcse{.9440}{.0032} & \mcse{.9819}{.0020} & \mcse{.9961}{.0008} & \mcse{.8409}{.0052} & \mcse{.9331}{.0037} & \mcse{.9810}{.0017} & \mcse{.9922}{.0012} \\
$h^\star_{\mathrm{gum},.01}$ & --- & \mcse{.7836}{.0058} & \mcse{.8859}{.0046} & \mcse{.9563}{.0029} & \mcse{.9873}{.0015} & \mcse{.7649}{.0061} & \mcse{.8744}{.0048} & \mcse{.9486}{.0029} & \mcse{.9825}{.0017} \\
$h^\star_{\mathrm{gum},.005}$ & --- & \mcse{.7577}{.0061} & \mcse{.8576}{.0051} & \mcse{.9378}{.0033} & \mcse{.9785}{.0018} & \mcse{.7376}{.0062} & \mcse{.8470}{.0052} & \mcse{.9278}{.0035} & \mcse{.9734}{.0020} \\
\bottomrule
\end{tabular}
\end{table}

The sensitivity results differ by component family.  For the equal-tail rule, area
under the curve and Type~II error agree to every recorded digit under the narrow
and widened priors in all eight cells.  The score difference is $-\log2$ with standard
deviation zero to the precision recorded here, which is
$\log\{(.5-.001)/(.999-.001)\}$, the ratio of prior densities.  That shift is
numerical rather than analytic: the equal-tail density is strictly positive for
every $0<r<1$ at every $\Delta$, including $\Delta>1/2$, so the newly admitted
region contributes positive marginal likelihood and the exact shift is data dependent.  It is
negligible here because this large-deficit, full-width component assigns little
likelihood to the observed pivots, so on these paths widening the support divides the marginal
likelihood by the added prior mass to within the recorded digits.  A
constant shift preserves document rankings, so the calibrated decisions are unchanged.

For the width-based rules, the score difference is not constant: its standard
deviation is $.49$, document rankings change, and area under the curve increases
by up to $.0022$, with the largest increase at temperature $.2$.  The widened
union has a larger area under the curve in seven of eight cells.  These cellwise
differences are descriptive and have paired standard errors from the same
prompt-cluster resample as the matched AUCs, recorded in the artifact.  They are
excluded from the Holm family of Section~\ref{sec:temperature-matched} because the
prior-support sensitivity analysis was not among the comparisons for which that
family was defined.  No adjusted $p$-value is claimed.  The $J=1$ component is
invariant under $\Delta\mapsto1-\Delta$, since
$f^{\mathrm{gum}}_{\Delta,1}(r)=r^{\Delta/(1-\Delta)}+r^{(1-\Delta)/\Delta}$ is
symmetric in the two exponents, and $\Unif(.001,.999)$ is symmetric about one
half, so folding it gives exactly $\Unif(.001,.5)$: the two priors induce the
same $J=1$ predictive, and the discretized versions agree to $3.7\times10^{-9}$
in log density, which is quadrature error.  A $\Delta$ above one half at $J=1$
is a relabeling of one below it rather than a distinct configuration.  Components
with larger $J$ are not symmetric: the same comparison gives $.69$, $.74$, $1.22$,
and $1.77$ nats at $J=4,16,64$, and $256$.  The widened prior admits deficits above
one half distributed over several coordinates, which moderate-width components
can represent but the equal-tail component, distributed over all $M-1$
coordinates, cannot.  It therefore defines a different model rather than merely
relaxing a parameter bound, and the observed improvement is confined to the
width-based family.  The ranking against the reference scores is unchanged:
$h_{\mathrm{ars}}$ has higher AUC than the widened union in seven of eight cells.

\subsection{Component calculations and stable computation}
\label{sec:stable-computation}

\subsubsection{Dirichlet tail calculations}
\label{sec:dirichlet-details}

For the reported settings, convergence to the spike limit is numerically slow.  In a finite-grid comparison at $M=1000$ and $\Delta=0.2$, using direct non-tabulated quadrature at $20{,}000$ equally spaced logits on $[-30,30]$, the maximum absolute log-density differences from $\alpha=\infty$ are $0.634$, $0.0842$, and $0.00875$ for $\alpha=10,100,1000$, respectively.  The atoms at $100$ and $1000$ represent intermediate near-equal tails.  Renormalized i.i.d.\ uniforms, as used in the accompanying simulation snapshot \citep{li2024watermarkframework}, do not follow a Dirichlet distribution.  The scaled tail marginal $Kq$ converges to $\Unif(0,2)$, whose variance is $1/3$; under the symmetric Dirichlet model, its variance is $(K-1)/(K\alpha+1)$.  The variance-matched member is therefore $\alpha\approx3$.

Two distinctions affect the comparison in Section~\ref{sec:tails}.  The endpoint $2$ is a limiting support boundary rather than a finite-$K$ bound: the normalizing sum fluctuates at $O(K^{-1/2})$, so at $K=999$ roughly $0.7\%$ of the scaled coordinates $Kq_k$ exceed $2$.  In addition, matching the variance of $q$ does not match the size-biased law observed after token selection.  Selection maps $\Unif(0,2)$ to the triangular law on $(0,2)$ with variance $2/9$ and maps $\operatorname{Gamma}(3,3)$ to $\operatorname{Gamma}(4,3)$ with variance $4/9$.  The size-biased variances therefore differ by a factor of two even though the unselected variances agree.  The simulation-snapshot tail is an out-of-family generator for the robustness assessment rather than a special case of the Dirichlet layer.

One-token simulation requires only a size-biased tail coordinate.  The selected tail probability is $\Delta Q_t^{\mathrm{sb}}$ with $Q_t^{\mathrm{sb}}\sim\operatorname{Beta}(\alpha+1,(K-1)\alpha)$, independently over tokens conditional on the applicable $\alpha$; the pivot is $R_t=U_t^{1-\Delta}$ with probability $1-\Delta$ and $R_t=U_t^{\Delta Q_t^{\mathrm{sb}}}$ otherwise.  The same size-biasing identity, $\E_{Q^{\mathrm{sb}}}[(Q^{\mathrm{sb}})^{-1}g(Q^{\mathrm{sb}})]=K\,\E_q[g(q)]$, gives an independent proof of \eqref{eq:gumbel-dirichlet}.

\subsubsection{Transform evaluation}

Writing $L=-\log r$ and substituting $u=1/q$ gives the univariate Laplace transform
\begin{equation}
  \psi_\alpha(c)=\E_q\big[e^{-c/q}\big],
  \qquad
  f^{\mathrm{gum}}_{\Delta,\alpha}(r)
  =r^{\Delta/(1-\Delta)}+\frac{K}{r}\,\psi_\alpha\!\left(\frac{L}{\Delta}\right),
  \label{eq:psi}
\end{equation}
so $\psi_\alpha$ can be tabulated once for each $\alpha$ and used for every $(\Delta,r)$ pair.  Each component lookup remains $O(1)$.  With $N_\Delta$ and $N_\alpha$ grid points, the clean joint grid has $G=N_\Delta N_\alpha$ components and costs $O(G)$ per token once the tables are available; only when the other grid dimensions are fixed is this cost linear in $N_\alpha$.

Discretizing the one-dimensional expectation preserves normalization if the nodes are rescaled so that $K$ times their weighted mean equals one; interpolation of the resulting table is a further numerical approximation described below and evaluated in Section~\ref{sec:numerical-checks}.

\subsubsection{Likelihood-ratio computation}

Let grid component $g$ have $\vartheta_g=(S_g,\eta_g)$ with $\eta_g=(\Delta_g,\alpha_g,J_g)$ and prior weight $w_g$, and write $f_{\vartheta_g}$ for the branch density \eqref{eq:branch-density} it selects.  The clean rules have no contamination coordinate; the contamination-aware rules of Supplementary Section~\ref{sec:contamination-supplement} include an additional component value $\rho_g$, and the clean case is $\rho_g=0$, at which \eqref{eq:contam-lr} reduces to the ordinary likelihood ratio.  Its logarithm is used below.  At $\alpha_g=\infty$ and $J_g=K$, that density is \eqref{eq:gumbel-spike}.  Define
\begin{equation}
  z_{g,t}=\log\left[\rho_g+(1-\rho_g)
  \frac{f_{\vartheta_g}(Y_t)}{f_0(Y_t)}\right].
  \label{eq:component-log-ratio}
\end{equation}
Algorithm~\ref{alg:bf} performs all averaging in log space.  Here $\LSE(a_1,\ldots,a_G)=\log\sum_g e^{a_g}$.

\begin{algorithm}[H]
\caption{Sequential grid Bayes factor for shared, tokenwise, or tokenwise-with-document-level-state components}
\label{alg:bf}
\begin{algorithmic}[1]
\Require pivots $Y_{1:n}$; grid $(\vartheta_g,w_g)_{g=1}^G$ written $\vartheta_g=(S_g,\eta_g)$ with $S_g$ the tail state, so that $\pi_S(s)=\sum_{g:S_g=s}w_g$ is the state's prior mass and, for any state with $\pi_S(s)>0$, $w_{g\mid s}=w_g/\pi_S(s)$ the weights within it (at $w\in\{0,1\}$ one state is empty and is dropped, so no $0/0$ arises); the grid weights $w_g$ are not the union-tail mixing weight $w$, which is recovered as $\pi_S(\mathrm{shape})$; class prior $q_H=\Prb(H_1)$; hierarchy $\mathcal H\in\{\mathrm{shared},\mathrm{tokenwise},\mathrm{tokenwise}^{S}\}$; optional level $\alpha_{\mathrm{test}}$
\State $L\gets0$; $a_g\gets\log w_g$ for every $g$; $b_s\gets0$ for every state $s$
\For{$t=1,\ldots,n$}
  \State compute all $z_{g,t}$ from \eqref{eq:component-log-ratio}
  \Statex
  \If{$\mathcal H=\mathrm{shared}$}
    \State $c_t\gets\LSE_g(a_g+z_{g,t})$
    \If{$c_t=-\infty$}
      \State \Return $L=-\infty$, $\Prb(H_1\mid Y_{1:t})=0$, and undefined shared-component weights
    \EndIf
    \State $a_g\gets a_g+z_{g,t}-c_t$ for every $g$ \Comment{component posterior}
  \ElsIf{$\mathcal H=\mathrm{tokenwise}$}
    \State $c_t\gets\LSE_g(\log w_g+z_{g,t})$
  \Else
    \Statex \hskip\algorithmicindent Because $S$ is document shared, each branch log marginal likelihood $b_s$ accumulates across tokens, whereas the conditional component weights $w_{g\mid s}$ remain fixed at their prior values at each token.
    \State $b_s\gets b_s+\LSE_{g:S_g=s}(\log w_{g\mid s}+z_{g,t})$ for each state $s$
    \State $c_t\gets\LSE_s(\log\pi_S(s)+b_s)-L$
    \If{$c_t=-\infty$}
      \State \Return $L=-\infty$ and $\Prb(H_1\mid Y_{1:t})=0$
    \EndIf
  \EndIf
  \State $L\gets L+c_t$ \Comment{$L=\log B_t$}
  \If{$\alpha_{\mathrm{test}}$ is supplied and $L\ge\log(1/\alpha_{\mathrm{test}})$}
    \State stop and reject $H_0$
  \EndIf
\EndFor
\State \Return $L$, the posterior from \eqref{eq:posterior}, and shared-component weights when applicable
\end{algorithmic}
\end{algorithm}

The mixture update costs $O(G)$ time per token and $O(G)$ memory once component likelihood ratios are available, writing $G$ for the number of grid components (distinct from the tail width $J$).  Generic Gumbel evaluation may cost $O(MG_\Delta)$, where $G_\Delta$ is the number of distinct deficit atoms in the grid, $96$ here, while the spike and least-favorable families admit the closed forms above.  The Dirichlet layer costs one tabulated $\psi_\alpha$ lookup per $(\Delta,\alpha)$ pair; adding six $\alpha$ atoms enlarges the component count from 96 to 576 in the clean experiments while preserving the $O(G)$ per-token structure.  The Gumbel components \eqref{eq:gumbel-spike} and \eqref{eq:gumbel-dirichlet} are strictly positive on $(0,1)$, so the one-step mixture numerator is positive at every observable pivot and the zero-support case does not arise for this scheme.

The implementation differs from Algorithm~\ref{alg:bf} in four respects.  Both detectors integrate the deficit over the same $96$-node Gauss--Legendre grid; the tokenwise path tabulates that integral as a function of the pivot and interpolates it, whereas the shared path accumulates component-specific log likelihoods across tokens and marginalizes over components afterward.  Second, the tokenwise \emph{spike Gumbel} mixture linearly interpolates $\log\int f_\Delta^{\mathrm{gum},\mathrm{sp}}\,\Pi(d\Delta)$ on an $80{,}001$-point $\logit(r)$ grid over $[-30,30]$; the largest measured midpoint error in log density is $5\times10^{-7}$.  Numerical quadrature estimates its total mass as $1-9.1\times10^{-9}$; the shared spike path evaluates \eqref{eq:gumbel-spike} directly.  Third, both Dirichlet paths evaluate the inner tabulated transform $\psi_\alpha$ of \eqref{eq:psi} on a $200{,}001$-node log-$c$ grid.  In a comparison with direct, non-tabulated quadrature --- $20{,}000$ equally spaced logits on $[-30,30]$, paired cyclically with the 96 deficit nodes, for all six configured $\alpha$ values --- the largest inner component log-density discrepancy is $9.5\times10^{-8}$.  Node rescaling normalizes the pre-interpolation discretized components, while numerical quadrature estimates the smallest inner-interpolated component mass as $1-7.3\times10^{-9}$.  The tokenwise Dirichlet rule applies a separate outer interpolation to the prior-averaged density on an $80{,}001$-node logit grid over $[-30,30]$.  Across all $80{,}000$ outer-grid midpoints, the maximum additional log-ratio discrepancy from direct evaluation of the prespecified mixture is $2.1\times10^{-6}$ near the upper grid boundary; numerical quadrature estimates the outer-interpolated numerator mass as $1-7.7\times10^{-9}$.  Fourth, the tokenwise union-tail rule tabulates the within-branch per-token mixture of Algorithm~\ref{alg:bf}'s tokenwise$^S$ arm on a separate $80{,}001$-node logit grid for each branch, normalizing within each, rather than recomputing it component-wise at every token; its shape branch therefore uses the same tabulated $\psi_\alpha$ transform as the two Dirichlet paths.  These finite-grid comparisons do not bound interpolation error uniformly, and the numerical mass calculations are not certified one-sided bounds.  Consequently, the exact martingale result applies only to analytic or exactly normalized densities.  An interpolated numerator retains the Ville bound if its total mass is certified not to exceed one or if it is conservatively renormalized using a certified upper mass bound; the resulting process is generally a test supermartingale rather than a Bayes-factor martingale.  On the same validation grid, a $20{,}001$-node inner table has a largest per-token discrepancy of $9.5\times10^{-6}$, which can accumulate to about $.0067$ over 700 tokens and change a decision when the log Bayes factor lies within that distance of a cutoff.  The reported calculations therefore use the $200{,}001$-node inner table.

The computational-materials index in Section~\ref{sec:computational-materials} identifies the lookup-grid, quadrature, outer-interpolation, and component-level tail diagnostics.

\subsection{Numerical checks and interpretation}
\label{sec:numerical-checks}

The standard errors in Section~\ref{sec:experiments} condition on the calibrated
cutoffs.  For fractional miss contributions that average over boundary
randomization, the binomial plug-in formula uses the Bernoulli variance upper
bound rather than estimating their exact conditional variance.  Maximum-regret
bootstrap samples resample documents independently within generating regimes,
preserve pairing across rules and horizons, and recompute the best-rule envelope
and maximum on each draw.  These calculations exclude uncertainty from
calibration, numerical approximation, and the choice of candidate rules and
generating families; they are not simultaneous confidence intervals.

Across the two clean experiments, 27 method--experiment combinations are evaluated at each of 700 prefixes, giving 18{,}900 prefix cells: twelve in the tokenwise-specification experiment and fifteen in the shared-deficit experiment, the shared-hierarchy rules occurring only in the latter.  Over the 69 non-diagnostic cells at the three tabulated horizons, realized evaluation-sample Type~I error ranges from $0.0442$ to $0.0564$; including the two diagnostic scores of Section~\ref{sec:family} gives 81 cells and widens the range to $[0.0424,0.0564]$.

Using the 96-node design in Table~\ref{tab:design}, the recorded paired clean-benchmark sweep at 48, 96 and 192 Gauss--Legendre nodes covers the shared and tokenwise equal-tail and Dirichlet rules.  The 36 equal-tail and Dirichlet Gumbel cells have identical Type~II errors across node counts.  This original diagnostic excludes the union-tail rules and does not bound quadrature error outside the sampled paths.

Over 300 paired Gumbel paths, the 96-node contamination-aware equal-tail log Bayes factors differ from 192-node values by at most $8.6\times10^{-8}$, attained at $\rho_\star=.4,n=700$; the corresponding value at $\rho_\star=0,n=700$ is $4.4\times10^{-8}$.  At 48 nodes the maximum rises to $2.0\times10^{-3}$.  The contamination-aware tail-shape rule's largest 96-versus-192 discrepancy over the same paths is $2.9\times10^{-7}$, slightly the larger of the two; this reflects the fixed $200{,}001$-node $\psi_\alpha$ table rather than a property of the layer.  The original contamination diagnostic also excludes the union-tail rule.

An additional decision-level check uses 300 calibration, 300 evaluation-null and 300 evaluation-alternative documents per experiment, retaining the 700-token horizon and production lookup tables. All 48 Gumbel cells, including the twelve shared or tokenwise union-tail cells, have identical Type~II errors at 48, 96 and 192 nodes. This smaller-sample diagnostic is separate from the original full-sample check. The extended contamination check retains all 300 original paired paths: the union-tail rule has a maximum 96-versus-192 log-Bayes-factor discrepancy of $2.91\times10^{-7}$, and every original equal-tail and Dirichlet cell is reproduced. Neither check certifies a one-sided interpolation-mass bound.

At $\alpha=\infty$, the maximum absolute discrepancy from the closed-form spike log density is $9.6\times10^{-8}$ on the validation grid, and the maximum discrepancy in the 700-token shared log Bayes factor is $9.0\times10^{-8}$ on the checked paths.  The $\alpha=\infty$ rule in Table~\ref{tab:tails} therefore implements the equal-tail detector to the reported numerical precision; the equal-tail generator T1 reproduces the shared-$\Delta$ simulator exactly.

With the $\alpha$ prior restricted to $\{\infty\}$, Type~I and Type~II decisions match those of the equal-tail rules at all 700 prefixes under both hierarchies and both decision rules.  The comparison uses identical pivot arrays, and the cutoffs differ by $O(10^{-13})$.  A unit test also confirms equality of the detector and tail-sweep implementations.

The 604 automated tests cover density normalization, null calibration, analytic identities, marginalization, paired simulation, model-family containment, archived-benchmark replay and prefix curves, generation-resume safeguards, manuscript--artifact consistency, and agreement with reference implementations.  These checks do not establish the one-sided mass bound required for an e-process guarantee.  When $b=c=0$, the directional McNemar effect is not estimable and no nontrivial conditional test is defined.  For Holm adjustment, the implementation records $p=1$ and retains the comparison in the prespecified family; the recorded value is not a test result.  None of the 18 shared-Bayes-versus-reference comparisons has zero discordances, so the convention does not affect the reported adjusted $p$-values.  Numerical summaries, prefix-curve tables, document-level indicators, and analysis code are organized in a separate companion repository; Section~\ref{sec:computational-materials} identifies their contents and the limits of their reproducibility claims. Vector figures are included with the manuscript source.

The regime sweeps compare matched and mismatched persistence assumptions; the tail-law sweep separates deficit-grid mismatch from tail-family mismatch.  In the tail-law sweep at $n=100$, the maximum regrets of the Dirichlet tail-shape mixture and $h^{\mathrm{sp}}_{.01}$ are $.0014$ and $.0270$, respectively.  These values are conditional on the fifteen rules and six tail laws evaluated and do not characterize the full least-favorable class or heterogeneous deployments.

\FloatBarrier
\subsubsection{Anytime monitoring}
\label{sec:anytime-monitoring}

Fixed-horizon calibration does not control Type~I error under optional stopping.  For the exact analytic densities, Theorem~\ref{thm:martingale} justifies the rule $\sup_{t\le700}B_t\ge20$ with Type~I error at most .05.  Table~\ref{tab:anytime} applies the same threshold to the numerical implementations.  The shared-spike grid row evaluates the analytic mixture directly, subject to floating-point error; the tokenwise Gumbel lookup, both Dirichlet rows, and the union-tail row, whose shape branch keeps the Dirichlet component live at $w=.5$, additionally use interpolation.  Their normalization errors are numerically small, but the observed below-one quadrature values are not certified upper bounds; those rows are approximate implementations of the analytic anytime rule.  Theorem~\ref{thm:martingale} applies to the analytic densities, not automatically to their finite-precision implementations.  Under shared $\Delta$, the shared hierarchy has lower observed Type~II error than the tokenwise hierarchy.

\begin{table}[htbp]
\centering
\caption{Anytime-threshold experiment at $n=700$ in the shared-$\Delta$ setting.  Reject on the first crossing of $B_t\ge20$.  Parentheses give Monte Carlo standard errors.  The exact analytic construction has the .05 guarantee; the interpolation-based rows are numerical approximations described in Supplementary Section~\ref{sec:stable-computation}.  The rows differ in hierarchy and component family.  These entries are not directly comparable to fixed-horizon 5\% cutoffs.}
\label{tab:anytime}
\begin{tabular}{@{}lrr@{}}
\toprule
Bayesian model & Type~I & Type~II \\
\midrule
tokenwise, spike & \mcse{.0074}{.0012} & \mcse{.0522}{.0031} \\
tokenwise $+$ tail shape & \mcse{.0078}{.0012} & \mcse{.0504}{.0031} \\
shared $\Delta$, spike & \mcse{.0074}{.0012} & \mcse{.0036}{.0008} \\
shared $\Delta$ $+$ tail shape & \mcse{.0086}{.0013} & \mcse{.0036}{.0008} \\
shared $\Delta$ $+$ union tail & \mcse{.0154}{.0017} & \mcse{.0036}{.0008} \\
\bottomrule
\end{tabular}
\end{table}

\FloatBarrier
\subsection{Archived benchmark data: reanalysis and replay}
\label{sec:released-analysis}

\subsubsection{Benchmark data of Li et al. and calibration}
\label{sec:released-calibration}

This supplementary analysis reanalyzes the open-model benchmark experiment of \citet{li2025framework} using the archived samples in the authors' \emph{WatermarkFramework} GitHub repository \citep{li2024watermarkframework}.  These are the pre-existing outputs introduced in Section~\ref{sec:released-output}; the repository reference identifies the exact commit used for replay.  For each of OPT-1.3B \citep{zhang2022opt} and Sheared-LLaMA-2.7B \citep{xia2023sheared}, the benchmark archive contains 500 Gumbel-max-watermarked and 500 unwatermarked continuations sharing the same prompts.  The prompts are 50-token windows from the first 500 eligible records in the \texttt{train} split of C4's \texttt{realnewslike} configuration \citep{raffel2020exploring}.  Each window immediately precedes the final 200 tokens of its record after truncation to at most 2028 tokens, as specified in their reference generation code \citep{li2024watermarkframework}; the first 200 continuation tokens are scored.  Watermarked generation uses temperature $.1$, key 15485863, and the skipgram pseudorandom function with context width four.  The unwatermarked generation function does not apply the command-line temperature and therefore samples from unscaled logits, corresponding to temperature one \citep{li2024watermarkframework}.  The temperature mismatch confounds the Type~II comparison.  The watermarked arm generated at temperature $.1$ contains substantial phrase repetition.  The reference implementation's \texttt{skipgram} function does not hash the four-token window jointly: it addresses the keyed draw with the oldest token of that window alone, so the pivot at position $t$ is the $W_t$ coordinate of the vector drawn from the seed $h(W_{t-4})$, and it repeats whenever the \emph{pair} $(W_{t-4},W_t)$ repeats.  Repetition of the context token alone therefore reuses a seed but does not guarantee repetition of the pivot; phrase repetition can reproduce both coordinates of the pair.  Grouping the archived watermarked positions confirms this mechanism: the pivot is constant within every one of the $38{,}104$ and $42{,}984$ groups that share a pair, and is not constant within $7{,}477$ of $24{,}182$ and $8{,}699$ of $26{,}243$ of the groups that share only the context token.  Watermarked documents contain $76$ and $86$ distinct pivot values on average among 200 positions, compared with $197$ for unwatermarked documents.  A statistic that ignores the pivot \emph{values} and uses only their number of distinct values separates the archived watermarked and unwatermarked samples with area under the curve $.997$--$.998$ and Type~II error $.004$--$.012$, lower than every rule in Table~\ref{tab:released-output}.  Consequently, error rates on these samples conflate watermark sensitivity with sensitivity to low-temperature repetition.  Section~\ref{sec:temperature-matched} removes this confound by regenerating both arms at a common temperature.  Table~\ref{tab:released-output} is therefore reported as an exploratory reproduction of the comparison supported by these benchmark data; it does not estimate a matched-temperature ranking of detection power.

The archived watermarked tensors contain exactly 200 scored positions.  The additional 20 tokens stored for unwatermarked generation do not provide a matched watermarked-pivot sequence, so the power curves for these archived data cannot extend beyond $n=200$.

No text is regenerated in this supplementary reanalysis.  The benchmark files in WatermarkFramework store the watermarked pivots and the unwatermarked token sequences.  The latter are passed through the reference implementation's key and hash functions on the CPU to reconstruct null pivots; this step uses no model weights, tokenizer, C4 download, or GPU.  For each model, one independent sample of 10,000 exact-null pivot paths, generated with base seed 240401253, supplies method-specific randomized $.05$ cutoffs at every prefix $n=1,\ldots,200$.  Each cutoff is applied unchanged to the archived unwatermarked and watermarked benchmark samples.  The boundary-integrated calibration rate is exactly $.05$; the archived unwatermarked benchmark outputs are reserved for empirical Type~I evaluation.  The displayed reference-score menu matches the plotting specification in their reference implementation.

All Bayesian rows in this subsection use the same $\Delta\sim\Unif(.001,.5)$ prior as the synthetic experiments.  Ninety-six Gauss--Legendre nodes in $\Delta$ approximate the continuous prior.  The Dirichlet tail-shape rows use the uniform prior on $\alpha\in\{0.1,1,10,100,1000,\infty\}$, independently of $\Delta$.  The union-tail row assigns weight $w=.5$ to this full-width shape prior and the remaining weight to the equal-tail width ladder.

Figure~\ref{fig:released-common-prefix} reports every available prefix under this common calibration.  The table that follows gives the $n=200$ endpoint with Monte Carlo standard errors.

\begin{figure}[p]
\centering
\includegraphics[width=\textwidth,height=.80\textheight,keepaspectratio]{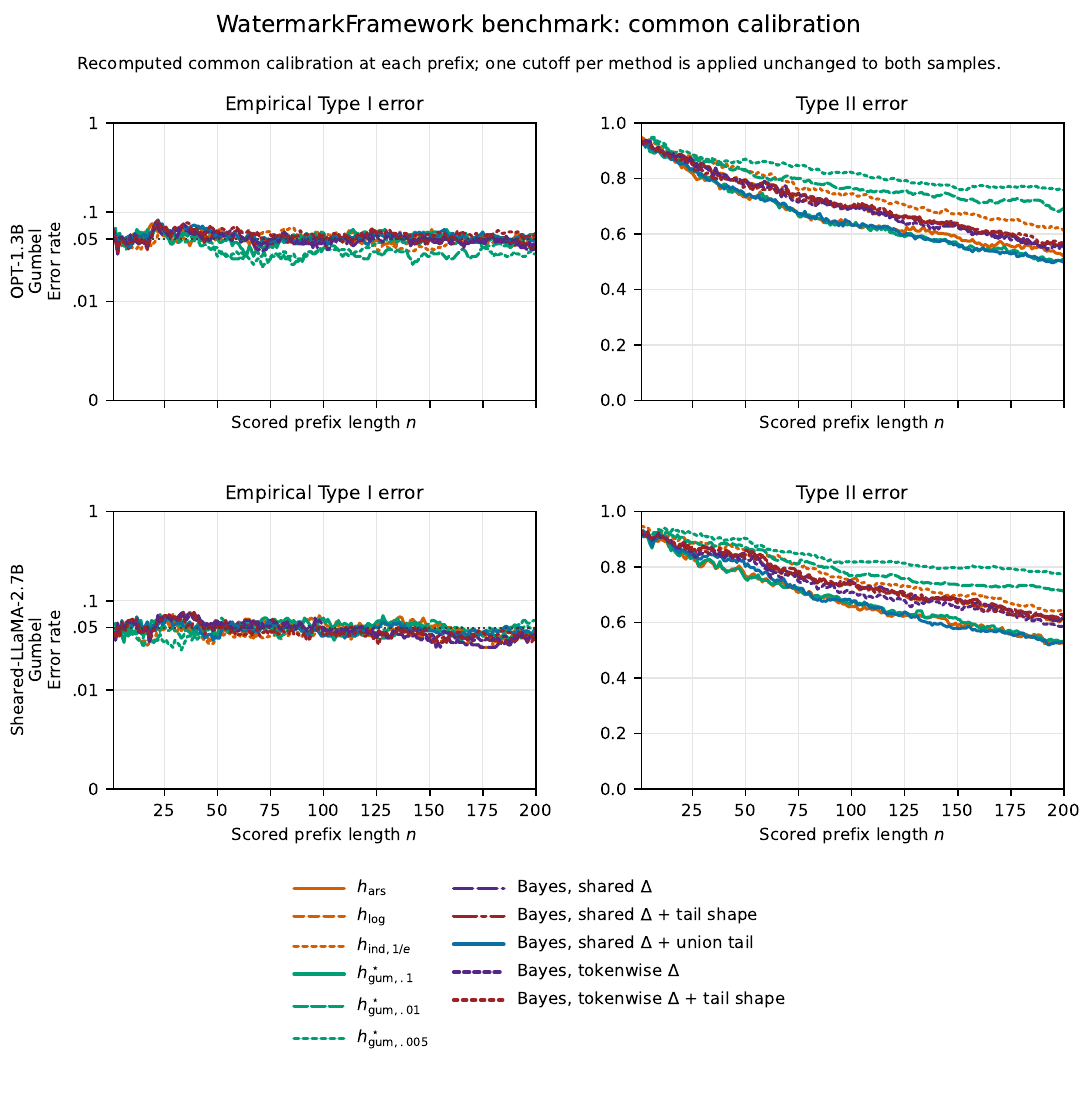}
\caption{Full-prefix empirical errors on the archived benchmark data of \citet{li2025framework} under the independent exact-null calibration used in Table~\ref{tab:released-output}.  Rows distinguish the two models; columns show Type~I and Type~II error.  Each curve uses 500 documents per arm.  At every prefix, one seeded sample of 10,000 exact-null paths supplies method-specific randomized 5\% cutoffs applied unchanged to both arms.  Colours group method families, and line patterns distinguish members; Bayesian labels follow Figure~\ref{fig:shared}, with $\Delta\sim\Unif(.001,.5)$ and union-tail weight $w=.5$.  Rates integrate cutoff-boundary randomization.  The Type~I axes are linear below $.01$ and logarithmic above it; the dotted line marks $.05$.  Type~II axes are linear.  The archived benchmark arms use different generation temperatures, so these curves do not estimate matched-temperature power.}
\label{fig:released-common-prefix}
\end{figure}

\begin{table}[htbp]
\centering
\scriptsize
\setlength{\tabcolsep}{3pt}
\caption{Endpoint performance at $n=200$ on the archived benchmark data of \citet{li2025framework}.  All Bayesian rows use $\Delta\sim\Unif(.001,.5)$.  Entries are empirical Type~I errors on 500 unwatermarked continuations and Type~II errors on 500 separately watermarked continuations; rates integrate the auxiliary cutoff-boundary randomization.  Parentheses give plug-in across-document Monte Carlo standard errors that treat the 500 conditional rejection contributions as independent.  They do not include cutoff-calibration, corpus-selection, key-to-key, or repeated-generation uncertainty.  Boldface marks the lowest observed Type~II error in each model column.  The tail-shape rows use the Dirichlet prior, while the union-tail row mixes tail shape and width with $w=.5$.  The archived benchmark arms use different temperatures, so these entries do not estimate matched-temperature power.}
\label{tab:released-output}
\begin{tabular}{@{}lrrrr@{}}
\toprule
& \multicolumn{2}{c}{OPT-1.3B} & \multicolumn{2}{c}{Sheared-LLaMA-2.7B} \\
\cmidrule(lr){2-3}\cmidrule(lr){4-5}
Method & Type~I & Type~II & Type~I & Type~II \\
\midrule
$h_{\mathrm{ars}}$ & \mcse{.0520}{.0099} & \mcse{.5200}{.0224} & \mcse{.0400}{.0088} & \bestmcse{.5240}{.0224} \\
$h_{\log}$ & \mcse{.0420}{.0090} & \mcse{.5680}{.0222} & \mcse{.0500}{.0098} & \mcse{.6160}{.0218} \\
$h_{\mathrm{ind},1/e}$ & \mcse{.0519}{.0099} & \mcse{.6125}{.0218} & \mcse{.0423}{.0088} & \mcse{.6485}{.0213} \\
$h^\star_{\mathrm{gum},.1}$ & \mcse{.0580}{.0105} & \mcse{.5100}{.0224} & \mcse{.0400}{.0088} & \mcse{.5320}{.0223} \\
$h^\star_{\mathrm{gum},.01}$ & \mcse{.0420}{.0090} & \mcse{.6920}{.0207} & \mcse{.0600}{.0106} & \mcse{.7120}{.0203} \\
$h^\star_{\mathrm{gum},.005}$ & \mcse{.0340}{.0081} & \mcse{.7580}{.0192} & \mcse{.0480}{.0096} & \mcse{.7740}{.0187} \\
Bayes, tokenwise $\Delta$ & \mcse{.0460}{.0094} & \mcse{.5520}{.0223} & \mcse{.0420}{.0090} & \mcse{.5900}{.0220} \\
Bayes, shared $\Delta$ & \mcse{.0420}{.0090} & \mcse{.5660}{.0222} & \mcse{.0480}{.0096} & \mcse{.6160}{.0218} \\
Bayes, tokenwise $\Delta$ $+$ tail shape & \mcse{.0600}{.0106} & \mcse{.5740}{.0221} & \mcse{.0440}{.0092} & \mcse{.6220}{.0217} \\
Bayes, shared $\Delta$ $+$ tail shape & \mcse{.0440}{.0092} & \mcse{.5680}{.0222} & \mcse{.0360}{.0083} & \mcse{.6240}{.0217} \\
Bayes, shared $\Delta$ $+$ union tail & \mcse{.0480}{.0096} & \bestmcse{.4980}{.0224} & \mcse{.0420}{.0090} & \mcse{.5280}{.0223} \\
\addlinespace
\bottomrule
\end{tabular}
\end{table}

Table~\ref{tab:released-output} gives the endpoint summary.  At $n=200$, the best reference-score Gumbel Type~II errors are $.510$ for OPT-1.3B and $.524$ for Sheared-LLaMA-2.7B.  Every Bayesian rule that fixes the tail at the full vocabulary width has higher error, with minima $.552$ and $.590$; the union-tail rule instead attains $.498$ and $.528$, the lowest displayed error on OPT-1.3B and the second-lowest, after $h_{\mathrm{ars}}$, on Sheared-LLaMA-2.7B.  Neither margin is statistically resolved: against $h^\star_{\mathrm{gum},.1}$ on OPT-1.3B the paired discordance is 29 versus 35 (exact two-sided McNemar $p=.53$), and against $h_{\mathrm{ars}}$ on Sheared-LLaMA-2.7B it is 35 versus 33 ($p=.90$).  Within the Bayesian family, the union-tail rule improves on the shared Dirichlet tail-shape rule by 79 versus 44 discordances on OPT-1.3B ($p=.0020$) and 94 versus 46 on Sheared-LLaMA-2.7B ($p=6.1\times10^{-5}$).  These rankings use common simulated-null calibration but are not matched on empirical Type~I error in the archived unwatermarked benchmark outputs.  They describe these particular archived samples and do not establish a population ordering.

\FloatBarrier
\subsubsection{Effective tail width and model fit}
\label{sec:released-width}

Our reanalysis of the archived benchmark data also assesses effective tail width against the full-width
assumption of the equal-tail Bayesian rules.  Because the recorded top probability
fixes $\Delta_t$, the exact density \eqref{eq:gumbel-tailwidth} has $J$ as its
only free parameter, so its likelihood can be profiled over $J$.  The archived token positions
are not independent: the skipgram function seeds from the token four back, so
positions sharing that address share a pseudorandom vector.  The estimate
therefore uses the first occurrence of each address, selected before the deficit
filter, because a later occurrence is conditioned on a decoding path that includes
the earlier keyed draw.
This leaves $1{,}083$ and $989$ informative first-use positions for OPT-1.3B and
Sheared-LLaMA-2.7B, respectively, with $\Delta_t\in[.05,.5]$ and each position
retaining its own deficit.  The profile maximum is $J=2$ for OPT-1.3B and $J=1$
for Sheared-LLaMA-2.7B; the log likelihood falls by $20.0$ and $35.0$ nats,
respectively, by $J=4$.  OPT-1.3B's advantage for $J=2$ over $J=1$ is only
$2.2$ nats.

First-use selection removes repeated pseudorandom-vector reuse within the
retained sample, but not dependence within a continuation, whose positions
share a prompt and a decoding path.  Each bootstrap replicate therefore
resamples 385 OPT-1.3B documents and 368 Sheared-LLaMA-2.7B documents with
replacement and retains every selected position of each.  Across 2000
replicates, $J=1$ is selected in $94\%$ of Sheared-LLaMA-2.7B replicates and
$J=2$ in $66\%$ of OPT-1.3B replicates; these frequencies
describe resampling stability, not calibrated significance tests.
Deduplication by the rounded pivot instead selects
$J=1$ for both.  That reduction is not valid here: the likelihood is
$f_{\Delta,J}(r)$, and dropping other members of a repeated-pivot group discards
their deficits, which differ in all $631$ and $776$ such groups.  The data favor
an effective width of one or two coordinates, not a single common value.

The fitted width reduces the lack of fit relative to the full-width assumption,
but residual lack of fit remains.  On the same first-use sample, the probability-integral
transform gives Kolmogorov--Smirnov $D=.0339$ at $p=.16$ for OPT-1.3B but
$D=.0445$ at $p=.038$ for Sheared-LLaMA-2.7B, nominally below $.05$, whereas the
assumed full width $J=M-1$ gives $D=.163$ and $D=.152$, at
$p=1\times10^{-25}$ and $p=3\times10^{-20}$.  Both fitted-width $p$-values are
descriptive rather than calibrated: $J$ is selected on the observations used
for the test, and the fixed-distribution Kolmogorov--Smirnov reference law does
not account for that selection or the remaining within-document dependence.
A calibration would need to repeat the profiling step and account for that
dependence.  The full-width tests fix $J$ and avoid the selection issue but
share the dependence, so their $p$-values are nominal too.  They summarize
substantial full-width misfit, not a calibrated significance level.

Independent simulated pivots assess width recovery at the two observed sample
sizes.  Over 200 replicates per width at $n=1{,}083$ and $n=989$, the truth
$J=1$ is recovered in $100\%$ of replicates and $J=2$ in $98.5\%$ and $96\%$,
respectively.  Recovery decreases to $89.5\%$ at $J=4$ and $93.0\%$ at $J=16$
in the lower-performing of the two sample sizes; recovery of $J=256$ is
$100\%$.  An earlier check used $8{,}000$ simulated positions, roughly eight
times the size of either model sample, and recovered the generating width in every replicate at all five widths.
That larger-sample result does not establish recovery in every replicate at the observed
sample sizes.  These simulations assess recovery under the generating family,
not the adequacy of that family for deployed outputs.  Here $J$ is an effective
width: a deployed tail need not be exactly equal, so $J=1$ means the residual
mass behaves as though carried by one coordinate, not that only one is nonzero.
The complete width-profile and bootstrap summaries are indexed in Section~\ref{sec:computational-materials}.

The stored watermarked NTP summaries remain poorly resolved near the lower endpoint: across the two models' archived Gumbel samples, $81.8\%$--$83.1\%$ of the recorded float32 top-probability deficits $1-\max_w p_{t,w}$ are below $.001$, and $66.9\%$--$69.9\%$ are recorded as zero.  These finite-precision zeros need not equal the underlying model probabilities, and the summaries do not identify the cause of each power difference.  Across the eleven Gumbel methods tabulated here and the three evaluated horizons, empirical Type~I error on the archived unwatermarked benchmark outputs ranges from $.032$ to $.064$ under the common calibration, the lower endpoint at $h^\star_{\mathrm{gum},.005}$ on OPT-1.3B at $n=50$; the Tr-GoF statistic, scored but not tabulated here, reaches $.074$.  Because the skipgram construction can reuse a seed when the token four positions back repeats, the fresh-pivot condition is not verified.  The reanalysis of the archived benchmark data is therefore restricted to fixed-horizon empirical errors; no anytime-valid claim is made for these sequences.

\FloatBarrier
Numerical checks and cross-experiment interpretation are reported in Supplementary Section~\ref{sec:numerical-checks}.

\subsubsection{Replay checks and regeneration provenance}
\label{sec:released-replay}

All six WatermarkFramework input files used by the replay loader are checked against commit-pinned SHA-256 values before deserialization.  CPU replay of the reference implementation's fixed-table hash and keyed pseudorandom draws is checked against fixed full-array hashes for both models.  For this replay, PyTorch handles verified deserialization and keyed draws; the separate model-generation workflow also uses it.  The common-calibration curves and the authors' archived error-rate records are kept separate.  The former evaluate each score with one fixed cutoff at each prefix.  The stored arrays contain aggregate rates but not the authors' document-level decisions or realized cutoffs, so neither a paired standard error for their difference nor the cutoff variability can be recovered.  Across the six reference scores that have an archived error-rate record, the largest difference between the common-calibration and the authors' archived Type~II rates is $.060$, at $h^\star_{\mathrm{gum},.005}$ on OPT-1.3B at $n=50$, where the common calibration is the more conservative of the two: size $.032$ against $.072$ and Type~II error $.870$ against $.810$.  This joint discrepancy in realized size and Type~II error cannot be attributed solely to binomial variation in the 500 evaluation documents.  The supplementary commit-pinned prefix-curve display reports these machine-readable records without digitizing the historical figure.

Among 100,000 replayed unwatermarked positions, the OPT-1.3B and
Sheared-LLaMA-2.7B samples use 15,475 and 10,862 distinct context tokens,
respectively, to address the pseudorandom function.  The reference implementation of Li et al.
seeds from the oldest token of the four-token window rather than from the
window as a whole.  The seed alphabet is therefore the vocabulary, and $\xi_t$
is a deterministic function of one earlier token instead of a draw that is
fresh given $\F_{t-1}$, as \eqref{eq:pivot-null} assumes.  Seed reuse does not
imply identical pivots, because different realized tokens may select different
keyed coordinates; the marginal pivot distribution may nevertheless appear
uniform.  However, these data do not establish the conditional-null property
required for the sequential results.  Accordingly, both archived-benchmark replay
and temperature-matched regeneration are restricted to fixed-horizon claims.

The archived benchmark tensors are fixed by the cited commit, but the model and C4
revisions used for their original generation were not recorded in that archive.
The supplied \emph{pinned regeneration recipe} cannot recover the unrecorded
historical model and C4 revisions.  It names model checkpoints and tokenizers by commit, identifies the C4
slice by repository revision and SHA-256, and requires the WatermarkFramework checkout
to be clean at the expected commit.  The included prompt-construction script
fetches that pinned slice by digest and fails unless its first 500 rows
reproduce the archived benchmark prompt tensor exactly.  The generator validates custom
prompt sidecars, token values, requested row counts, and both table and selected-row
digests before loading a model.  It seeds the unwatermarked arm and records
model and tokenizer revisions, prompt digests, seeds, batch size, device,
software versions, code hashes, and the verified framework commit in each new
array file. Raw-arm streams are indexed by global prompt number and prompt-row
digest, preventing a restart across separately generated blocks. Batch size
remains part of the run identity because model arithmetic can depend on it.  A resume must match the existing identity and all completed
arm arrays; it may append temperatures or methods, preserving the original
record and extending the completed-arm manifest.  Checkpoints are replaced
atomically.  Legacy or incompatible files are rejected without mutation and
require a new output path. Downstream analysis verifies the selected prompt
rows, array hashes, run compatibility, nonoverlapping offsets and paired arms.
Explicit legacy opt-in accepts only the eight SHA-256-allowlisted historical
archives.

The eight stored array files for this study's temperature-matched generations
predate this run-record format and have not been regenerated or retroactively
assigned new metadata.
Their recorded provenance is consequently weaker than that of a new run under
the revised recipe; the accompanying snapshot manifest fingerprints the
stored files but does not verify their historical generation settings.
An earlier regeneration recipe restarted the raw-arm random stream for the
base and extension prompt blocks.  Whether the stored runs used that recipe
cannot be established from their arrays.  Their prompt-cluster uncertainty
calculations therefore remain conditional on independent raw-generation
streams across prompts, an assumption not certified by the historical records.
The reported analyses can be replayed from those stored arrays.  Because the
revised generation safeguards were tested without a complete model-generation
rerun, they validate the procedure but do not establish end-to-end reproducibility
or bitwise equality across hardware.
Figure~\ref{fig:released-upstream-prefix} plots the reference-score curves archived by \citet{li2025framework},
retaining their original cutoff construction.
The companion benchmark-data summaries and decision arrays contain file and replay hashes, cutoffs, error rates, and document-level decisions; Section~\ref{sec:computational-materials} identifies these files.

\begin{figure}[htbp]
\centering
\includegraphics[width=.85\textwidth]{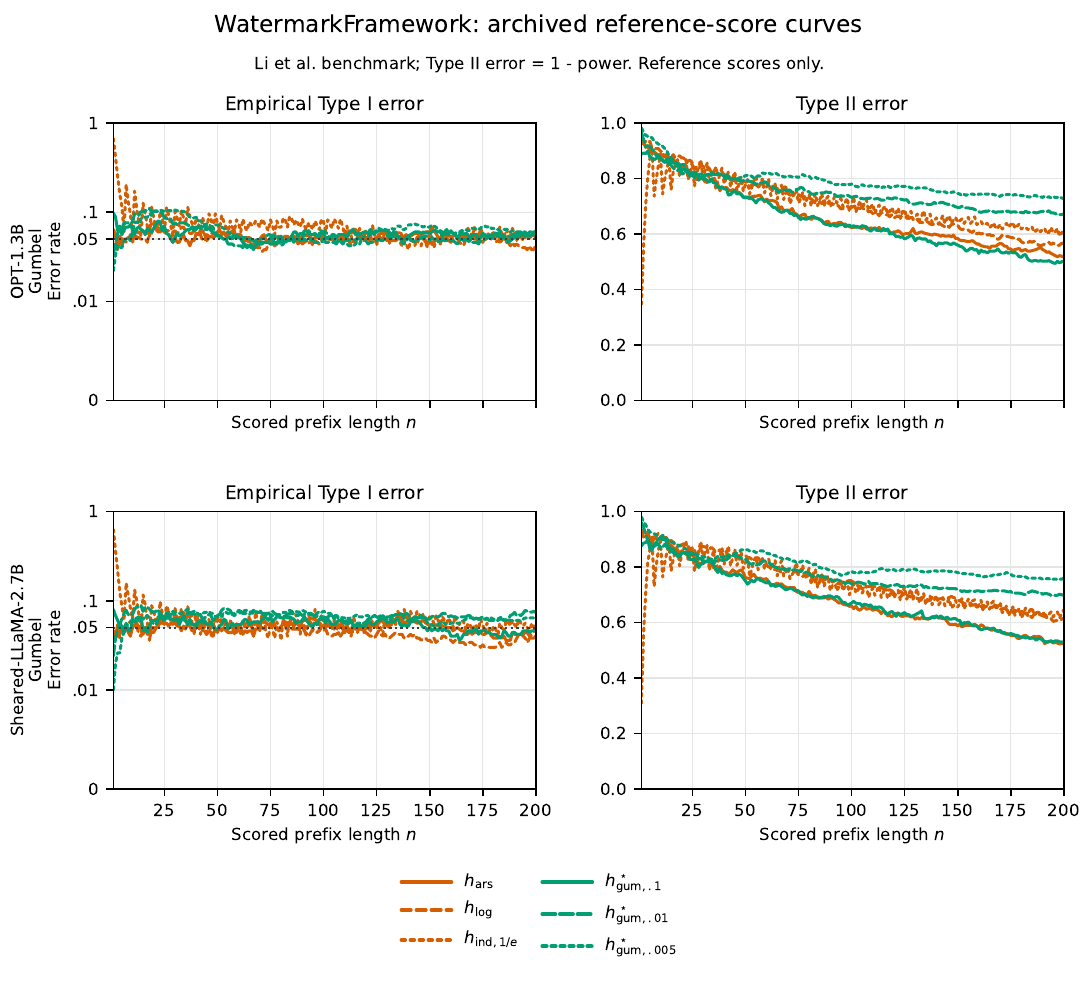}
\caption{Reference-score error curves reconstructed from the error-rate arrays archived by \citet{li2025framework} in WatermarkFramework.  The four Gumbel 200-entry \texttt{*-null.json} and \texttt{*-result.json} arrays supply Type~I error and power, respectively; Type~II error is one minus power.  No curve is digitized from a figure.  Only reference scores are recorded in that archive, so no Bayesian rule appears.  Models, method labels, axis scales, and the $.05$ Type~I reference line follow Figure~\ref{fig:released-common-prefix}.  These curves retain the cutoff construction used by Li et al. and are not used in the common-calibration endpoint table.}
\label{fig:released-upstream-prefix}
\end{figure}

\FloatBarrier
\subsection{Additional synthetic experiments}
\label{sec:synthetic-details}

\subsubsection{Numerical settings and boundary randomization}
\label{sec:synthetic-design-details}

\begin{table}[htbp]
\centering
\caption{Common synthetic numerical design.  The generating deficit range agrees with the Bayes prior, whereas the latent structure agrees only with the corresponding hierarchy.  The Gumbel spike component equals the generating family; the Dirichlet rule averages over the listed tail shapes, including the equal-tail atom.  Entries are fixed design settings rather than Monte Carlo estimates.}
\label{tab:design}
\begin{tabular}{@{}ll@{}}
\toprule
Quantity & Setting \\
\midrule
Vocabulary & $M=1000$ \\
Clean horizons & every prefix $n=1,\ldots,700$ \\
Robustness horizons & $n\in\{100,300,700\}$ \\
Fixed-horizon level & $\alpha_{\mathrm{test}}=0.05$, calibrated independently for each $n$ \\
Monte Carlo samples & 10,000 calibration; 5,000 null; 5,000 alternative \\
Deficit range / prior & $\Delta\sim\Unif(0.001,0.5)$ \\
Tail-concentration prior (Dirichlet rules) & $\alpha\sim\Unif\{0.1,1,10,100,1000,\infty\}$ \\
Quadrature & 96 Gauss--Legendre nodes in $\Delta$; 576-component \\
 & $(\Delta,\alpha)$ grid and a $200{,}001$-node $\psi_\alpha$ table \\
 & for the Dirichlet rules \\
Random seeds & clean 240401245; contamination 240401246; \\
 & contamination quadrature 240401247; \\
 & deficit sweep and $\Delta$-learning diagnostic 240401248; \\
 & tail sweep 240401250; contamination-asymmetry diagnostic \\
 & 240401251; \\
 & max-regret bootstrap salts 91730041 (deficit), 91730042 (tail) \\
Contamination settings & Supplementary Section~\ref{sec:contamination-supplement} \\
\bottomrule
\end{tabular}
\end{table}

At an atom of the score distribution, the rejection probability $\gamma$ at the boundary is chosen to attain 5\% size in the calibration sample.  Aggregate error rates use the Rao--Blackwellized expectation over this auxiliary randomization.  Paired comparisons use one realized auxiliary uniform, common to every method for the same document and evaluation cell.  Across the 81 evaluation cells, 274 documents fall on a score atom, all for the discrete $h_{\mathrm{ind},1/e}$; continuous scores are unaffected.

\FloatBarrier
\subsubsection{Shared-deficit endpoint results}
\label{sec:shared-endpoints}

Table~\ref{tab:shared} gives selected horizons from the shared-deficit
experiment in Section~\ref{sec:shared-sensitivity} and Figure~\ref{fig:shared}.

\begin{table}[htbp]
\centering
\small
\caption{Fixed-horizon Type~II error in the shared-$\Delta$ sensitivity experiment.  Each method is independently calibrated at nominal 5\% Type~I error.  Parentheses give Monte Carlo standard errors.  Bold identifies the lowest displayed error at each horizon, and all entries attaining it when there is a tie.  The two Dirichlet rows replace the equal-tail Gumbel component by the Dirichlet tail-shape mixture \eqref{eq:gumbel-dirichlet}.}
\label{tab:shared}
\begin{tabular}{@{}llrrr@{}}
\toprule
Scheme & Method & $n=100$ & $n=300$ & $n=700$ \\
\midrule
Gumbel & $h_{\mathrm{ars}}$ & \mcse{.0426}{.0029} & \mcse{.0192}{.0019} & \mcse{.0094}{.0014} \\
 & $h_{\log}$ & \mcse{.1698}{.0053} & \mcse{.0984}{.0042} & \mcse{.0654}{.0035} \\
 & $h_{\mathrm{ind},1/e}$ & \mcse{.2192}{.0057} & \mcse{.1309}{.0048} & \mcse{.0847}{.0039} \\
 & $h^\star_{\mathrm{gum},.1}$ & \mcse{.0934}{.0041} & \mcse{.0556}{.0032} & \mcse{.0342}{.0026} \\
 & $h^\star_{\mathrm{gum},.01}$ & \mcse{.0420}{.0028} & \mcse{.0178}{.0019} & \mcse{.0098}{.0014} \\
 & $h^\star_{\mathrm{gum},.005}$ & \mcse{.0328}{.0025} & \mcse{.0124}{.0016} & \mcse{.0062}{.0011} \\
 & Bayes, tokenwise $\Delta$ & \mcse{.0246}{.0022} & \mcse{.0096}{.0014} & \mcse{.0042}{.0009} \\
 & Bayes, tokenwise $\Delta$ $+$ tail shape & \mcse{.0274}{.0023} & \mcse{.0098}{.0014} & \mcse{.0042}{.0009} \\
 & Bayes, shared $\Delta$ & \mcse{.0194}{.0020} & \bestmcse{.0056}{.0011} & \bestmcse{.0014}{.0005} \\
 & Bayes, shared $\Delta$ $+$ tail shape & \mcse{.0194}{.0020} & \mcse{.0058}{.0011} & \bestmcse{.0014}{.0005} \\
 & Bayes, shared $\Delta$ $+$ union tail & \bestmcse{.0192}{.0019} & \mcse{.0058}{.0011} & \bestmcse{.0014}{.0005} \\
\bottomrule
\end{tabular}
\end{table}

For $h^\star_{\mathrm{gum},.005}$, the equal-tail resimulation gives Type~II errors $0.0328$, $0.0124$, and $0.0062$ at $n=100,300,700$, respectively.  It differs from the accompanying simulation snapshot by replacing token-varying normalized-uniform tails with \eqref{eq:spike-ntp}.

\FloatBarrier
\subsubsection{Regret across generating regimes}
\label{sec:regimes}

To assess sensitivity to the generating distribution of $\Delta$, each rule is calibrated once under the exact null and applied unchanged to six regimes.  Regimes A and B use $\Delta\sim\Unif(.001,.5)$ and $\Unif(.001,.05)$; the remaining regimes fix $\Delta\in\{.002,.005,.0075,.02\}$.  Of these point values, only $.005$ equals a tested $\Delta_0$; it is also the tuning value of $h^\star_{\mathrm{gum},.005}$.  The value $.002$ lies below the tested point grid, while $.0075$ and $.02$ lie within the two smallest intervals between its grid points.  The Bayes rules retain their prespecified deficit and tail priors throughout.  The equal-tail generator makes the tail-shape dimension unnecessary in these regimes, so the effect of averaging over it is measured empirically.

Table~\ref{tab:regret} reports the Type~II error of every eligible competitor in each regime at each horizon, alongside the resulting maximum regret, where regret in a regime is the excess Type~II error relative to the best displayed rule there.

\begin{table}[p]
\centering
\scriptsize
\renewcommand{\arraystretch}{0.85}
\setlength{\tabcolsep}{3pt}
\setlength{\aboverulesep}{0pt}
\setlength{\belowrulesep}{0pt}
\setlength{\extrarowheight}{0pt}

\caption{Deficit-regime sweep for the shared-$\Delta$ Gumbel generator.  Entries are Type~II errors at nominal 5\% size; maximum regret is the largest excess over the best rule scored in the sweep, across the six regimes.  Parentheses give binomial Monte Carlo standard errors for error rates and paired-bootstrap standard errors for maximum regret.  Regimes A and B use $\Delta\sim\Unif(.001,.5)$ and $\Unif(.001,.05)$; the remaining columns fix $\Delta$.  Bold marks observed column minima, ties included.  Tail shape is the Dirichlet layer \eqref{eq:gumbel-dirichlet} at concentration $\alpha$; the union tail \eqref{eq:union-tail} mixes it with the width ladder \eqref{eq:gumbel-tailwidth}.  Fixed-deficit equal-tail scores are treated as diagnostics and excluded from the regret benchmark; they need not match the generating deficit or tail family.  Table~\ref{tab:family} gives a clean-benchmark illustration; the computational-materials index in Supplementary Section~\ref{sec:computational-materials} identifies the full diagnostics for this sweep.  The $h^\star_{\mathrm{gum},\Delta_0}$ rows are the least-favorable score of \citet{li2025framework} at its three published tunings.}
\label{tab:regret}
\begin{tabular}{@{}lrrrrrrcr@{}}
\toprule
& \multicolumn{2}{c}{Uniform} & \multicolumn{4}{c}{Point deficit} & & Max \\
\cmidrule(lr){2-3}\cmidrule(lr){4-7}\cmidrule(lr){9-9}
Rule & A & B & $.002$ & $.005$ & $.0075$ & $.02$ & & regret \\
\midrule
\multicolumn{9}{@{}l}{\textit{Horizon $n=100$}}\\
$h_{\mathrm{ars}}$ & \mcse{.0386}{.0027} & \mcse{.3752}{.0068} & \mcse{.8792}{.0046} & \mcse{.8034}{.0056} & \mcse{.7340}{.0062} & \mcse{.4034}{.0069} & & \mcse{.2902}{.0059} \\
$h_{\log}$ & \mcse{.1620}{.0052} & \mcse{.8686}{.0048} & \mcse{.9438}{.0033} & \mcse{.9442}{.0032} & \mcse{.9384}{.0034} & \mcse{.8920}{.0044} & & \mcse{.7730}{.0061} \\
$h_{\mathrm{ind},1/e}$ & \mcse{.2183}{.0058} & \mcse{.9017}{.0042} & \mcse{.9470}{.0032} & \mcse{.9428}{.0033} & \mcse{.9475}{.0032} & \mcse{.9125}{.0040} & & \mcse{.7935}{.0056} \\
$h^\star_{\mathrm{gum},.1}$ & \mcse{.0878}{.0040} & \mcse{.7126}{.0064} & \mcse{.9346}{.0035} & \mcse{.9238}{.0038} & \mcse{.9110}{.0040} & \mcse{.7946}{.0057} & & \mcse{.6756}{.0068} \\
$h^\star_{\mathrm{gum},.01}$ & \mcse{.0396}{.0028} & \mcse{.3696}{.0068} & \mcse{.8986}{.0043} & \mcse{.8244}{.0054} & \mcse{.7600}{.0060} & \mcse{.3966}{.0069} & & \mcse{.3162}{.0071} \\
$h^\star_{\mathrm{gum},.005}$ & \mcse{.0306}{.0024} & \mcse{.3002}{.0065} & \mcse{.8754}{.0047} & \mcse{.7606}{.0060} & \mcse{.6760}{.0066} & \mcse{.2900}{.0064} & & \mcse{.2322}{.0063} \\
\addlinespace[1pt]
Bayes, tokenwise $\Delta$ & \mcse{.0226}{.0021} & \mcse{.2330}{.0060} & \mcse{.8196}{.0054} & \mcse{.6648}{.0067} & \mcse{.5606}{.0070} & \mcse{.1926}{.0056} & & \mcse{.1168}{.0048} \\
Bayes, shared $\Delta$ & \bestmcse{.0160}{.0018} & \mcse{.1722}{.0053} & \mcse{.7744}{.0059} & \mcse{.5664}{.0070} & \mcse{.4458}{.0070} & \bestmcse{.1190}{.0046} & & \mcse{.0036}{.0018} \\
Bayes, shared $\Delta$ $+$ tail shape & \mcse{.0166}{.0018} & \bestmcse{.1716}{.0053} & \bestmcse{.7718}{.0059} & \mcse{.5636}{.0070} & \bestmcse{.4438}{.0070} & \bestmcse{.1190}{.0046} & & \bestmcse{.0008}{.0016} \\
Bayes, shared $\Delta$ $+$ union tail & \mcse{.0164}{.0018} & \mcse{.1724}{.0053} & \mcse{.7726}{.0059} & \bestmcse{.5628}{.0070} & \mcse{.4454}{.0070} & \mcse{.1200}{.0046} & & \mcse{.0016}{.0016} \\
\addlinespace[3pt]
\multicolumn{9}{@{}l}{\textit{Horizon $n=300$}}\\
$h_{\mathrm{ars}}$ & \mcse{.0172}{.0018} & \mcse{.1840}{.0055} & \mcse{.8586}{.0049} & \mcse{.6932}{.0065} & \mcse{.5582}{.0070} & \mcse{.1074}{.0044} & & \mcse{.4858}{.0071} \\
$h_{\log}$ & \mcse{.0932}{.0041} & \mcse{.7626}{.0060} & \mcse{.9456}{.0032} & \mcse{.9382}{.0034} & \mcse{.9190}{.0039} & \mcse{.8310}{.0053} & & \mcse{.8296}{.0049} \\
$h_{\mathrm{ind},1/e}$ & \mcse{.1267}{.0047} & \mcse{.8312}{.0053} & \mcse{.9454}{.0032} & \mcse{.9367}{.0034} & \mcse{.9315}{.0036} & \mcse{.8749}{.0047} & & \mcse{.8735}{.0048} \\
$h^\star_{\mathrm{gum},.1}$ & \mcse{.0506}{.0031} & \mcse{.4882}{.0071} & \mcse{.9352}{.0035} & \mcse{.9056}{.0041} & \mcse{.8730}{.0047} & \mcse{.6100}{.0069} & & \mcse{.7760}{.0061} \\
$h^\star_{\mathrm{gum},.01}$ & \mcse{.0174}{.0018} & \mcse{.1802}{.0054} & \mcse{.8732}{.0047} & \mcse{.7014}{.0065} & \mcse{.5686}{.0070} & \mcse{.0804}{.0038} & & \mcse{.4940}{.0071} \\
$h^\star_{\mathrm{gum},.005}$ & \mcse{.0142}{.0017} & \mcse{.1366}{.0049} & \mcse{.8352}{.0052} & \mcse{.5996}{.0069} & \mcse{.4420}{.0070} & \mcse{.0352}{.0026} & & \mcse{.3922}{.0071} \\
\addlinespace[1pt]
Bayes, tokenwise $\Delta$ & \mcse{.0094}{.0014} & \mcse{.1054}{.0043} & \mcse{.7614}{.0060} & \mcse{.4808}{.0071} & \mcse{.3048}{.0065} & \mcse{.0140}{.0017} & & \mcse{.2734}{.0065} \\
Bayes, shared $\Delta$ & \bestmcse{.0040}{.0009} & \mcse{.0492}{.0031} & \bestmcse{.5120}{.0071} & \mcse{.2078}{.0057} & \bestmcse{.0970}{.0042} & \bestmcse{.0014}{.0005} & & \bestmcse{.0004}{.0013} \\
Bayes, shared $\Delta$ $+$ tail shape & \mcse{.0048}{.0010} & \mcse{.0500}{.0031} & \mcse{.5136}{.0071} & \mcse{.2076}{.0057} & \mcse{.0976}{.0042} & \bestmcse{.0014}{.0005} & & \mcse{.0016}{.0016} \\
Bayes, shared $\Delta$ $+$ union tail & \mcse{.0048}{.0010} & \bestmcse{.0490}{.0031} & \mcse{.5122}{.0071} & \bestmcse{.2074}{.0057} & \mcse{.0986}{.0042} & \mcse{.0026}{.0007} & & \mcse{.0016}{.0012} \\
\addlinespace[3pt]
\multicolumn{9}{@{}l}{\textit{Horizon $n=700$}}\\
$h_{\mathrm{ars}}$ & \mcse{.0080}{.0013} & \mcse{.1022}{.0043} & \mcse{.7970}{.0057} & \mcse{.5034}{.0071} & \mcse{.3042}{.0065} & \mcse{.0052}{.0010} & & \mcse{.5700}{.0073} \\
$h_{\log}$ & \mcse{.0610}{.0034} & \mcse{.5834}{.0070} & \mcse{.9376}{.0034} & \mcse{.9220}{.0038} & \mcse{.8940}{.0044} & \mcse{.7072}{.0064} & & \mcse{.8904}{.0036} \\
$h_{\mathrm{ind},1/e}$ & \mcse{.0803}{.0038} & \mcse{.7033}{.0065} & \mcse{.9432}{.0033} & \mcse{.9259}{.0037} & \mcse{.9139}{.0040} & \mcse{.7930}{.0057} & & \mcse{.9103}{.0040} \\
$h^\star_{\mathrm{gum},.1}$ & \mcse{.0340}{.0026} & \mcse{.3236}{.0066} & \mcse{.9246}{.0037} & \mcse{.8680}{.0048} & \mcse{.7988}{.0057} & \mcse{.3548}{.0068} & & \mcse{.8362}{.0052} \\
$h^\star_{\mathrm{gum},.01}$ & \mcse{.0088}{.0013} & \mcse{.1052}{.0043} & \mcse{.8270}{.0053} & \mcse{.5288}{.0071} & \mcse{.3062}{.0065} & \mcse{.0038}{.0009} & & \mcse{.6000}{.0073} \\
$h^\star_{\mathrm{gum},.005}$ & \mcse{.0068}{.0012} & \mcse{.0766}{.0038} & \mcse{.7582}{.0061} & \mcse{.3696}{.0068} & \mcse{.1660}{.0053} & \mcse{.0002}{.0002} & & \mcse{.5312}{.0073} \\
\addlinespace[1pt]
Bayes, tokenwise $\Delta$ & \mcse{.0040}{.0009} & \mcse{.0544}{.0032} & \mcse{.6670}{.0067} & \mcse{.2506}{.0061} & \mcse{.0876}{.0040} & \bestmcse{.0000}{.0000} & & \mcse{.4400}{.0074} \\
Bayes, shared $\Delta$ & \bestmcse{.0012}{.0005} & \mcse{.0146}{.0017} & \mcse{.2292}{.0059} & \mcse{.0324}{.0025} & \bestmcse{.0036}{.0008} & \bestmcse{.0000}{.0000} & & \mcse{.0022}{.0013} \\
Bayes, shared $\Delta$ $+$ tail shape & \mcse{.0014}{.0005} & \bestmcse{.0142}{.0017} & \bestmcse{.2270}{.0059} & \bestmcse{.0318}{.0025} & \bestmcse{.0036}{.0008} & \bestmcse{.0000}{.0000} & & \bestmcse{.0002}{.0004} \\
Bayes, shared $\Delta$ $+$ union tail & \bestmcse{.0012}{.0005} & \mcse{.0158}{.0018} & \mcse{.2292}{.0059} & \mcse{.0378}{.0027} & \mcse{.0064}{.0011} & \bestmcse{.0000}{.0000} & & \mcse{.0060}{.0012} \\
\bottomrule
\end{tabular}
\end{table}

When the generating value is $\Delta=.005$, $h^{\mathrm{sp}}_{.005}$ has the lowest observed error at $n=300$, $.2066$ versus $.2078$ for the shared mixture; it does not attain the minimum at the three unmatched point deficits.  Across $n=100,300,700$, the maximum regrets are $.0036,.0004,.0022$ for the equal-tail mixture, $.0008,.0016,.0002$ for the tail-shape mixture, and $.0016,.0016,.0060$ for the union-tail mixture.  The tail-shape mixture attains the smallest maximum regret of the ten displayed rules at $n=100$ and $n=700$, and the equal-tail mixture at $n=300$; the union-tail mixture is smallest at none of the three.  The diagnostic is excluded from the regret benchmark, so its maximum regrets $.0258,.0028,.0016$ do not define the regime-specific minima.  At $n=100$, the three published tunings of $h^\star_{\mathrm{gum}}$ have maximum regrets $.2322$, $.3162$ and $.6756$ for $\Delta_0=.005$, $.01$ and $.1$, compared with $.2902$ for $h_{\mathrm{ars}}$.  Thus, variation across the three tuning constants exceeds the difference between the best of these tunings and $h_{\mathrm{ars}}$ in this comparison.  Because these regimes are equal tailed by construction, the tail-width block of \eqref{eq:union-tail} does not represent variation in the generating law; averaging over that block reduces prior mass on the matched equal-tail component.  In the temperature-matched experiment of Section~\ref{sec:temperature-matched}, by contrast, the estimated tails are sparse and the tail-width block accounts for the improvement within the Bayesian family.  No tested fixed-$\Delta_0$ rule has maximum regret no greater than the shared spike mixture at all three horizons, and none does so for either enlarged mixture.  The comparison does not cover the continuum of fixed deficits.

Because regret is measured against the best displayed rule, adding the tail-shape rule changes the reference minimum in cells where it has the lowest observed error.  On the present grid this change affects the shared spike rule only at $n=700$, where its maximum regret is $.0022$ with the tail-shape rule included and exactly zero without it.  When the tail-shape rule is excluded, the shared equal-tail rule attains the lowest observed error, including ties, in all six displayed regimes at $n=700$, so its entire $.0022$ regret is the tail-shape rule's margin over it.  At $n=100$ and $n=300$, the tail-shape rule does not attain the minimum in the cell that determines the maximum, so the two comparator sets agree at $.0036$ and $.0004$.

These regret rankings are exploratory.  They combine estimated cutoffs from 10,000 calibration paths with errors from 5,000 evaluation paths, take maxima over six regimes, and select the best observed rule in each cell.  Differences of $.0002$--$.0010$ correspond to one to five evaluation documents.  Because the rules are paired and the statistic includes cellwise selection and maximization, uncertainty cannot be assessed from unpaired binomial standard errors; no simultaneous interval for maximum regret is reported.  The table reports seed-specific regret estimates and does not establish an ordering of population regrets among close entries.  The companion deficit-sweep summaries include three further generating regimes at higher deficits in which no rule misses a document at $n=100$, $300$ or $700$; see Section~\ref{sec:computational-materials}.

The least-favorable minimax score of \citet{li2025framework} has among the largest empirical regrets on this grid because it targets the $\Delta$-regular class rather than the spike generating family.  Its realized Type~I errors, $.0538$, $.0488$ and $.0510$, are near the nominal level, so the difference occurs in power rather than size.  At the point regime $\Delta=.005$, both likelihood ratios use the correct deficit; their difference isolates the assumed NTP profile within this experiment.  The least-favorable score has Type~II errors $.7606$, $.5996$ and $.3696$, against $.5682$, $.2066$ and $.0322$ for $h^{\mathrm{sp}}_{.005}$.  This comparison does not show that profile mismatch dominates arbitrary deficit misspecification; severely mismatched point deficits produce larger differences in some cells.  Under the common generator of Table~\ref{tab:shared}, the profile-mismatch increments at $\Delta_0=.01$ are $.0222$, $.0122$ and $.0084$, while changing $\Delta_0$ from $.005$ to $.01$ within the least-favorable family produces increments $.0092$, $.0054$ and $.0036$.  The former is about $2.3$ times the latter at each horizon.  The tokenwise mixture is misspecified under the persistent document-level deficit generator and records regret between $.1168$ and $.4400$.

\FloatBarrier
\subsubsection{Tail shape: regret across tail laws}
\label{sec:tails}

The tail-law experiment holds the deficit distribution fixed and varies the residual NTP distribution.  Although $\Delta$ constrains the largest NTP coordinate, it does not determine the remaining $K=M-1$ probabilities.  Equal-tail rules fix these probabilities; the Dirichlet layer models them through a symmetric concentration family.  The finite sweep includes generating shapes within the Dirichlet family and one outside it.

The design follows Section~\ref{sec:regimes}: each document has one $\Delta\sim\Unif(.001,.5)$, each rule is calibrated once under the exact Gumbel null, and 5,000 documents are generated per cell.  Because the null is $\Unif(0,1)$ for every tail law, the same calibration sample applies to all regimes.  No rule is retuned.  Only the generating tail law changes:
\begin{enumerate}
  \setlength{\itemsep}{0pt}
  \setlength{\parsep}{0pt}
  \item[T1] equal tail, $\alpha=\infty$ --- the equal-tail specification stated by \citet{li2025framework} and the generator used in the other Gumbel experiments in this study;
  \item[T2] $\operatorname{Dirichlet}(10)$, an atom of the prior grid;
  \item[T3] $\operatorname{Dirichlet}(3)$, in family, not an atom, and variance-matched to T6 as explained in Section~\ref{sec:dirichlet-details};
  \item[T4] $\operatorname{Dirichlet}(0.5)$, in family, not an atom, and further from every atom than T3;
  \item[T5] $\operatorname{Dirichlet}(0.1)$, a strongly sparse tail;
  \item[T6] the normalized-uniform tail used by the accompanying simulation snapshot \citep{li2024watermarkframework}: draw $X_{t,k}\stackrel{\mathrm{iid}}{\sim}\Unif(0,1)$ and set $q_{t,k}=X_{t,k}/\sum_{\ell=1}^{K}X_{t,\ell}$; this law is outside the Dirichlet family at every $\alpha$.
\end{enumerate}
The competitors are the layer with a single assumed tail shape, $\alpha_0\in\{0.1,1,10,100,1000\}$, one for every finite atom of the prior; the layer at $\alpha=\infty$, which is the \emph{spike} shared Gumbel Bayes rule of Table~\ref{tab:shared}; the union tail of \eqref{eq:union-tail}; and the six reference scores $h_{\mathrm{ars}}$, $h_{\log}$, $h_{\mathrm{ind},1/e}$ and $h^\star_{\mathrm{gum},\Delta_0}$ at the three tunings of \citet{li2025framework}.  The diagnostic $h^{\mathrm{sp}}_{.01}$ of Supplementary Section~\ref{sec:family} is also scored but is excluded from the regret benchmark because it fixes one deficit within the generating equal-tail family rather than averaging over a prespecified deficit distribution.  The mixture rule uses the prespecified product prior $\Unif(.001,.5)\otimes\Unif\{0.1,1,10,100,1000,\infty\}$ on a 576-component grid.  This is the clean shared-$\Delta$ Bayes rule with the Dirichlet tail prior used in the other experiments.  This sweep uses separate calibration and evaluation samples, so the $\alpha=\infty$ column need not reproduce Table~\ref{tab:shared} exactly.

At $\alpha=\infty$, the layer agrees with the equal-tail detector to the reported numerical precision.  Supplementary Section~\ref{sec:numerical-checks} gives the normalization, interpolation, pathwise-containment, and implementation-agreement checks.  Realized Type~I error over the 45 rule-by-horizon cells lies in $[.0454,.0608]$.

\begin{table}[p]
\centering
\scriptsize
\renewcommand{\arraystretch}{0.85}
\setlength{\tabcolsep}{3pt}
\setlength{\aboverulesep}{0pt}
\setlength{\belowrulesep}{0pt}
\setlength{\extrarowheight}{0pt}

\caption{Exploratory tail-shape sweep for the shared-$\Delta$ Gumbel generator.  Entries are Type~II errors at nominal 5\% size; maximum regret is the largest excess over the best rule scored in the sweep, across the six tail laws.  Parentheses give binomial Monte Carlo standard errors for error rates and paired-bootstrap standard errors for maximum regret.  Bold marks observed column minima, ties included.  Tail shape is the Dirichlet layer \eqref{eq:gumbel-dirichlet} at concentration $\alpha$; the union tail \eqref{eq:union-tail} mixes it with the width ladder \eqref{eq:gumbel-tailwidth}.  Fixed-deficit equal-tail scores are treated as diagnostics and excluded from the regret benchmark; they need not match the generating deficit or tail family.  Table~\ref{tab:family} gives a clean-benchmark illustration; the computational-materials index in Supplementary Section~\ref{sec:computational-materials} identifies the full diagnostics for this sweep.  The $h^\star_{\mathrm{gum},\Delta_0}$ rows are the least-favorable score of \citet{li2025framework} at its three published tunings.}
\label{tab:tails}
\begin{tabular}{@{}lrrrrrrcr@{}}
\toprule
& \multicolumn{6}{c}{Type~II error by generating tail law} & & Max \\
\cmidrule(lr){2-7}\cmidrule(lr){9-9}
Rule & T1 & T2 & T3 & T4 & T5 & T6 & & regret \\
\midrule
\multicolumn{9}{@{}l}{\textit{Horizon $n=100$}}\\
Tail shape, $\alpha=\infty$ (Table~\ref{tab:shared} rule) & \bestmcse{.0162}{.0018} & \bestmcse{.0160}{.0018} & \mcse{.0184}{.0019} & \mcse{.0192}{.0019} & \mcse{.0210}{.0020} & \mcse{.0174}{.0018} & & \mcse{.0046}{.0009} \\
Tail shape, $\alpha_0=1000$ & \bestmcse{.0162}{.0018} & \bestmcse{.0160}{.0018} & \mcse{.0184}{.0019} & \mcse{.0192}{.0019} & \mcse{.0210}{.0020} & \mcse{.0174}{.0018} & & \mcse{.0046}{.0009} \\
Tail shape, $\alpha_0=100$ & \bestmcse{.0162}{.0018} & \bestmcse{.0160}{.0018} & \mcse{.0184}{.0019} & \mcse{.0192}{.0019} & \mcse{.0208}{.0020} & \mcse{.0174}{.0018} & & \mcse{.0044}{.0009} \\
Tail shape, $\alpha_0=10$ & \mcse{.0164}{.0018} & \bestmcse{.0160}{.0018} & \mcse{.0184}{.0019} & \mcse{.0188}{.0019} & \mcse{.0202}{.0020} & \mcse{.0176}{.0019} & & \mcse{.0038}{.0009} \\
Tail shape, $\alpha_0=1$ & \mcse{.0168}{.0018} & \bestmcse{.0160}{.0018} & \bestmcse{.0178}{.0019} & \mcse{.0186}{.0019} & \mcse{.0172}{.0018} & \mcse{.0180}{.0019} & & \bestmcse{.0014}{.0005} \\
Tail shape, $\alpha_0=0.1$ & \mcse{.0170}{.0018} & \mcse{.0162}{.0018} & \mcse{.0180}{.0019} & \mcse{.0186}{.0019} & \mcse{.0166}{.0018} & \mcse{.0180}{.0019} & & \bestmcse{.0014}{.0005} \\
Bayes, shared $\Delta$ $+$ tail shape & \mcse{.0168}{.0018} & \mcse{.0162}{.0018} & \mcse{.0180}{.0019} & \bestmcse{.0182}{.0019} & \bestmcse{.0164}{.0018} & \mcse{.0180}{.0019} & & \bestmcse{.0014}{.0005} \\
Bayes, shared $\Delta$ $+$ union tail & \mcse{.0172}{.0018} & \mcse{.0164}{.0018} & \mcse{.0192}{.0019} & \mcse{.0204}{.0020} & \mcse{.0200}{.0020} & \bestmcse{.0166}{.0018} & & \mcse{.0036}{.0008} \\
\addlinespace[1pt]
$h_{\mathrm{ars}}$ & \mcse{.0386}{.0027} & \mcse{.0418}{.0028} & \mcse{.0400}{.0028} & \mcse{.0442}{.0029} & \mcse{.0388}{.0027} & \mcse{.0382}{.0027} & & \mcse{.0260}{.0018} \\
$h_{\log}$ & \mcse{.1664}{.0053} & \mcse{.1678}{.0053} & \mcse{.1684}{.0053} & \mcse{.1664}{.0053} & \mcse{.1662}{.0053} & \mcse{.1668}{.0053} & & \mcse{.1518}{.0034} \\
$h_{\mathrm{ind},1/e}$ & \mcse{.2170}{.0058} & \mcse{.2234}{.0059} & \mcse{.2185}{.0058} & \mcse{.2193}{.0059} & \mcse{.2185}{.0058} & \mcse{.2248}{.0059} & & \mcse{.2082}{.0042} \\
$h^\star_{\mathrm{gum},.1}$ & \mcse{.0924}{.0041} & \mcse{.0952}{.0042} & \mcse{.0960}{.0042} & \mcse{.0896}{.0040} & \mcse{.0870}{.0040} & \mcse{.0892}{.0040} & & \mcse{.0792}{.0029} \\
$h^\star_{\mathrm{gum},.01}$ & \mcse{.0368}{.0027} & \mcse{.0416}{.0028} & \mcse{.0360}{.0026} & \mcse{.0418}{.0028} & \mcse{.0340}{.0026} & \mcse{.0344}{.0026} & & \mcse{.0256}{.0020} \\
$h^\star_{\mathrm{gum},.005}$ & \mcse{.0298}{.0024} & \mcse{.0328}{.0025} & \mcse{.0296}{.0024} & \mcse{.0332}{.0025} & \mcse{.0260}{.0023} & \mcse{.0272}{.0023} & & \mcse{.0168}{.0016} \\
\addlinespace[1pt]
\multicolumn{9}{@{}l}{\textit{Horizon $n=300$}}\\
Tail shape, $\alpha=\infty$ (Table~\ref{tab:shared} rule) & \bestmcse{.0040}{.0009} & \mcse{.0060}{.0011} & \mcse{.0046}{.0010} & \mcse{.0062}{.0011} & \mcse{.0066}{.0011} & \mcse{.0058}{.0011} & & \mcse{.0020}{.0007} \\
Tail shape, $\alpha_0=1000$ & \bestmcse{.0040}{.0009} & \mcse{.0060}{.0011} & \mcse{.0046}{.0010} & \mcse{.0062}{.0011} & \mcse{.0066}{.0011} & \mcse{.0058}{.0011} & & \mcse{.0020}{.0007} \\
Tail shape, $\alpha_0=100$ & \bestmcse{.0040}{.0009} & \mcse{.0060}{.0011} & \mcse{.0046}{.0010} & \mcse{.0060}{.0011} & \mcse{.0066}{.0011} & \mcse{.0058}{.0011} & & \mcse{.0020}{.0007} \\
Tail shape, $\alpha_0=10$ & \mcse{.0042}{.0009} & \mcse{.0060}{.0011} & \bestmcse{.0044}{.0009} & \mcse{.0060}{.0011} & \mcse{.0062}{.0011} & \mcse{.0058}{.0011} & & \mcse{.0016}{.0006} \\
Tail shape, $\alpha_0=1$ & \mcse{.0042}{.0009} & \bestmcse{.0058}{.0011} & \mcse{.0046}{.0010} & \mcse{.0056}{.0011} & \bestmcse{.0046}{.0010} & \mcse{.0058}{.0011} & & \bestmcse{.0002}{.0002} \\
Tail shape, $\alpha_0=0.1$ & \mcse{.0042}{.0009} & \bestmcse{.0058}{.0011} & \mcse{.0046}{.0010} & \bestmcse{.0054}{.0010} & \mcse{.0048}{.0010} & \mcse{.0058}{.0011} & & \bestmcse{.0002}{.0002} \\
Bayes, shared $\Delta$ $+$ tail shape & \mcse{.0042}{.0009} & \bestmcse{.0058}{.0011} & \mcse{.0046}{.0010} & \bestmcse{.0054}{.0010} & \bestmcse{.0046}{.0010} & \mcse{.0058}{.0011} & & \bestmcse{.0002}{.0002} \\
Bayes, shared $\Delta$ $+$ union tail & \mcse{.0044}{.0009} & \bestmcse{.0058}{.0011} & \mcse{.0052}{.0010} & \mcse{.0060}{.0011} & \mcse{.0060}{.0011} & \bestmcse{.0056}{.0011} & & \mcse{.0014}{.0005} \\
\addlinespace[1pt]
$h_{\mathrm{ars}}$ & \mcse{.0184}{.0019} & \mcse{.0196}{.0020} & \mcse{.0190}{.0019} & \mcse{.0200}{.0020} & \mcse{.0224}{.0021} & \mcse{.0186}{.0019} & & \mcse{.0178}{.0016} \\
$h_{\log}$ & \mcse{.0966}{.0042} & \mcse{.1022}{.0043} & \mcse{.1012}{.0043} & \mcse{.0966}{.0042} & \mcse{.0934}{.0041} & \mcse{.0986}{.0042} & & \mcse{.0968}{.0030} \\
$h_{\mathrm{ind},1/e}$ & \mcse{.1316}{.0048} & \mcse{.1322}{.0048} & \mcse{.1348}{.0048} & \mcse{.1282}{.0047} & \mcse{.1204}{.0046} & \mcse{.1282}{.0047} & & \mcse{.1304}{.0036} \\
$h^\star_{\mathrm{gum},.1}$ & \mcse{.0506}{.0031} & \mcse{.0540}{.0032} & \mcse{.0524}{.0032} & \mcse{.0526}{.0032} & \mcse{.0496}{.0031} & \mcse{.0544}{.0032} & & \mcse{.0488}{.0021} \\
$h^\star_{\mathrm{gum},.01}$ & \mcse{.0206}{.0020} & \mcse{.0202}{.0020} & \mcse{.0170}{.0018} & \mcse{.0184}{.0019} & \mcse{.0162}{.0018} & \mcse{.0176}{.0019} & & \mcse{.0166}{.0016} \\
$h^\star_{\mathrm{gum},.005}$ & \mcse{.0146}{.0017} & \mcse{.0148}{.0017} & \mcse{.0128}{.0016} & \mcse{.0142}{.0017} & \mcse{.0120}{.0015} & \mcse{.0122}{.0016} & & \mcse{.0106}{.0011} \\
\addlinespace[1pt]
\multicolumn{9}{@{}l}{\textit{Horizon $n=700$}}\\
Tail shape, $\alpha=\infty$ (Table~\ref{tab:shared} rule) & \bestmcse{.0008}{.0004} & \mcse{.0026}{.0007} & \mcse{.0016}{.0006} & \mcse{.0020}{.0006} & \mcse{.0022}{.0007} & \bestmcse{.0018}{.0006} & & \mcse{.0004}{.0002} \\
Tail shape, $\alpha_0=1000$ & \bestmcse{.0008}{.0004} & \mcse{.0026}{.0007} & \mcse{.0016}{.0006} & \mcse{.0020}{.0006} & \mcse{.0022}{.0007} & \bestmcse{.0018}{.0006} & & \mcse{.0004}{.0002} \\
Tail shape, $\alpha_0=100$ & \bestmcse{.0008}{.0004} & \mcse{.0026}{.0007} & \mcse{.0016}{.0006} & \mcse{.0020}{.0006} & \mcse{.0022}{.0007} & \bestmcse{.0018}{.0006} & & \mcse{.0004}{.0002} \\
Tail shape, $\alpha_0=10$ & \bestmcse{.0008}{.0004} & \mcse{.0026}{.0007} & \bestmcse{.0014}{.0005} & \bestmcse{.0018}{.0006} & \mcse{.0020}{.0006} & \bestmcse{.0018}{.0006} & & \mcse{.0002}{.0002} \\
Tail shape, $\alpha_0=1$ & \bestmcse{.0008}{.0004} & \bestmcse{.0024}{.0007} & \bestmcse{.0014}{.0005} & \bestmcse{.0018}{.0006} & \bestmcse{.0018}{.0006} & \bestmcse{.0018}{.0006} & & \bestmcse{.0000}{.0002} \\
Tail shape, $\alpha_0=0.1$ & \bestmcse{.0008}{.0004} & \bestmcse{.0024}{.0007} & \bestmcse{.0014}{.0005} & \bestmcse{.0018}{.0006} & \bestmcse{.0018}{.0006} & \bestmcse{.0018}{.0006} & & \bestmcse{.0000}{.0002} \\
Bayes, shared $\Delta$ $+$ tail shape & \bestmcse{.0008}{.0004} & \bestmcse{.0024}{.0007} & \bestmcse{.0014}{.0005} & \bestmcse{.0018}{.0006} & \bestmcse{.0018}{.0006} & \bestmcse{.0018}{.0006} & & \bestmcse{.0000}{.0002} \\
Bayes, shared $\Delta$ $+$ union tail & \bestmcse{.0008}{.0004} & \bestmcse{.0024}{.0007} & \mcse{.0016}{.0006} & \bestmcse{.0018}{.0006} & \mcse{.0020}{.0006} & \bestmcse{.0018}{.0006} & & \mcse{.0002}{.0002} \\
\addlinespace[1pt]
$h_{\mathrm{ars}}$ & \mcse{.0108}{.0015} & \mcse{.0120}{.0015} & \mcse{.0118}{.0015} & \mcse{.0102}{.0014} & \mcse{.0122}{.0016} & \mcse{.0118}{.0015} & & \mcse{.0104}{.0010} \\
$h_{\log}$ & \mcse{.0618}{.0034} & \mcse{.0666}{.0035} & \mcse{.0672}{.0035} & \mcse{.0640}{.0035} & \mcse{.0594}{.0033} & \mcse{.0620}{.0034} & & \mcse{.0658}{.0026} \\
$h_{\mathrm{ind},1/e}$ & \mcse{.0777}{.0038} & \mcse{.0859}{.0040} & \mcse{.0850}{.0039} & \mcse{.0854}{.0040} & \mcse{.0795}{.0038} & \mcse{.0844}{.0039} & & \mcse{.0836}{.0027} \\
$h^\star_{\mathrm{gum},.1}$ & \mcse{.0322}{.0025} & \mcse{.0324}{.0025} & \mcse{.0320}{.0025} & \mcse{.0336}{.0025} & \mcse{.0322}{.0025} & \mcse{.0342}{.0026} & & \mcse{.0324}{.0017} \\
$h^\star_{\mathrm{gum},.01}$ & \mcse{.0110}{.0015} & \mcse{.0122}{.0016} & \mcse{.0102}{.0014} & \mcse{.0090}{.0013} & \mcse{.0096}{.0014} & \mcse{.0106}{.0014} & & \mcse{.0102}{.0011} \\
$h^\star_{\mathrm{gum},.005}$ & \mcse{.0066}{.0011} & \mcse{.0098}{.0014} & \mcse{.0074}{.0012} & \mcse{.0060}{.0011} & \mcse{.0074}{.0012} & \mcse{.0088}{.0013} & & \mcse{.0074}{.0009} \\
\bottomrule
\end{tabular}
\end{table}

The near agreement between $h^{\mathrm{sp}}_{.01}$ and the shared spike Bayes rule under the equal-tail generator does not persist under T5.  Under this sparse law, their $n=100$ errors are $.0434$ and $.0210$ and their empirical maximum regrets are $.0270$ and $.0046$; the Dirichlet tail-shape mixture has error $.0164$ and maximum regret $.0014$.  The tail exponent $1/(\Delta q)-1$ depends directly on the tail coordinate, so the diagnostic is sensitive to the generating tail law.

The $\alpha_0=1000$ row is identical to the $\alpha=\infty$ row in all 18 displayed cells; the $\alpha_0=100$ row differs in two cells, each by one document.  Because the experiment does not remove the concentrated atoms from the mixture prior, it does not identify their individual contribution.  The largest observed separation from the equal-tail rule occurs at the sparse end of the grid: under T5 at $n=100$, the layer at $\alpha_0=0.1$ has error $.0166$ and the mixture $.0164$, against $.0210$ for $\alpha=\infty$.

Among the fourteen displayed rules, the $\alpha$ mixture attains the lowest observed maximum regret at all three horizons, tying the $\alpha_0=1$ and $\alpha_0=0.1$ point rules at each.  T3 and T4 are the two in-family laws that are not atoms of the prior.  At T4, the mixture error is .0182, compared with .0186 for the lowest-error displayed point rule; at T3, the $\alpha_0=1$ point rule has error .0178, compared with .0180 for the mixture.  At T6, which is outside the family, the mixture has three more misses than the equal-tail rule.  These one-to-three-document differences are not statistically resolved.  The maximum-regret statistic selects extrema over six calibrated cells, and no simultaneous interval is supplied.

Every rule has its own null cutoff, so calibrated errors are compared directly.  All five finite atoms of the prior and the equal-tail limit are included as point-rule comparators.  The comparison remains restricted to this grid and does not optimize over $\alpha_0$.

At $n=100$, the best-performing published least-favorable tuning in each tail-law cell has error between $.0260$ and $.0332$, compared with $.0160$ to $.0210$ for the equal-tail layer.  Every published least-favorable tuning has higher observed Type~II error than every layer rule in each tested cell.  All tested laws distribute residual mass over multiple coordinates, whereas the least-favorable profile concentrates it on one.  None of the tested laws approaches the one-coordinate least-favorable profile.

For T6, the size-biased tail-coordinate variance is $2/9$, compared with $4/9$ under $\operatorname{Dirichlet}(3)$.  At $n=700$, the T5 error range across layer rules is $.0004$, or two documents, and every displayed layer rule has T6 error $.0018$.

The tail-shape sweep has two scope limitations.  By $n=300$ the mixture and the two sparse atoms $\alpha_0\in\{0.1,1\}$ are within $.0002$ of the empirical minimum and every displayed layer rule within $.0020$, and the equal-tail rule's maximum regret decreases from $.0046$ to $.0020$ to $.0004$ over the three horizons.  The sweep is also Gumbel-only and fixed-horizon.  The companion tail-sweep summaries include full tables, per-regime tail moments, the prespecified-prior fingerprint, and normalization diagnostics; see Section~\ref{sec:computational-materials}.

\FloatBarrier
\subsubsection{Tokenwise-deficit equal-tail experiment}
\label{sec:tokenwise-supplement}

Figure~\ref{fig:tokenwise} follows the written specification $\Delta_t\stackrel{\mathrm{iid}}{\sim}\Unif(0.001,0.5)$ of \citet{li2025framework}.  At $n=100,300,700$, observed Type~II error is lower under the tokenwise-deficit design for every rule common to the two experiments.  For Gumbel, the tokenwise Bayes mixture, its Dirichlet-layer counterpart, its union-tail counterpart, and several reference scores have no observed misses by $n=100$ among 5,000 alternatives; the exact one-sided 95\% upper bound is $0.000599$.  The three tokenwise mixtures reject the same documents at all three recorded horizons.  At $n=100$, $h_{\log}$ has 4 misses and the indicator has Type~II error $0.0198$, which is an expectation over the randomized boundary rather than an integer miss count.  The comparison therefore depends on whether the generating deficit is shared or redrawn tokenwise.

\begin{figure}[H]
\centering
\includegraphics[width=\textwidth]{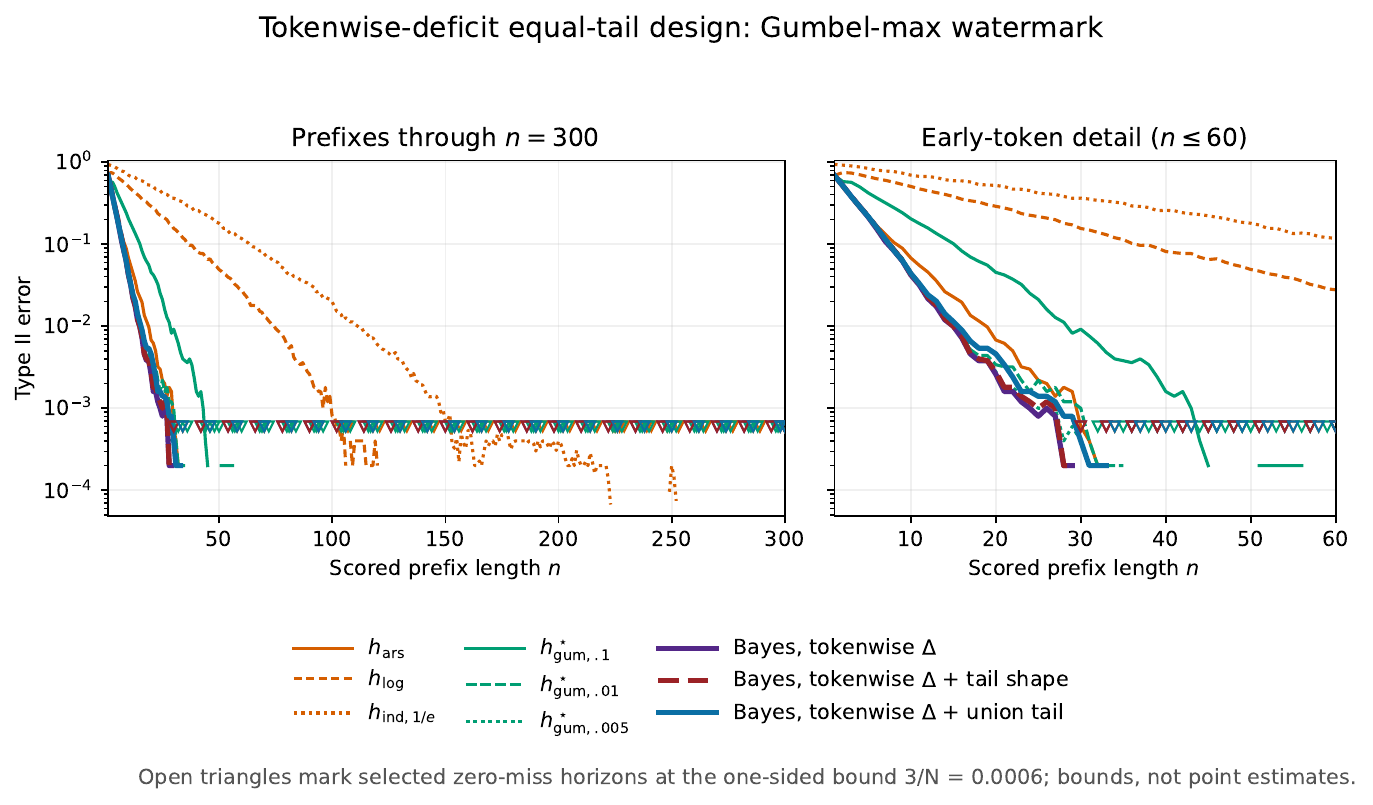}
\caption{Fixed-horizon Type~II error under the tokenwise-deficit equal-tail specification of \citet{li2025framework}: $\Delta_t$ is redrawn independently at every token.  Each rule uses its independently calibrated 5\% null cutoff.  The three Bayesian curves use the tokenwise hierarchy, with equal-tail, Dirichlet tail-shape, and union-tail ($w=.5$) priors; method labels and family colours follow Figure~\ref{fig:shared}.  The left panel shows the first 300 tokens; the right enlarges the first 60.  Curves retain all positive estimates, including rates below $3/N$; open triangles mark selected zero-miss horizons at the one-sided $3/N=0.0006$ upper bound, not point estimates, and are not connected to the curves.}
\label{fig:tokenwise}
\end{figure}

\FloatBarrier
\subsection{Additional temperature-matched analyses}
\label{sec:matched-details}

\subsubsection{Generation and calibration details}
\label{sec:matched-design-details}

A temperature-matched design regenerates both arms at a common temperature.  The initial prompts are the archived benchmark prompt tensors of Li et al., and the key, sampler, and pivot functions are imported unchanged from \citet{li2024watermarkframework}; the unwatermarked sampling law is changed to use the same temperature as the watermarked generator.  Temperature is varied over $\{.1,.2,.3,.4,.5,.6,.7,1\}$ for both models, with 200 scored tokens at temperature $.3$ and below and 100 above, where detection saturates at shorter horizons.  The informative window $[.2,.5]$ uses 2500 documents per arm and the remaining temperatures 500.  Because their benchmark archive supplies only 500 prompts, the prompt set is extended by streaming the same C4 subset under the same filter used by \citet{li2024watermarkframework}: records are read in file order from a slice pinned by digest, tokenized with truncation at 2028 tokens, dropped unless they carry both a prompt and a full scored continuation, and reduced to the 50-token window immediately preceding that continuation.  All 500 reconstructed prompts are identical token for token to their archived tensor, which confirms the construction, and the 2000 prompts at zero-based indices 500--2499 extend the same population.  Null pivots are obtained by replaying the reference implementation's key and hash functions over the regenerated unwatermarked tokens, as described in Supplementary Section~\ref{sec:released-calibration}.

Split-half calibration produces a variable realized size.  With 500 documents at temperature $.1$, the induced Type~II standard deviation across resplits ranges from $.014$ to $.040$ over the rules, exceeding the differences among the leading rules; rankings from a single split are therefore unstable.  This quantity is recorded per cell as \texttt{type2\_split\_sd}.
The main analysis therefore estimates each prefix-specific cutoff from the
full unwatermarked arm and assesses Type~I error by five-fold cross-fitting,
as described in Section~\ref{sec:temperature-matched}.

\subsubsection{Bootstrap reporting and sample-size comparisons}
\label{sec:matched-inference-details}

The reported $p$-values use normal tails computed from the joint prompt-cluster
bootstrap standard errors; empirical bootstrap-tail probabilities are stored
separately.  The floor of $1/2000$ imposed by $2{,}000$ replicates applies to those percentile probabilities, not to the reported normal-tail values: among the thirty-five Holm-significant contrasts, thirty equal this floor, three equal $.001$, and two equal $.002$.

The stability of the observed ordering depends on the sample size.  With 500 documents, the size of the original benchmark prompt set, the estimated union-tail AUC exceeds that of the shared-$\Delta$ rule in all eight cells; the estimated $h_{\mathrm{ars}}$ AUC exceeds that of $h^\star_{\mathrm{gum},.1}$ in only five of eight cells.  These 500-document comparisons are descriptive: the stored cluster bootstrap covers only the 2500-document analysis, so no adjusted $p$-value is reported and inference is limited to the signs of the eight observed contrasts.  The 500-document sample therefore does not establish an ordering of the two leading reference scores.  The tail-width model improves on the equal-tail model that it generalizes at both sample sizes, but does not generally exceed the best reference-score AUC at either size.

\subsubsection{Prefix curves and repetition diagnostics}
\label{sec:temperature-prefix-details}

The repetition-only statistic that attains Type~II error at worst $.012$ on the archived watermarked/unwatermarked benchmark pair has errors $.948$ on OPT-1.3B and $.932$ on Sheared-LLaMA-2.7B in the matched design at temperature $.1$, close to the null-discrimination benchmark of $.95$.  At temperature one, its error is $.69$ on both models.  Because keyed sampling is deterministic conditional on the context, the watermarked continuations retain greater repetition than the unwatermarked continuations; this residual is a property of the scheme rather than the temperature mismatch.

Prefix power at low temperature is nonmonotone in the token count.  For a heuristic characterization, suppose distinct repeated-pivot groups contribute independent scores with equal null variance, multiplicities $k_i$, and a common per-position mean shift.  A token sum over $n$ positions then has null standard deviation proportional to $(\sum_i k_i^2)^{1/2}$ and mean shift proportional to $n$, giving standardized separation proportional to $n/(\sum_i k_i^2)^{1/2}$.  This quantity is less than $\sqrt n$ when a pivot or score repeats; a repeated context alone need not repeat the pivot, as Supplementary Section~\ref{sec:released-calibration} explains.  The calculation quantifies a loss relative to independent-score scaling, but does not itself imply nonmonotone prefix power and is not a variance formula for the shared-mixture Bayes factor.  Dividing both a token-sum statistic and its calibrated cutoff by $n$ changes no decision.  The distinct fraction at 100 tokens decreases from $.79$ and $.81$ at temperature $.5$ to $.67$ and $.70$ at temperature $.2$.  Over a common 100-token horizon on OPT-1.3B, the union-tail Type~II error increases from its within-curve minimum by $.016$ and $.014$ at temperatures $.2$ and $.3$, compared with $.001$ and $.000$ at $.4$ and $.5$; $h_{\mathrm{ars}}$ has the same pattern.  The unequal maximum horizons in Figures~\ref{fig:temperature-matched} and~\ref{fig:temperature-matched-sheared} therefore affect visual comparisons of the temperature-specific endpoints.  They account for only a small part of the overall temperature association: at a common 100-token horizon, the union-tail Type~II error decreases from $.6176$ at temperature $.2$ to $.0088$ at temperature $.5$ on OPT-1.3B, and from $.6268$ to $.0168$ on Sheared-LLaMA-2.7B.

\begin{figure}[htbp]
\centering
\includegraphics[width=\textwidth]{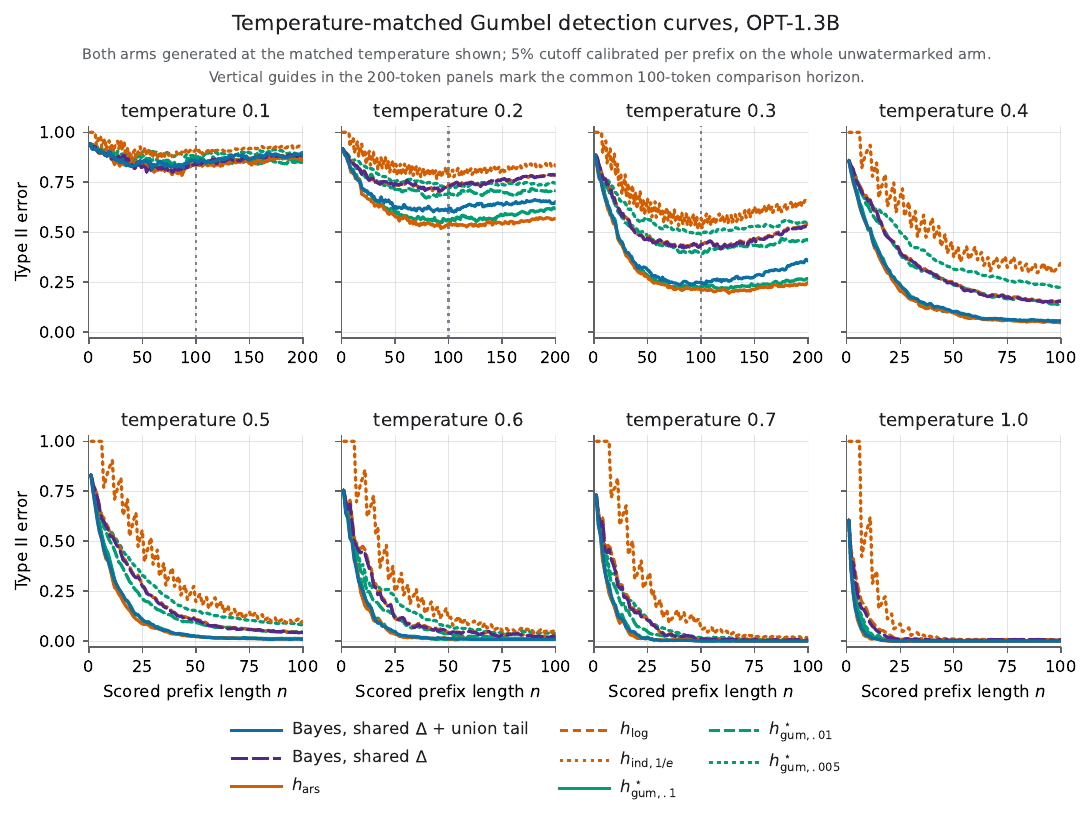}
\caption{Temperature-matched Type~II error for OPT-1.3B under the empirical-null calibration of Section~\ref{sec:temperature-matched}; Figure~\ref{fig:temperature-matched-sheared} uses the same display for Sheared-LLaMA-2.7B.  At each prefix, method-specific 5\% cutoffs are calibrated on the full unwatermarked arm.  Each arm has 2500 documents at temperatures $.2$--$.5$ and 500 elsewhere.  Eight of the twenty scored rules are drawn: shared-$\Delta$ Bayes with equal tails or the union-tail prior ($w=.5$), and six reference scores.  Method labels and colours follow Figure~\ref{fig:shared}, with dashes distinguishing members within each family.  The first three panels extend to 200 tokens; vertical guides mark the common 100-token comparison horizon.  Other panels stop at 100.  All scored rules are retained in the companion prefix-curve table indexed in Section~\ref{sec:computational-materials}.}
\label{fig:temperature-matched}
\end{figure}

\begin{figure}[htbp]
\centering
\includegraphics[width=\textwidth]{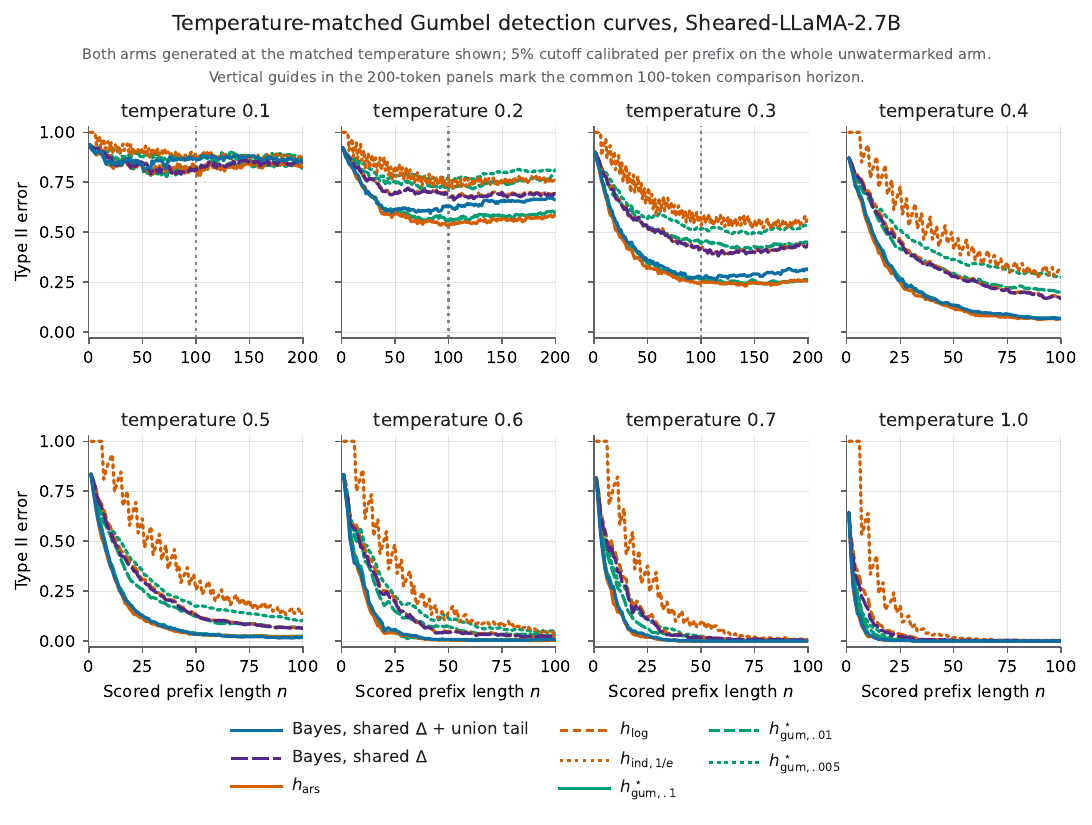}
\caption{Temperature-matched Type~II error for Sheared-LLaMA-2.7B, using the same axes, method labels, sample sizes, and empirical-null calibration protocol as Figure~\ref{fig:temperature-matched}.  Nonmonotonicity in the token count is also present for OPT-1.3B, with different magnitudes.}
\label{fig:temperature-matched-sheared}
\end{figure}

\FloatBarrier

\subsection{Hierarchical extensions}
\label{sec:hierarchical-extensions}

This section gives the model specifications, hyperpriors, persistence
experiments, and temperature-matched comparisons summarized in
Section~\ref{sec:hierarchical-summary}.  The two extensions modify the
conditional distributions of token-specific deficits or tail widths while
retaining the likelihood construction of Section~\ref{sec:bayes-factors}.

\subsubsection{Hierarchical model for token-specific deficits}
\label{sec:deficit-pooling}

The models in \eqref{eq:hier-shared} and \eqref{eq:hier-tokenwise} impose two
distinct assumptions on within-document variation: a common deficit $\Delta$ or
independent token-specific deficits.  We extend these models by specifying the
conditional distribution of $\Delta_t$ through a document-level hyperparameter
$\psi=(\mu,\kappa)$:
\begin{equation}
  X_t\mid\psi\ \stackrel{\mathrm{iid}}{\sim}\ \operatorname{Beta}(\kappa\mu,\kappa(1-\mu)),\qquad
  \Delta_t=\Delta_{\min}+(\Delta_{\max}-\Delta_{\min})X_t,\qquad \psi\sim\pi_\psi,
  \label{eq:deficit-pooling}
\end{equation}
Conditional on $\psi$, the deficits are independent and identically distributed;
integrating over $\psi$ induces marginal dependence and partial pooling across
token positions.  The document marginal likelihood therefore integrates the
product of the conditional token likelihoods, as represented in
Figure~\ref{fig:pooling-model}(a).  The hyperprior includes both earlier models as
special cases.  At the atom $(\mu,\kappa)=(1/2,2)$, the conditional distribution is
$\operatorname{Beta}(1,1)$, giving independent uniform deficits and hence the
tokenwise model.  We also include a degenerate component, denoted
$\kappa=\infty$, with $X_t=\mu$ at every position; a uniform prior on $\mu$
recovers the shared-deficit model analytically.
The hyperprior $\pi_\psi$ assigns equal mass to
$\kappa\in\{2,8,32,128,\infty\}$.  At $\kappa=2$, $\mu=1/2$; at each other
concentration, $\mu\mid\kappa\sim\Unif(0,1)$.
The reported experiments use 12 Gauss--Legendre nodes for $\mu$ and 64
Gauss--Jacobi nodes for each nondegenerate conditional Beta distribution, giving
49 hyperparameter components; the simulation seed is $20260902$.
Separate endpoint checks compare cumulative log Bayes factors on 120 independent
30-token null paths.  The maximum discrepancies are below $5\times10^{-6}$ for
the tokenwise endpoint and $10^{-7}$ for the shared endpoint; the latter check
uses matched 96-node quadrature for $\mu$, not the 12-node production grid.
These are numerical checks on the tested paths, not uniform error bounds.

The evaluated hierarchical deficit model
uses the equal-tail submodel $\alpha=\infty$ and $J=K$.  The $S,\alpha,J$ nodes in
Figure~\ref{fig:pooling-model}(a) describe the general model family; in this
experiment, $\alpha$ and $J$ are fixed, $S$ is redundant, and $\bm q_t$ is deterministic.

\begin{figure}[t]
\centering
\begin{minipage}[t]{0.48\textwidth}
\centering
\textbf{(a) Hierarchical deficit}\par
{\footnotesize Evaluated restriction: $\alpha=\infty,\ J=K$}\par\smallskip
\begin{tikzpicture}[font=\small]
  \node[priornode] (pis) at (-1.75,4.6) {$\pi_S$};
  \node[priornode] (pia) at (0,4.6) {$\pi_\alpha$};
  \node[priornode] (pij) at (1.4,4.6) {$\pi_J$};
  \node[priornode] (pipsi) at (2.8,4.6) {$\pi_\psi$};
  \node[latentnode] (b) at (-1.75,3.1) {$S$};
  \node[latentnode] (alpha) at (0,3.1) {$\alpha$};
  \node[latentnode] (j) at (1.4,3.1) {$J$};
  \node[latentnode] (psi) at (2.8,3.1) {$\psi$};
  \node[latentnode] (q) at (0,1.5) {$\bm q_t$};
  \node[latentnode] (delta) at (2.8,1.5) {$\Delta_t$};
  \node[detnode] (p) at (1.4,0) {$\bP_t$};
  \node[obsnode] (y) at (2.8,0) {$Y_t$};
  \draw[gmarrow] (pis) -- (b);
  \draw[gmarrow] (pia) -- (alpha);
  \draw[gmarrow] (pij) -- (j);
  \draw[gmarrow] (pipsi) -- (psi);
  \draw[gmarrow] (b) -- (q);
  \draw[gmarrow] (alpha) -- (q);
  \draw[gmarrow] (j) -- (q);
  \draw[gmarrow] (psi) -- (delta);
  \draw[gmarrow] (delta) -- (p);
  \draw[gmarrow] (q) -- (p);
  \draw[gmarrow] (p) -- (y);
  \begin{scope}[on background layer]
    \node[platenode,inner xsep=5mm,inner ysep=4mm,
      fit=(q)(delta)(p)(y)] (platePD) {};
  \end{scope}
  \node[font=\scriptsize,anchor=north east,yshift=-1pt]
    at (platePD.south east) {$t=1{:}n$};
\end{tikzpicture}
\end{minipage}\hfill
\begin{minipage}[t]{0.48\textwidth}
\centering
\textbf{(b) Hierarchical tail width}\par
{\footnotesize Evaluated restriction: $S=\mathrm{width},\ \alpha=\infty$}\par\smallskip
\begin{tikzpicture}[font=\small]
  \node[priornode] (pis) at (-1.75,4.6) {$\pi_S$};
  \node[priornode] (pia) at (0,4.6) {$\pi_\alpha$};
  \node[priornode] (piphi) at (1.4,4.6) {$\pi_\phi$};
  \node[priornode] (pid) at (2.8,4.6) {$\pi_\Delta$};
  \node[latentnode] (b) at (-1.75,3.1) {$S$};
  \node[latentnode] (alpha) at (0,3.1) {$\alpha$};
  \node[latentnode] (phi) at (1.4,3.1) {$\phi$};
  \node[latentnode] (delta) at (2.8,3.1) {$\Delta$};
  \node[latentnode] (q) at (0,1.5) {$\bm q_t$};
  \node[latentnode] (j) at (1.4,1.5) {$J_t$};
  \node[detnode] (p) at (1.4,0) {$\bP_t$};
  \node[obsnode] (y) at (2.8,0) {$Y_t$};
  \draw[gmarrow] (pis) -- (b);
  \draw[gmarrow] (pia) -- (alpha);
  \draw[gmarrow] (piphi) -- (phi);
  \draw[gmarrow] (pid) -- (delta);
  \draw[gmarrow] (b) -- (q);
  \draw[gmarrow] (alpha) -- (q);
  \draw[gmarrow] (phi) -- (j);
  \draw[gmarrow] (j) -- (q);
  \draw[gmarrow] (q) -- (p);
  \draw[gmarrow] (delta) -- (p);
  \draw[gmarrow] (p) -- (y);
  \begin{scope}[on background layer]
    \node[platenode,inner xsep=5mm,inner ysep=4mm,
      fit=(q)(j)(p)(y)] (platePW) {};
  \end{scope}
  \node[font=\scriptsize,anchor=north east,yshift=-1pt]
    at (platePW.south east) {$t=1{:}n$};
\end{tikzpicture}
\end{minipage}
\caption{General hierarchical models for token-specific deficits and tail widths, extending Figure~\ref{fig:hierarchical-model}.  The diagrams retain the general tail hierarchy; panel headings give the restrictions used in the evaluated submodels.  Under these restrictions, $\bm q_t$ is deterministic in both panels and $S$ is redundant in (a).  (a) The document-level hyperparameter $\psi=(\mu,\kappa)$ specifies the conditional distribution of $\Delta_t$.  The atom $(\mu,\kappa)=(1/2,2)$ gives independent uniform deficits; $\kappa=\infty$ gives a common deficit.  (b) The document-level hyperparameter $\phi=(c,\lambda)$ specifies the conditional distribution of $J_t$ in the width branch, with a common deficit $\Delta$.  At $\lambda=\infty$, $J_t=J_c$ at every position; at $\lambda=0$, the widths are independent and uniform on the ladder.  Plates enclose token-specific quantities.  Sections~\ref{sec:deficit-pooling} and~\ref{sec:width-pooling} specify the evaluated submodels.}
\label{fig:pooling-model}
\end{figure}
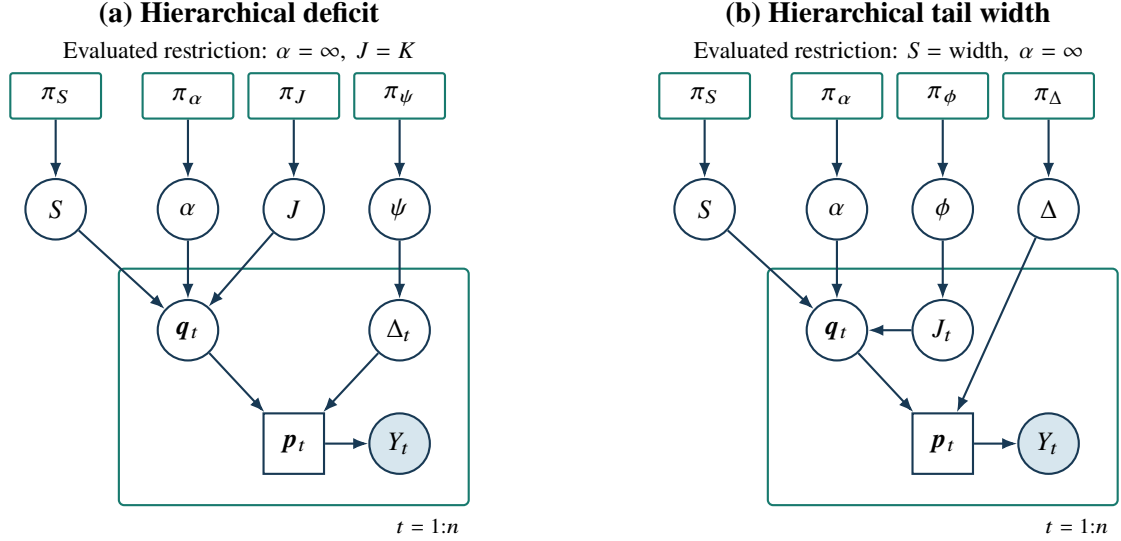

Varying $\kappa$ in a generative model with $\mu\sim\Unif(0,1)$ changes both the
marginal distribution and within-document dependence, since
$\operatorname{Var}(X_t)=\tfrac1{12}+\tfrac1{6(\kappa+1)}$.  A decrease in $\kappa$
increases marginal dispersion while reducing persistence, so this design does
not identify their separate effects on power.  Table~\ref{tab:deficit-pooling}
instead uses a fixed-marginal
Gaussian copula: $Z_t=\rho W+\sqrt{1-\rho^2}\,\varepsilon_t$, with $W$ drawn once per
document and $W,\varepsilon_1,\ldots,\varepsilon_n$ independent standard normals.
Taking $\Delta_t$ as the $\Unif(.001,.5)$ quantile of $\Phi(Z_t)$ gives exactly the
same marginal law at every persistence level.  The latent pair has correlation
$\rho^2$, and the transformed uniform variables have correlation
$(6/\pi)\arcsin(\rho^2/2)$, the reported intraclass correlation.
P1 is the document-shared design and P6 the tokenwise one; the horizons stop at
$n=60$ because power is near one by $n=100$.

Across the six persistence levels and three horizons, the hierarchical and
shared-deficit models differ in estimated Type~II error by at most $.0012$.
Exact McNemar tests use the paired decisions on the same documents.  After Holm
adjustment across the thirty-six cells of the two contrasts reported here, none
of the hierarchical-versus-shared differences is significant.  This does not
establish equivalence: no equivalence margin was prespecified, and the sample
sizes do not exclude small differences.  Relative to the tokenwise model, the
hierarchical model has a significantly lower Type~II error in one cell, P2 at
$n=60$, with a difference of $.0030$ and Holm $p=.010$.  The tokenwise model also
has higher estimated Type~II error than the shared model by up to $.0036$ under
persistent deficits.  Thus both the shared and hierarchical models improve on
the independent tokenwise specification in some persistent regimes, without a
detected advantage of the additional hyperparameter mixture over the shared model.

The marginal binomial standard errors in Table~\ref{tab:deficit-pooling} do not
describe paired differences.  The companion hierarchical-deficit summaries contain discordant counts and exact McNemar tests, with per-document decisions supplied separately; Section~\ref{sec:computational-materials} identifies both records.  The single 36-cell Holm family
includes every cell of both contrasts, not a separate family for each comparator.
It is comprehensive but post hoc: the Type~II estimates preceded the paired tests,
although no cell was omitted after the estimates were seen.

In both synthetic hierarchical-model experiments, maximum regret is computed at each horizon
against the best scored rule in each regime and then maximized over regimes.
Its standard error uses 2,000 paired document-bootstrap replicates, recomputing
both the best-rule envelope and maximum on each draw.  A maximum of differences
is not a binomial proportion.

An earlier experiment used the Beta generator in \eqref{eq:deficit-pooling},
with $\mu\sim\Unif(0,1)$, intermediate concentrations
$\kappa\in\{.5,2,8,32\}$, and horizons $n\in\{100,300,700\}$.
Its generator and numerical summary are included as a separately labeled historical record in the companion repository indexed in Section~\ref{sec:computational-materials}.  The hierarchical model had lower estimated Type~II error than
both endpoint models in all twelve concentration--horizon cells.  Because this
design changes the marginal deficit distribution together with dependence,
those differences cannot be attributed to persistence alone.  Its horizons
also differ from the fixed-marginal experiment reported here.

\begin{table}[htbp]
\centering
\scriptsize
\renewcommand{\arraystretch}{0.9}
\setlength{\tabcolsep}{3pt}%
\caption{Hierarchical deficit model under the fixed-marginal Gaussian-copula design of Section~\ref{sec:deficit-pooling}; P1 is document-shared and P6 tokenwise.  Entries are Type~II errors at nominal 5\% size with binomial Monte Carlo standard errors; bold marks observed column minima, ties included.  Maximum regret is the largest excess over the best scored rule across P1--P6, with a paired 2,000-replicate document-bootstrap standard error.  The $h^\star_{\mathrm{gum},\Delta_0}$ rows use the three published tunings of \citet{li2025framework}.}
\label{tab:deficit-pooling}
\begin{tabular}{@{}lrrrrrrcr@{}}
\toprule
& \multicolumn{6}{c}{Type~II error by deficit-persistence regime} & & Max \\
\cmidrule(lr){2-7}\cmidrule(lr){9-9}
Rule & P1 & P2 & P3 & P4 & P5 & P6 & & regret \\
Intraclass correlation & 1.000 & .894 & .622 & .346 & .086 & .000 & & \\
\midrule
\multicolumn{9}{@{}l}{\textit{Horizon $n=15$}}\\[1pt]
$h_{\mathrm{ars}}$ & \mcse{.1408}{.0049} & \mcse{.1266}{.0047} & \mcse{.1018}{.0043} & \mcse{.0632}{.0034} & \mcse{.0316}{.0025} & \mcse{.0258}{.0022} & & \mcse{.0322}{.0025} \\
$h_{\log}$ & \mcse{.4140}{.0070} & \mcse{.4022}{.0069} & \mcse{.4048}{.0069} & \mcse{.3954}{.0069} & \mcse{.3802}{.0069} & \mcse{.3848}{.0069} & & \mcse{.3726}{.0061} \\
$h_{\mathrm{ind},1/e}$ & \mcse{.5612}{.0070} & \mcse{.5638}{.0070} & \mcse{.5888}{.0070} & \mcse{.5976}{.0069} & \mcse{.6066}{.0069} & \mcse{.6146}{.0069} & & \mcse{.6024}{.0064} \\
$h^\star_{\mathrm{gum},.1}$ & \mcse{.2636}{.0062} & \mcse{.2444}{.0061} & \mcse{.2172}{.0058} & \mcse{.1676}{.0053} & \mcse{.1214}{.0046} & \mcse{.1090}{.0044} & & \mcse{.1478}{.0040} \\
$h^\star_{\mathrm{gum},.01}$ & \mcse{.1190}{.0046} & \mcse{.1000}{.0042} & \mcse{.0706}{.0036} & \mcse{.0392}{.0027} & \mcse{.0176}{.0019} & \mcse{.0142}{.0017} & & \mcse{.0032}{.0010} \\
$h^\star_{\mathrm{gum},.005}$ & \mcse{.1178}{.0046} & \mcse{.1000}{.0042} & \mcse{.0704}{.0036} & \mcse{.0394}{.0028} & \mcse{.0172}{.0018} & \mcse{.0142}{.0017} & & \mcse{.0020}{.0008} \\
Bayes, shared $\Delta$ & \mcse{.1164}{.0045} & \mcse{.0994}{.0042} & \mcse{.0698}{.0036} & \mcse{.0378}{.0027} & \mcse{.0176}{.0019} & \mcse{.0134}{.0016} & & \mcse{.0014}{.0006} \\
Bayes, tokenwise $\Delta$ & \mcse{.1166}{.0045} & \bestmcse{.0992}{.0042} & \bestmcse{.0696}{.0036} & \bestmcse{.0374}{.0027} & \bestmcse{.0162}{.0018} & \bestmcse{.0122}{.0016} & & \mcse{.0008}{.0006} \\
Bayes, hierarchical $\Delta$ & \bestmcse{.1158}{.0045} & \bestmcse{.0992}{.0042} & \mcse{.0702}{.0036} & \mcse{.0378}{.0027} & \mcse{.0164}{.0018} & \mcse{.0128}{.0016} & & \bestmcse{.0006}{.0006} \\
\addlinespace[4pt]
\multicolumn{9}{@{}l}{\textit{Horizon $n=30$}}\\[1pt]
$h_{\mathrm{ars}}$ & \mcse{.0924}{.0041} & \mcse{.0704}{.0036} & \mcse{.0378}{.0027} & \mcse{.0176}{.0019} & \mcse{.0016}{.0006} & \mcse{.0006}{.0003} & & \mcse{.0318}{.0025} \\
$h_{\log}$ & \mcse{.2924}{.0064} & \mcse{.2710}{.0063} & \mcse{.2596}{.0062} & \mcse{.2180}{.0058} & \mcse{.1598}{.0052} & \mcse{.1478}{.0050} & & \mcse{.2404}{.0054} \\
$h_{\mathrm{ind},1/e}$ & \mcse{.4368}{.0070} & \mcse{.4238}{.0070} & \mcse{.4288}{.0070} & \mcse{.4230}{.0070} & \mcse{.4052}{.0069} & \mcse{.4086}{.0070} & & \mcse{.4170}{.0056} \\
$h^\star_{\mathrm{gum},.1}$ & \mcse{.1778}{.0054} & \mcse{.1492}{.0050} & \mcse{.1148}{.0045} & \mcse{.0676}{.0036} & \mcse{.0190}{.0019} & \mcse{.0088}{.0013} & & \mcse{.1172}{.0045} \\
$h^\star_{\mathrm{gum},.01}$ & \mcse{.0988}{.0042} & \mcse{.0686}{.0036} & \mcse{.0392}{.0027} & \mcse{.0154}{.0017} & \mcse{.0012}{.0005} & \mcse{.0008}{.0004} & & \mcse{.0382}{.0030} \\
$h^\star_{\mathrm{gum},.005}$ & \mcse{.0628}{.0034} & \bestmcse{.0434}{.0029} & \mcse{.0204}{.0020} & \mcse{.0062}{.0011} & \bestmcse{.0008}{.0004} & \mcse{.0002}{.0002} & & \mcse{.0022}{.0009} \\
Bayes, shared $\Delta$ & \mcse{.0608}{.0034} & \mcse{.0442}{.0029} & \bestmcse{.0192}{.0019} & \mcse{.0072}{.0012} & \bestmcse{.0008}{.0004} & \mcse{.0002}{.0002} & & \mcse{.0012}{.0005} \\
Bayes, tokenwise $\Delta$ & \mcse{.0622}{.0034} & \mcse{.0450}{.0029} & \mcse{.0202}{.0020} & \bestmcse{.0060}{.0011} & \bestmcse{.0008}{.0004} & \bestmcse{.0000}{.0000} & & \mcse{.0016}{.0006} \\
Bayes, hierarchical $\Delta$ & \bestmcse{.0606}{.0034} & \mcse{.0444}{.0029} & \mcse{.0196}{.0020} & \mcse{.0064}{.0011} & \bestmcse{.0008}{.0004} & \mcse{.0002}{.0002} & & \bestmcse{.0010}{.0006} \\
\addlinespace[4pt]
\multicolumn{9}{@{}l}{\textit{Horizon $n=60$}}\\[1pt]
$h_{\mathrm{ars}}$ & \mcse{.0574}{.0033} & \mcse{.0404}{.0028} & \mcse{.0164}{.0018} & \mcse{.0020}{.0006} & \bestmcse{.0000}{.0000} & \bestmcse{.0000}{.0000} & & \mcse{.0278}{.0023} \\
$h_{\log}$ & \mcse{.2122}{.0058} & \mcse{.1918}{.0056} & \mcse{.1556}{.0051} & \mcse{.1030}{.0043} & \mcse{.0434}{.0029} & \mcse{.0232}{.0021} & & \mcse{.1826}{.0050} \\
$h_{\mathrm{ind},1/e}$ & \mcse{.3120}{.0066} & \mcse{.2860}{.0064} & \mcse{.2748}{.0063} & \mcse{.2298}{.0059} & \mcse{.1740}{.0054} & \mcse{.1532}{.0051} & & \mcse{.2824}{.0059} \\
$h^\star_{\mathrm{gum},.1}$ & \mcse{.1206}{.0046} & \mcse{.1012}{.0043} & \mcse{.0580}{.0033} & \mcse{.0158}{.0018} & \mcse{.0012}{.0005} & \mcse{.0002}{.0002} & & \mcse{.0910}{.0037} \\
$h^\star_{\mathrm{gum},.01}$ & \mcse{.0528}{.0032} & \mcse{.0364}{.0026} & \mcse{.0132}{.0016} & \mcse{.0012}{.0005} & \bestmcse{.0000}{.0000} & \bestmcse{.0000}{.0000} & & \mcse{.0232}{.0021} \\
$h^\star_{\mathrm{gum},.005}$ & \mcse{.0480}{.0030} & \mcse{.0340}{.0026} & \mcse{.0124}{.0016} & \mcse{.0010}{.0004} & \bestmcse{.0000}{.0000} & \bestmcse{.0000}{.0000} & & \mcse{.0184}{.0017} \\
Bayes, shared $\Delta$ & \bestmcse{.0296}{.0024} & \bestmcse{.0168}{.0018} & \bestmcse{.0054}{.0010} & \bestmcse{.0008}{.0004} & \bestmcse{.0000}{.0000} & \bestmcse{.0000}{.0000} & & \bestmcse{.0000}{.0002} \\
Bayes, tokenwise $\Delta$ & \mcse{.0328}{.0025} & \mcse{.0204}{.0020} & \mcse{.0064}{.0011} & \bestmcse{.0008}{.0004} & \bestmcse{.0000}{.0000} & \bestmcse{.0000}{.0000} & & \mcse{.0036}{.0008} \\
Bayes, hierarchical $\Delta$ & \mcse{.0304}{.0024} & \mcse{.0174}{.0018} & \bestmcse{.0054}{.0010} & \bestmcse{.0008}{.0004} & \bestmcse{.0000}{.0000} & \bestmcse{.0000}{.0000} & & \mcse{.0008}{.0004} \\
\bottomrule
\end{tabular}
\end{table}

On the temperature-matched data, the hierarchical deficit model has higher
estimated AUC than the shared-deficit model in all eight cells, with differences
from $.0004$ to $.0031$ and mean $.0015$ (Table~\ref{tab:pooled-real}).  Unlike the
synthetic persistence experiments, these comparisons evaluate deficits arising
from the language model's context-dependent next-token distributions.  The
differences are one to two orders of magnitude smaller than the union-tail
model's mean AUC improvement of $.0328$ over the same baseline on the same
documents.  The hierarchical-versus-shared contrast was not included in the
Holm family and is therefore reported descriptively, without a claim of
multiplicity-adjusted significance.

\begin{table}[htbp]
\centering
\scriptsize
\setlength{\tabcolsep}{2.6pt}%
\caption{Hierarchical versus shared-deficit models on the temperature-matched data: AUC at 100 tokens for 2500 documents per arm, with prompt-cluster bootstrap standard errors.  The difference row is hierarchical minus shared and is descriptive; this contrast is outside the Holm family of Section~\ref{sec:temperature-matched}.  Bold marks within-block maxima, ties included; reference scores form a separate block.}
\label{tab:pooled-real}
\begin{tabular}{@{}lrrrrrrrr@{}}
\toprule
& \multicolumn{4}{c}{OPT-1.3B} & \multicolumn{4}{c}{Sheared-LLaMA-2.7B} \\
\cmidrule(lr){2-5}\cmidrule(lr){6-9}
Rule & $T{=}.2$ & $.3$ & $.4$ & $.5$ & $T{=}.2$ & $.3$ & $.4$ & $.5$ \\
\midrule
Bayes, shared $\Delta$ & \mcse{.7828}{.0063} & \mcse{.8906}{.0045} & \mcse{.9594}{.0028} & \mcse{.9871}{.0014} & \mcse{.7814}{.0058} & \mcse{.8903}{.0046} & \mcse{.9553}{.0027} & \mcse{.9824}{.0017} \\
Bayes, hierarchical $\Delta$ & \bestmcse{.7837}{.0063} & \bestmcse{.8921}{.0045} & \bestmcse{.9599}{.0028} & \bestmcse{.9876}{.0014} & \bestmcse{.7845}{.0058} & \bestmcse{.8926}{.0046} & \bestmcse{.9579}{.0027} & \bestmcse{.9833}{.0017} \\
\addlinespace
Difference & $+$\mcse{.0009}{.0006} & $+$\mcse{.0015}{.0004} & $+$\mcse{.0005}{.0005} & $+$\mcse{.0004}{.0001} & $+$\mcse{.0031}{.0007} & $+$\mcse{.0022}{.0008} & $+$\mcse{.0025}{.0005} & $+$\mcse{.0009}{.0005} \\
\addlinespace
\multicolumn{9}{@{}l}{\textit{Reference scores}}\\
$h_{\mathrm{ars}}$ & \bestmcse{.8499}{.0051} & \bestmcse{.9465}{.0031} & \bestmcse{.9824}{.0020} & \bestmcse{.9963}{.0009} & \bestmcse{.8467}{.0051} & \bestmcse{.9373}{.0036} & \bestmcse{.9829}{.0016} & \bestmcse{.9931}{.0012} \\
$h_{\log}$ & \mcse{.7825}{.0063} & \mcse{.8904}{.0045} & \mcse{.9585}{.0028} & \mcse{.9871}{.0014} & \mcse{.7807}{.0059} & \mcse{.8900}{.0046} & \mcse{.9544}{.0028} & \mcse{.9821}{.0018} \\
$h_{\mathrm{ind},1/e}$ & \mcse{.7531}{.0066} & \mcse{.8614}{.0052} & \mcse{.9320}{.0037} & \mcse{.9783}{.0019} & \mcse{.7573}{.0063} & \mcse{.8616}{.0052} & \mcse{.9364}{.0032} & \mcse{.9691}{.0023} \\
$h^\star_{\mathrm{gum},.1}$ & \mcse{.8459}{.0053} & \mcse{.9440}{.0032} & \mcse{.9819}{.0020} & \mcse{.9961}{.0008} & \mcse{.8409}{.0052} & \mcse{.9331}{.0037} & \mcse{.9810}{.0017} & \mcse{.9922}{.0012} \\
$h^\star_{\mathrm{gum},.01}$ & \mcse{.7836}{.0058} & \mcse{.8859}{.0046} & \mcse{.9563}{.0029} & \mcse{.9873}{.0015} & \mcse{.7649}{.0061} & \mcse{.8744}{.0048} & \mcse{.9486}{.0029} & \mcse{.9825}{.0017} \\
$h^\star_{\mathrm{gum},.005}$ & \mcse{.7577}{.0061} & \mcse{.8576}{.0051} & \mcse{.9378}{.0033} & \mcse{.9785}{.0018} & \mcse{.7376}{.0062} & \mcse{.8470}{.0052} & \mcse{.9278}{.0035} & \mcse{.9734}{.0020} \\
\bottomrule
\end{tabular}
\end{table}

The AUC standard errors come from the joint prompt-cluster resampling of
Section~\ref{sec:temperature-matched}: a binomial standard error does not apply to
an AUC, and ordinary DeLong inference does not account for the prompt pairing.
The same resamples provide the paired-difference standard errors, which are
several times smaller than the marginal errors because both rules score the
same documents.

\FloatBarrier
\subsubsection{Hierarchical model for token-specific tail widths}
\label{sec:width-pooling}

We next specify a hierarchical model for token-specific tail widths.  The
width models of Sections~\ref{sec:tail-width} and~\ref{sec:hierarchy} use a fixed
uniform prior on the ladder, with either a
common width or independent widths at each token.  Here the conditional
distribution of $J_t$ depends on a document-level hyperparameter
$\phi=(c,\lambda)$:
\begin{equation}
  \Prb(J_t=J_k\mid\phi)\;\propto\;e^{-\lambda|k-c|},\qquad
  J_t\mid\phi\ \text{i.i.d.},\qquad \phi\sim\pi_\phi,
  \label{eq:width-pooling}
\end{equation}
The deficit remains common to the document, so this extension changes only the
width specification; see Figure~\ref{fig:pooling-model}(b).  Integrating over
$\phi$ induces marginal dependence and partial pooling of the token-specific
widths.  The hyperprior $\pi_\phi$ assigns equal mass to the five concentrations
$\lambda\in\{0,.5,1.5,4,\infty\}$ and, conditional on each positive
concentration, equal mass to the five centers.  The $\lambda=0$ component occurs
once because its distribution does not depend on the center.  The resulting
$21$ values of $\phi$, combined with $96$ deficit nodes, give $2{,}016$
components.  The simulation seed is $20260903$.  The degenerate component
$\lambda=\infty$ assigns a common width to the document; integrating over its
centers reproduces the fixed uniform ladder model to $3.6\times10^{-15}$ on the
checked paths.  At $\lambda=0$, widths are independent and uniform on the
ladder, but $\Delta$ remains common.  This differs from the tokenwise union
model in \eqref{eq:token-bf-state}, which also redraws the deficit at each position.

The evaluated hierarchical width model fixes $S=\mathrm{width}$ and
$\alpha=\infty$.  These nodes in Figure~\ref{fig:pooling-model}(b) describe the
general model family, while $\bm q_t$ is deterministic given $J_t$ in the
evaluated submodel.  It is compared with the shared tail-width model, which
draws one width from the same ladder for the whole document.

Table~\ref{tab:width-pooling} evaluates within-document dependence using a
Gaussian copula on the ladder.  The generative family differs from the
conditional model in \eqref{eq:width-pooling}, allowing assessment beyond the
detector's assumed family.  The copula also holds the marginal width
distribution fixed as persistence varies.

Specifically, $Z_t=\rho W+\sqrt{1-\rho^2}\,\varepsilon_t$, with $W$ drawn once per
document and $W,\varepsilon_1,\ldots,\varepsilon_n$ independent standard normals, is mapped through $\Phi(Z_t)$ to
the uniform ladder's quantiles.  The factor loadings $1$, $.95$, $.8$, $.6$, $0$
for Q1--Q5 correspond to latent pair correlations $\rho^2$ of $1$, $.9025$, $.64$,
$.36$, $0$, not correlations of $\rho$.  Only Q1 holds width fixed within a
document; Q2--Q4 allow within-document variation and Q5 specifies independent
token-specific widths.  The hierarchical model includes Q1 and Q5 at
$\lambda=\infty$ and $\lambda=0$; the shared-width model includes Q1 but not Q5.
Realized width persistence,
the between-document share of log-width variance, is $1.000$, $.854$, $.594$,
$.323$, and $.000$, respectively.

Both models use the same five-rung support, so their comparison isolates the
dependence specification.  Exact McNemar tests use paired decisions on the same
documents, with Holm adjustment across fifteen regime--horizon cells.  Two
differences are significant, both under the independent-width regime Q5: the
hierarchical model reduces estimated Type~II error relative to the shared-width
model by $.0038$ at $n=100$ ($p=.001$) and $.0024$ at $n=300$ ($p=.007$).
The common-width assumption is misspecified under Q5, whereas the hierarchical
model includes its width distribution.  Point estimates also favor the
hierarchical model at Q3 and Q4, but the corresponding unadjusted rejections
do not remain significant after Holm adjustment ($p=.057$ and $.112$).
No difference is significant at Q1, where the shared-width model is correctly
specified, or at Q2, where widths vary but remain strongly dependent.
Every absolute difference is at most $.0038$, and none is significant at
$n=700$.  These nonsignificant comparisons do not establish equivalence.

As in the deficit sweep, marginal binomial standard errors do not describe the
paired differences.  The companion hierarchical-width summaries contain discordant counts and exact McNemar tests, with per-document decisions supplied separately; Section~\ref{sec:computational-materials} identifies both records.  The 15-cell Holm family includes
every regime and horizon but was added after the Type~II estimates, so it too
is comprehensive but post hoc.

The additional hierarchical structure produces smaller observed changes in
Type~II error than the residual-tail specification in these experiments.
No hierarchical-versus-shared deficit comparison is significant in
Section~\ref{sec:deficit-pooling}, whereas extending the equal-tail model to a
width ladder in Table~\ref{tab:tail-widths} changes estimated Type~II error by an
order of magnitude more than the width-dependence specification considered here.

\begin{table}[htbp]
\centering
\scriptsize
\renewcommand{\arraystretch}{0.9}
\setlength{\tabcolsep}{3pt}%
\caption{Hierarchical width model under the fixed-marginal Gaussian-copula design of Section~\ref{sec:width-pooling}; Q1 holds one width per document and Q5 redraws width independently by token.  Entries are Type~II errors at nominal 5\% size with binomial Monte Carlo standard errors; bold marks observed column minima, ties included.  Maximum regret is the largest excess over the best scored rule across Q1--Q5, with a paired 2,000-replicate document-bootstrap standard error.  The $h^\star_{\mathrm{gum},\Delta_0}$ rows use the three published tunings of \citet{li2025framework}.}
\label{tab:width-pooling}
\begin{tabular}{@{}lrrrrrcr@{}}
\toprule
& \multicolumn{5}{c}{Type~II error by width-persistence regime} & & Max \\
\cmidrule(lr){2-6}\cmidrule(lr){8-8}
Rule & Q1 & Q2 & Q3 & Q4 & Q5 & & regret \\
\midrule
\multicolumn{8}{@{}l}{\textit{Horizon $n=100$}}\\[1pt]
$h_{\mathrm{ars}}$ & \mcse{.0716}{.0036} & \mcse{.0648}{.0035} & \mcse{.0698}{.0036} & \mcse{.0688}{.0036} & \mcse{.0596}{.0033} & & \mcse{.0204}{.0019} \\
$h_{\log}$ & \mcse{.1530}{.0051} & \mcse{.1704}{.0053} & \mcse{.1702}{.0053} & \mcse{.1704}{.0053} & \mcse{.1648}{.0052} & & \mcse{.1246}{.0038} \\
$h_{\mathrm{ind},1/e}$ & \mcse{.2192}{.0059} & \mcse{.2370}{.0060} & \mcse{.2276}{.0059} & \mcse{.2332}{.0060} & \mcse{.2306}{.0060} & & \mcse{.1904}{.0045} \\
$h^\star_{\mathrm{gum},.1}$ & \mcse{.0968}{.0042} & \mcse{.0918}{.0041} & \mcse{.1004}{.0043} & \mcse{.1000}{.0042} & \mcse{.0954}{.0042} & & \mcse{.0552}{.0029} \\
$h^\star_{\mathrm{gum},.01}$ & \mcse{.1408}{.0049} & \mcse{.1106}{.0044} & \mcse{.0796}{.0038} & \mcse{.0634}{.0034} & \mcse{.0494}{.0031} & & \mcse{.0794}{.0041} \\
$h^\star_{\mathrm{gum},.005}$ & \mcse{.1814}{.0054} & \mcse{.1492}{.0050} & \mcse{.0992}{.0042} & \mcse{.0640}{.0035} & \mcse{.0426}{.0029} & & \mcse{.1200}{.0048} \\
Bayes, shared $\Delta$ (equal tail) & \mcse{.1042}{.0043} & \mcse{.1032}{.0043} & \mcse{.0936}{.0041} & \mcse{.0802}{.0038} & \mcse{.0662}{.0035} & & \mcse{.0506}{.0032} \\
\quad $+$ shared tail width & \bestmcse{.0614}{.0034} & \mcse{.0532}{.0032} & \mcse{.0572}{.0033} & \mcse{.0512}{.0031} & \mcse{.0440}{.0029} & & \mcse{.0038}{.0008} \\
\quad $+$ hierarchical tail width & \mcse{.0616}{.0034} & \bestmcse{.0526}{.0032} & \bestmcse{.0544}{.0032} & \bestmcse{.0484}{.0030} & \bestmcse{.0402}{.0028} & & \bestmcse{.0002}{.0007} \\
\addlinespace[4pt]
\multicolumn{8}{@{}l}{\textit{Horizon $n=300$}}\\[1pt]
$h_{\mathrm{ars}}$ & \mcse{.0340}{.0026} & \mcse{.0284}{.0023} & \mcse{.0360}{.0026} & \mcse{.0336}{.0025} & \mcse{.0292}{.0024} & & \mcse{.0174}{.0017} \\
$h_{\log}$ & \mcse{.0946}{.0041} & \mcse{.0912}{.0041} & \mcse{.0980}{.0042} & \mcse{.1004}{.0043} & \mcse{.0950}{.0041} & & \mcse{.0842}{.0034} \\
$h_{\mathrm{ind},1/e}$ & \mcse{.1256}{.0047} & \mcse{.1274}{.0047} & \mcse{.1274}{.0047} & \mcse{.1364}{.0049} & \mcse{.1250}{.0047} & & \mcse{.1202}{.0044} \\
$h^\star_{\mathrm{gum},.1}$ & \mcse{.0508}{.0031} & \mcse{.0490}{.0031} & \mcse{.0584}{.0033} & \mcse{.0548}{.0032} & \mcse{.0498}{.0031} & & \mcse{.0386}{.0022} \\
$h^\star_{\mathrm{gum},.01}$ & \mcse{.0448}{.0029} & \mcse{.0318}{.0025} & \mcse{.0308}{.0024} & \mcse{.0226}{.0021} & \mcse{.0228}{.0021} & & \mcse{.0226}{.0022} \\
$h^\star_{\mathrm{gum},.005}$ & \mcse{.0890}{.0040} & \mcse{.0582}{.0033} & \mcse{.0374}{.0027} & \mcse{.0214}{.0020} & \mcse{.0192}{.0019} & & \mcse{.0668}{.0037} \\
Bayes, shared $\Delta$ (equal tail) & \mcse{.0572}{.0033} & \mcse{.0514}{.0031} & \mcse{.0498}{.0031} & \mcse{.0386}{.0027} & \mcse{.0270}{.0023} & & \mcse{.0350}{.0023} \\
\quad $+$ shared tail width & \bestmcse{.0222}{.0021} & \mcse{.0178}{.0019} & \mcse{.0220}{.0021} & \mcse{.0172}{.0018} & \mcse{.0172}{.0018} & & \mcse{.0024}{.0006} \\
\quad $+$ hierarchical tail width & \mcse{.0234}{.0021} & \bestmcse{.0176}{.0019} & \bestmcse{.0214}{.0020} & \bestmcse{.0162}{.0018} & \bestmcse{.0148}{.0017} & & \bestmcse{.0012}{.0006} \\
\addlinespace[4pt]
\multicolumn{8}{@{}l}{\textit{Horizon $n=700$}}\\[1pt]
$h_{\mathrm{ars}}$ & \mcse{.0172}{.0018} & \mcse{.0156}{.0018} & \mcse{.0208}{.0020} & \mcse{.0188}{.0019} & \mcse{.0156}{.0018} & & \mcse{.0120}{.0012} \\
$h_{\log}$ & \mcse{.0624}{.0034} & \mcse{.0544}{.0032} & \mcse{.0650}{.0035} & \mcse{.0662}{.0035} & \mcse{.0624}{.0034} & & \mcse{.0594}{.0026} \\
$h_{\mathrm{ind},1/e}$ & \mcse{.0844}{.0039} & \mcse{.0818}{.0039} & \mcse{.0882}{.0040} & \mcse{.0872}{.0040} & \mcse{.0884}{.0040} & & \mcse{.0836}{.0032} \\
$h^\star_{\mathrm{gum},.1}$ & \mcse{.0336}{.0025} & \mcse{.0270}{.0023} & \mcse{.0364}{.0026} & \mcse{.0372}{.0027} & \mcse{.0344}{.0026} & & \mcse{.0304}{.0020} \\
$h^\star_{\mathrm{gum},.01}$ & \mcse{.0152}{.0017} & \mcse{.0112}{.0015} & \mcse{.0162}{.0018} & \mcse{.0134}{.0016} & \mcse{.0116}{.0015} & & \mcse{.0068}{.0009} \\
$h^\star_{\mathrm{gum},.005}$ & \mcse{.0268}{.0023} & \mcse{.0156}{.0018} & \mcse{.0150}{.0017} & \mcse{.0100}{.0014} & \mcse{.0092}{.0014} & & \mcse{.0176}{.0020} \\
Bayes, shared $\Delta$ (equal tail) & \mcse{.0344}{.0026} & \mcse{.0306}{.0024} & \mcse{.0312}{.0025} & \mcse{.0192}{.0019} & \mcse{.0132}{.0016} & & \mcse{.0252}{.0019} \\
\quad $+$ shared tail width & \bestmcse{.0092}{.0014} & \bestmcse{.0078}{.0012} & \mcse{.0102}{.0014} & \mcse{.0072}{.0012} & \mcse{.0054}{.0010} & & \bestmcse{.0006}{.0003} \\
\quad $+$ hierarchical tail width & \mcse{.0098}{.0014} & \bestmcse{.0078}{.0012} & \bestmcse{.0096}{.0014} & \bestmcse{.0068}{.0012} & \bestmcse{.0048}{.0010} & & \bestmcse{.0006}{.0005} \\
\bottomrule
\end{tabular}
\end{table}

\clearpage
\subsection{Computational-materials index}
\label{sec:computational-materials}

\begingroup
\small
\raggedright
The companion repository contains this technical report, its source and
bibliography, and the supporting numerical records and code.  Its README maps each file to the analyses below and explains
its fields; a file index and SHA-256 checksums identify the exact packaged
versions.  The numerical results needed to interpret the conclusions are
reported in the manuscript and supplement.  The machine-readable files provide
additional numerical detail and permit independent checks of the tabulated
results.

Except for the historical record noted below, numerical files retain their
relative paths under \path{results/bayesian_paper_benchmark/} in the companion
repository.  The following index gives basenames; the README gives full paths.
\begin{itemize}
  \item \emph{Synthetic comparisons and posterior summaries}
  (Sections~\ref{sec:shared-sensitivity}, \ref{sec:family},
  \ref{sec:union-weight-sweep}, and~\ref{sec:synthetic-details}):
  \path{paired_comparisons.json}, \path{regime_sweep.json},
  \path{tail_regime_sweep.json}, \path{delta_learning.json}, and
  \path{block_posterior.json}, together with benchmark error-rate tables
  and paired rejection indicators.

  \item \emph{Contamination and numerical validation}
  (Sections~\ref{sec:contamination-supplement}, \ref{sec:stable-computation},
  and~\ref{sec:numerical-checks}): \path{contamination_rho_diagnostic.json},
  \path{benchmark_summary.json}, \path{clean_quadrature_sensitivity.json},
  \path{clean_union_quadrature_sensitivity.json}, and
  \path{contamination_union_quadrature_sensitivity.json}, with the other
  contamination summaries and interpolation checks identified in the README.

  \item \emph{Reanalysis of the archived benchmark data of Li et al.}
  (Section~\ref{sec:released-analysis}): \path{tail_width_estimate.json},
  \path{real_data_summary.json}, \path{real_data_results.csv}, and
  \path{real_data_indicators.npz}, together with the full prefix curves and
  repetition diagnostics.  These files contain this study's reanalysis.
  Li et al.'s original benchmark tensors and reference-score JSON arrays
  remain separately available from their WatermarkFramework repository at
  the commit cited in \citet{li2024watermarkframework}; the README identifies
  the required files and their verified hashes.

  \item \emph{Temperature-matched comparisons}
  (Sections~\ref{sec:temperature-matched} and~\ref{sec:matched-details}):
  \path{delong_tests_bootstrap.json} and \path{delong_tests_holdout.json},
  with the mixing-weight, wider-prior, empirical-deficit, and prefix-curve
  summaries identified in the README.  Despite their historical filenames,
  the reported inference uses the paired prompt-cluster bootstrap described
  in the manuscript, not the independent-class DeLong variance formula.
  The eight stored generation archives and extension-prompt tables are
  included under \path{real_model/temperature_matched/} within the results
  directory; the README lists their exact paths.

  \item \emph{Hierarchical extensions}
  (Section~\ref{sec:hierarchical-extensions}):
  \path{deficit_persistence_sweep.json} and
  \path{width_persistence_sweep.json}, with the corresponding
  \path{deficit_persistence_indicators.npz} and
  \path{width_persistence_indicators.npz}.  The earlier Beta-generator script
  and summary are kept separately under \path{historical/6551ee5/}; they are
  not the fixed-marginal Gaussian-copula experiment reported in the tables.

  \item \emph{Implementation and provenance}: the \path{code/} directory
  contains the analysis scripts, tests, dependency specifications, and the
  prompt-construction script \path{build_prompts.py}.
  The current-source snapshot \path{provenance.json} records file hashes and
  the manifest-writer environment, not a generation-time certificate.
\end{itemize}

The repository README distinguishes stored-result checks from regeneration.
External model checkpoints, tokenizers, C4 records, and the reference
implementation must be obtained from their cited sources for model generation;
they are not redistributed here.  Stored-array analyses also require the
checksum-verified benchmark tensors for prompt alignment.  The historical-generation qualifications in
Section~\ref{sec:released-replay}, including the unverified independence of raw
random streams, continue to apply.  Packaging the records does not strengthen
those provenance claims.  Only local directory names in public-facing metadata
are omitted; numerical values and original generation records are unchanged.
\par
\endgroup

\end{document}